\documentclass{article}
\PassOptionsToPackage{numbers}{natbib}

\usepackage[final]{neurips_2022}

\usepackage[utf8]{inputenc} 
\usepackage[T1]{fontenc}    
\usepackage{hyperref}  
\hypersetup{
    colorlinks=true,
    linkcolor=blue,
    filecolor=magenta,      
    urlcolor=cyan,
}
\usepackage{url}            
\usepackage{booktabs}       
\usepackage{amsfonts}       
\usepackage{nicefrac}       
\usepackage{microtype}      
\usepackage{xcolor}         
\usepackage{tikz}
\usepackage{ifthen,amssymb}
\usepackage[cmex10]{amsmath} 
\usepackage{amsfonts,mathtools}
\usepackage{enumerate}
\usepackage{algorithm}
\usepackage{algpseudocode}
\usepackage{comment}
\usepackage{amsthm}

\def\var{\mathop{\rm Var}\nolimits}%
\newcommand{\gv}{{\bf g}}

\newcommand{\uv}{{\bf u}}
\newcommand{\vv}{{\bf v}}

\newcommand{\xv}{{\bf x}}
\newcommand{\yv}{{\bf y}}

\newcommand{\Av}{{\bf A}}
\newcommand{\Bv}{{\bf B}}

\newcommand{\Gv}{{\bf G}}
\newcommand{\Hv}{{\bf H}}

\newcommand{\Uv}{{\bf U}}
\newcommand{\Vv}{{\bf V}}

\newcommand{\Xv}{{\bf X}}
\newcommand{\Yv}{{\bf Y}}

\def\textiid{i.i.d.\@\xspace}
\newcommand\iid{\ifmmode\text{ i.i.d. } \else \textiid \fi}

\newcommand{\beqs}{\begin{equation*}}
\newcommand{\eeqs}{\end{equation*}}
\newcommand{\beq}{\begin{equation}}
\newcommand{\eeq}{\end{equation}}

\usepackage{xcolor}
\usepackage{appendix}

\DeclareMathOperator{\Tr}{\text{\rm{Tr}}}

\DeclareMathOperator*{\argmin}{arg\,min}

\theoremstyle{plain}
\newtheorem{theorem}{Theorem}
\newtheorem{assumption}{Assumption}
\newtheorem{lemma}{Lemma}

\newtheorem{definition}{Definition}
\newtheorem{remark}{Remark}
\newtheorem{example}{Example}

\usepackage{todonotes}

\title{Stochastic Second-Order Methods Improve Best-Known Sample Complexity of SGD for Gradient-Dominated Functions}

\author{%
  Saeed Masiha\thanks{equal contribution} \\
  College of Management of Technology\\
  EPFL, Lausanne, Switzerland\\
  \texttt{mohammadsaeed.masiha@epfl.ch} \\
  \And
  Saber Salehkaleybar$^*$ \\
  School of Computer and Communication Sciences\\
  EPFL, Lausanne, Switzerland\\
  \texttt{saber.salehkaleybar@epfl.ch}
  \And
  Niao He\\
  Department of Computer Science\\
  ETH, Zurich, Switzerland\\
  \texttt{niao.he@inf.ethz.ch}
  \And
  Negar Kiyavash\\
  College of Management of Technology\\
  EPFL, Lausanne, Switzerland\\
  \texttt{negar.kiyavash@epfl.ch}
  \And
  Patrick Thiran\\
  School of Computer and Communication Sciences\\
  EPFL, Lausanne, Switzerland\\
  \texttt{patrick.thiran@epfl.ch}
}

\begin{document}

\vspace{-50mm}
\maketitle

\vspace{-5mm}
\begin{abstract}

We study the performance of Stochastic Cubic Regularized Newton (SCRN) on a class of functions satisfying gradient dominance property with $1\le\alpha\le2$ which holds in a wide range of applications in machine learning and signal processing. This condition ensures that any first-order stationary point is a global optimum. We prove that the total sample complexity of SCRN in achieving $\epsilon$-global optimum is $\mathcal{O}(\epsilon^{-7/(2\alpha)+1})$ for $1\le\alpha< 3/2$ and $\mathcal{\tilde{O}}(\epsilon^{-2/(\alpha)})$ for $3/2\le\alpha\le 2$. SCRN improves the best-known sample complexity of stochastic gradient descent. Even under a weak version of gradient dominance property, which is applicable to policy-based reinforcement learning (RL), SCRN  achieves the same improvement over stochastic policy gradient methods. Additionally, we show that the average sample complexity of SCRN can be reduced to ${\mathcal{O}}(\epsilon^{-2})$ for $\alpha=1$ using a variance reduction method with time-varying batch sizes. Experimental results in various RL settings showcase the remarkable performance of SCRN compared to first-order methods.

\end{abstract}
\vspace{-3mm}
\section{Introduction}
\vspace{-1mm}
Consider the following unconstrained stochastic non-convex optimization problem:
\begin{align}
    \min_{\xv\in\mathbb{R}^{d}}F(\xv):=\mathbb{E}_{\xi}[f(\xv,\xi)]\label{problem0},
\end{align}
where the random variable $\xi$ is sampled from an underlying distribution $P_{\xi}$. In order to optimize the objective function $F(\xv)$, we have access to the first and second derivatives of stochastic function $f(\xv,\xi)$. The above optimization problem covers a wide range of problems, from the offline setting where the objective function is minimized over a fixed number of samples, to the online setting where the samples are drawn sequentially. 

In the deterministic case (where we have access to the derivatives of $F(\xv)$), the gradient descent (GD) algorithm in the non-convex setting only guarantees convergence to a first-order stationary point (FOSP) (i.e., a point $\xv$ such that $\|\nabla F(\xv)\|=0$), which can be a local minimum, a local maximum, or a saddle point. 
In contrast, second-order methods accessing the Hessian of $F$ (or Hessian of $f(\xv,\xi)$ in the stochastic setting) can exploit the curvature information to effectively escape saddle points and converge to a 
second-order stationary point (SOSP) (i.e., such that $\|\nabla F(\xv)\|=0$, $\nabla^{2}F(\xv)\succeq 0$). In their seminal work, Nesterov and Polyak \cite{nesterov2006cubic} proposed the so-called cubic-regularized Newton (CRN) algorithm which exploits Hessian information and globally converges to an SOSP at a sub-linear rate of $\mathcal{O}(1/k^{2/3})$, where $k$ is the number of iterations.

In recent years, the performance of stochastic CRN (SCRN) for general non-convex functions has been the focus of several studies (for more details, see the related work in Section~\ref{related_work}). A variance-reduced version of SCRN \cite{arjevani2020second} can find $(\epsilon,\gamma)$-SOSP (i.e., a point $\xv$ such that $\|\nabla F(\xv)\|\le \epsilon$ and $\nabla^2F(\xv)\succeq -\gamma I$) with the sample complexity of $\tilde{\mathcal{O}}(\epsilon^{-3})$. Moreover, this rate is optimal for achieving $\epsilon$-approximate FOSP (i.e., a point $\xv$ such that $\|\nabla F(\xv)\|\le \epsilon$) and it cannot be improved using any stochastic $p$-th order methods for $p\geq 2$ \cite{arjevani2020second}.

Nesterov and Polyak \citep{nesterov2006cubic} studied CRN under the gradient dominance property (See Assumption \ref{assump:4}). Under Assumption \ref{assump:4} with $\alpha=2$, they showed that iterates of $F(\xv_{t})-\min_{\xv}F(\xv)$ converges to zero super-linearly.
The case of $\alpha=2$ (commonly called Polyak-Łojasiewicz (PL) condition) was originally introduced by Polyak in \citep{polyak1963gradient}, who showed that GD achieves linear convergence rate. The gradient dominance property or its weak variations are satisfied in a quite wide range of machine learning applications such as neural networks with one hidden layer \citep{li2017convergence} or ResNet with linear activation \citep{hardt2016identity} (for more details, see Section \ref{related_work}). A particularly important application of the weak version of gradient dominance property with $\alpha=1$ (Assumption \ref{relaxed weak PL}) is in policy-based reinforcement learning (RL) \citep{yuan2021general}. 


Khaled et al. \cite{khaled2020better} showed that under PL condition, the stochastic GD (SGD) with time-varying step-size returns a point $\hat{\xv}$ with a sample complexity of $\mathcal{O}(1/\epsilon)$, to reach $\mathbb{E}[F(\hat{\xv})]-\min_{\xv}F(\xv)\le \epsilon$.
Furthermore, the dependency of the sample complexity of SGD on $\epsilon$ is optimal \citep{nguyen2019tight}. Recently, Fontaine et al. \citep{fontaine2021convergence} obtained a sample complexity of $\mathcal{O}(\epsilon^{-\frac{4}{\alpha}+1})$ for SGD under gradient dominance property  with $1\le \alpha\le 2$. 
This shows that the worst sample complexity occurs for $\alpha=1$, which is precisely the value of $\alpha$ that finds important applications in policy-based RL. 
Indeed, under a weak version of gradient dominance property with $\alpha=1$,
it has been shown that stochastic policy gradient (SPG) converges to the optimum point with a sample complexity of $\tilde{\mathcal{O}}(\epsilon^{-3})$ \citep{yuan2021general,ding2021global}. 
We know that in the deterministic case, CRN outperforms GD under gradient dominance property for all $\alpha\in[1,2]$ \citep{nesterov2006cubic,zhou2018convergence}\footnote{In particular, for $\alpha\in [1,3/2)$, the number of iterations of CRN is $\mathcal{O}(1/\epsilon^{3/(2\alpha)-1})$, for $\alpha=3/2$ is $\mathcal{O}(\log(1/\epsilon))$, and for $\alpha\in (3/2,2]$, the number of iterations is $\mathcal{O}(\log\log(1/\epsilon))$. For $\alpha\in [1,2)$, the number of iterations of GD is $\mathcal{O}(1/\epsilon^{2/\alpha-1})$ and for $\alpha=2$, is $\mathcal{O}(\log(1/\epsilon))$ \citep{zhou2018convergence}.}. Therefore, a natural question that arises is whether this holds true in the stochastic setting as well? That is, does SCRN improve upon SGD under the gradient dominance property?

Herein, we address this question. Specifically, our main contributions are as follows:
\begin{itemize}
\vspace{-1.5mm}
    \item We analyze the sample complexity of SCRN under gradient dominance property for $1\leq \alpha \leq 2$ in order to return an $\epsilon$-global stationary point $\hat{\xv}$ satisfying $F(\hat{\xv})-\min_{\xv}F(\xv)\le \epsilon$ (in expectation or with high probability).
     As stated in Table \ref{Table_comparison_SGD_SCRN}, SCRN improves upon the best-known sample complexity of SGD for all $1\leq \alpha <2$. The largest improvement is for $\alpha=1$, and is ${\mathcal{O}}({\epsilon^{-0.5}})$.
    \item In the setting of policy-based RL, under the weak version of gradient dominance property with $\alpha=1$ (Assumption \ref{relaxed weak PL}), we show that SCRN achieves a sample complexity of $\tilde{\mathcal{O}}(\epsilon^{-2.5})$, improving over the best-known sample complexity of SPG by a factor of $\tilde{\mathcal{O}}(\epsilon^{-0.5})$. 


    \item We show that an adaptation of a variance-reduced SCRN \cite{zhou2020stochastic} with time-varying batch sizes further improves the sample complexity of SCRN.
    For $\alpha=1$, the average sample complexity is reduced to ${\mathcal{O}}(\epsilon^{-2})$.
\end{itemize}


\begin{table}[]
\vspace{-1.5mm}
\addtolength{\tabcolsep}{-1pt}
    \caption{Comparison of sample complexities of SGD on Lipschitz gradient functions and SCRN on Lipschitz Hessian functions to achieve $\epsilon$-global stationary point under gradient dominance property with $\alpha\in[1,2]$. The last column indicates the improvement of SCRN with respect to SGD.}
  \begin{center}
  \renewcommand{\arraystretch}{2}\small
   \begin{tabular}{|c|c|c|c|}
 \hline
 $\alpha$ &SGD \cite{fontaine2021convergence} & SCRN (Ours) & Improvement\\
 \hline
$[1,\frac{3}{2})$&   $\mathcal{O}({\epsilon^{-4/\alpha+1}})$ &$\mathcal{O}(\epsilon^{-7/(2\alpha)+1})$&$\mathcal{O}({\epsilon^{-1/(2\alpha)}})$\\[0.1em]
 \hline 
$\frac{3}{2}$& $\mathcal{O}({\epsilon^{-{4}/{\alpha}+1}})=\mathcal{O}({\epsilon^{-{5}/{3}}})$ &$\tilde{\mathcal{O}}({\epsilon^{-{7}/{(2\alpha)}+1}})=\tilde{\mathcal{O}}({\epsilon^{-{4}/{3}}})$&$\tilde{\mathcal{O}}({\epsilon^{-{1}/{3}}})$\\[0.1em]
\hline
$(\frac{3}{2},2)$& $\mathcal{O}({\epsilon^{-{4}/{\alpha}+1}})$  &$\tilde{\mathcal{O}}({\epsilon^{-2/\alpha}})$ w.h.p& $\tilde{\mathcal{O}}({\epsilon^{-{2}/{\alpha}+1}})$\\[0.1em]
\hline
$2$  & $\mathcal{O}({\epsilon^{-1}})$ & $\tilde{\mathcal{O}}({\epsilon^{-{2}/{\alpha}}})=\tilde{\mathcal{O}}({\epsilon^{-1}})$ w.h.p&--\\[0.1em]
 \hline
\end{tabular}
\end{center}

    \label{Table_comparison_SGD_SCRN}
    \vspace{-5mm}
\end{table}


\vspace{-3mm}
\subsection{Related work}\label{related_work}
\vspace{-1.5mm}
\textbf{Gradient dominance property and its applications:}
The gradient dominance property with $\alpha=2$ (commonly called PL condition) was originally introduced by Polyak in \citep{polyak1963gradient}.
It was shown by Karimi et al. \citep{karimi2016linear} to be weaker than the most recent global optimality conditions that appeared in the literature of machine learning~\cite{liu2014asynchronous, necoara2019linear,zhang2013linear}.
 The gradient dominance property is also satisfied (sometimes locally rather than globally, and also under distributional assumptions) for the population risk in some learning models including neural networks with one hidden layer \citep{li2017convergence}, ResNet with linear activation \citep{hardt2016identity}, and generalized linear model and robust regression \citep{foster2018uniform}.
Moreover, in policy-based reinforcement learning (RL), a weak version of gradient dominance property with $\alpha=1$ (see Assumption \ref{relaxed weak PL}) holds for some classes of policies (such as Gaussian policy and log-linear policy).

\textbf{Variants of cubic regularized Newton method:}
For non-convex optimization, Nesterov and Polyak \cite{nesterov2006cubic} proposed the CRN algorithm, which converges to a SOSP with the convergence rate of $\mathcal{O}(1/k^{2/3})$ (where $k$ is the number of iterations) by solving a cubic sub-problem in each iteration. 
Cartis et al. \cite{cartis2011adaptive} presented an adaptive framework for cubic regularization method.
~In \cite{kohler2017sub,xu2020second}, sub-sampled versions of gradient and Hessian were used in CRN to overcome the computational burden of Hessian matrix evaluations in high dimensional settings. In the context of stochastic optimization, Tripuraneni et al. \cite{tripuraneni2018stochastic} proposed a stochastic cubic regularization algorithm that obtains $\epsilon$-SOSP with sample complexity of $\tilde{\mathcal{O}}(\epsilon^{-3.5})$.
Arjavani et al. \cite{arjevani2020second} improved the sample complexity to $\tilde{\mathcal{O
}}(\epsilon^{-3})$ using variance reduction.
In the convex setting, Song et al. \cite{song2019inexact} presented a proximal CRN and its accelerated version, and proved a sample complexity of $\tilde{\mathcal{O}}(\epsilon^{-2})$ to reach $\epsilon$-global stationary point as long as the approximated Hessian in each iteration satisfies certain properties.
In finite-sum non-convex setting, 
Zhou et al. \cite{zhou2018stochastic} proposed an adaptive sub-sampled CRN method that requires $\tilde{\mathcal{O}}(N+N^{4/5}\epsilon^{-3/2})$ to find $\epsilon$-SOSP, where $N$ is the total number of samples. Sample complexity was further reduced to
  $\tilde{\mathcal{O}}(N+N^{2/3}\epsilon^{-3/2})$ using various variance reduction methods \cite{wang2019stochastic,zhang2022adaptive,zhou2019stochastic}. To the best of our knowledge, no previous work on analyzing SCRN for gradient-dominant functions exists.

\vspace{-2mm}
\subsection{Notations}
\vspace{-1mm}
We adopt the following notation in the sequel. Calligraphic letters (e.g., $\mathcal{S}$) denote spaces. Upper-case bold letters (e.g., $\Av$) denote matrices, and the lower-case bold letters (e.g., $\xv$) denote vectors. $\|\cdot\|$ denote the $\ell_{2}$-norm for vectors and the operator norm for matrices ($\|\Av\|:=\lambda_{\max}(\Av^{T}\Av)$ where $\lambda_{\max}(\Xv) $ is the maximum eigenvalue of matrix $\Xv$), respectively.  $\Av\succeq\Bv$ indicates that $\Av-\Bv$ is positive semi-definite. We use the notation $\mathcal{O}$ to hide constants, and the notation $\tilde{\mathcal{O}}$ to hide both constants and logarithmic factors. $X\le_{1-\delta}Y$ denotes that random variable $X$ is less than or equal to random variable $Y$ with the probability  at least $1-\delta$. 
\vspace{-2.5mm}
\section{Setup}
\vspace{-2mm}

Recall the stochastic non-convex optimization problem in \eqref{problem0}, where the goal is to minimize the objective function $F(\xv)$ having access to stochastic gradients $\nabla f(\xv,\xi)$ and stochastic Hessian matrices $\nabla^2 f(\xv,\xi)$. We make the following assumption about the objective function in \eqref{problem0}.
\begin{assumption}\label{assump1}
        The Hessian of $F$ is Lipschitz continuous with constant $L_2$, i.e.,
    \begin{align}
       \|\nabla^{2} F(\xv)-\nabla^{2} F(\yv)\|\le L_{2}\|\xv-\yv\|_{2},\quad \forall \xv ,\yv\in \mathbb{R}^d.
    \end{align}
\end{assumption}

Consider the empirical estimators ~$\gv_{t}:= \frac{1}{n_{1}}\sum_{i=1}^{n_{1}}\nabla f(\xv_{t},\xi_{i}),$ and $~\Hv_{t} := \frac{1}{n_{2}}\sum_{i=1}^{n_{2}}\nabla^{2} f(\xv_{t},\xi_{i})$ where $n_1$ and $n_2$ are the numbers of samples used for estimating the gradient vector and Hessian matrix, respectively.

\begin{algorithm}

\caption{Stochastic cubic regularized Newton method with stopping criterion}\label{alg:cap}
\textbf{Input:} Batch sizes $n_{1}$, $n_{2}$, initial point $\xv_{0}$, accuracy $\epsilon$, cubic penalty parameter $M$, maximum number of iterations $T$
\begin{algorithmic}[1]
\State $t\gets 1$
\State $\|{\bf\Delta}_{0}\|=\infty$
\While{$\|{\bf\Delta}_{t-1}\|\ge \sqrt[2\alpha]{\epsilon}$ or $t\le T$}
\State $\gv_{t} \gets \frac{1}{n_{1}}\sum_{i=1}^{n_{1}}\nabla f(\xv_{t},\xi_{i})$
\State $\Hv_{t} \gets \frac{1}{n_{2}}\sum_{i=1}^{n_{2}}\nabla^{2} f(\xv_{t},\xi_{i})$
\State ${\bf\Delta}_{t} \gets \argmin_{{\bf\Delta}\in\mathbb{R}^{d}}\langle \gv_{t}, {\bf\Delta}\rangle+\frac{1}{2}\langle {\bf\Delta}, \Hv_{t}{\bf\Delta}\rangle+\frac{M}{6}\|{\bf\Delta}\|^{3}$
\State $\xv_{t+1}\gets \xv_{t}+{\bf \Delta}_{t}$
\State $t \gets t+1$
\EndWhile
\State\Return{$\xv_{t}$}
\end{algorithmic}
\label{algorithm}
\vspace{-1mm}
\end{algorithm}
\vspace{-1mm}
\begin{definition}[Total sample complexity]
Given $\epsilon,\delta>0$, the total sample complexity is the number of calls (queries) of stochastic gradient and stochastic Hessian along the iterations until reaching a point $\xv$ that satisfies one of the following: (1) for high probability analysis: $F(\xv)-F(\xv^*)\leq \epsilon$ with probability at least $1-\delta$; or (2) for  analysis in expectation: $\mathbb{E}[F(\xv)]-F(\xv^*)\leq \epsilon$.
\end{definition}

Our goal is to study the performance of stochastic cubic regularized Newton (SCRN) for objective functions that satisfy the gradient dominance property. Algorithm \ref{alg:cap}  describes the steps in SCRN. At each iteration $t$, we take batches of stochastic gradient vectors and Hessian matrices (lines 4 and 5) and then solve the following sub-problem to obtain $\mathbf{\Delta}_t$ (line 6):  
\begin{align}\label{sub-problem}
  \min_{{\bf\Delta}\in\mathbb{R}^{d}}m_{t}({\bf\Delta}):=\langle \gv_{t}, {\bf\Delta}\rangle+\frac{1}{2}\langle {\bf\Delta}, \Hv_{t}{\bf\Delta}\rangle+\frac{M}{6}\|{\bf\Delta}\|^{3}.
\end{align}

We assume that there is an oracle that returns a global solution for this sub-problem (This assumption will be relaxed subsequently, see Remark \ref{inexact_subsolver}). Finally, we update $\xv_t$ in line 7. 



\vspace{-2mm}
\section{SCRN under gradient dominance property}
\vspace{-1mm}
We shall study the performance of SCRN for functions satisfying gradient dominance property, defined as follows.

\begin{assumption}
Function $F(\xv)$ satisfies gradient dominance property when for every $\xv\in\mathbb{R}^{d}$,
\begin{equation}
    F(\xv)-F(\xv^*)\leq \tau_{F} \|\nabla F(\xv)\|^{\alpha},
\end{equation}
where $\xv^{*}\in \argmin_{\xv}F(\xv)$, $\tau_F>0$, and $\alpha\in [1,2]$ are two constants.
\label{assump:4}
\end{assumption}

The case $\alpha=2$ is often referred to as PL condition \citep{polyak1963gradient, karimi2016linear}. In this paper, we consider all $\alpha$'s in the interval $[1,2]$.  
The gradient dominance property holds for a large class of functions including sub-analytic functions,
logarithm, and exponential functions, and semi-algebraic functions. These function classes cover some of the
most common non-convex objectives used in practice (see related work in Section~\ref{related_work}).

In the following lemma, we present a recursion inequality that captures the behaviour of the function  $F(\xv_{t})-F(\xv^{*})$ at each iteration~$t$ for SCRN under gradient dominance property.
\begin{lemma}\label{lemma0001}
Assume that function $F$ satisfies Assumption~\ref{assump1} (Lipschitz Hessian) and Assumption~\ref{assump:4} (gradient dominance property) for $\alpha\geq 1$. Then the resulting update $\xv_{t+1}$ in Algorithm~\ref{algorithm} (line 7) after plugging in ${{\bf\Delta}_{t}}$, the solution of sub-problem in \eqref{sub-problem}, satisfies the following:
\begin{align}
&F(\xv_{t+1})-F(\xv^*)\leq \nonumber\\
&\qquad C(F(\xv_t)-F(\xv_{t+1}))^{2\alpha/3}+C_g\|\nabla F(\xv_t)-\gv_{t}\|^\alpha+C_{H}\|\nabla^2F(\xv_t)-\Hv_{t}\|^{2\alpha},\label{eq_recursion_ineq}
\end{align}
where $C,C_g,C_{H}>0$ are constants depending on $M,L_2,$ and $\tau_F$, and defined in \eqref{eq20}.
\end{lemma}

Due to space limitations, all proofs are moved to the appendix.

In the following, we first provide  an analysis in expectation of SCRN under gradient dominance property with $\alpha\in[1,3/2]$. Next, we study the same algorithm for $\alpha\in(3/2,2]$ using a high probability analysis.
\vspace{-1.5mm}
\subsection{SCRN under gradient dominance property with $\alpha\in [1,3/2]$}\label{subsec_SCRN_PL_alpha_(1,3/2)}
We make the following assumption on the stochastic gradients and Hessians.

\begin{assumption}\label{assump2}
         For a given $\alpha\in[1,3/2]$ and for each query point
$\xv\in\mathbb{R}^{d}$:
\begin{align}
    &\mathbb{E}[\nabla f(\xv,\xi)]=\nabla F(\xv),\quad\mathbb{E}[\|\nabla f(\xv,\xi)-\nabla F(\xv)\|^{2}_{2}]\le \sigma^{2}_{1},\\
    &\mathbb{E}[\nabla^{2} f(\xv,\xi)]=\nabla^{2}F(\xv),\quad \mathbb{E}[\|\nabla^{2}f(\xv,\xi)-\nabla^{2}F(\xv)\|^{2\alpha}]\le \sigma_{2,\alpha}^{2},\label{Hessian_bounded}
\end{align}
where
$\sigma_1$ and $\sigma_{2,\alpha}$ are two constants. 

\end{assumption}

\begin{remark}
The assumption $ \mathbb{E}[\|\nabla^{2}f(\xv,\xi)-\nabla^{2}F(\xv)\|^{2\alpha}]\le \sigma_{2,\alpha}^{2}$ for $1\le \alpha\le 3/2$ is slightly stronger than the usual assumption $ \mathbb{E}[\|\nabla^{2}f(\xv,\xi)-\nabla^{2}F(\xv)\|^{2}]\le \sigma_{2}^{2}$. We need this assumption because of the specific form of the error of Hessian estimator in recursion inequality in \eqref{eq_recursion_ineq}. As a result of the assumption in~\eqref{Hessian_bounded}, using a version of matrix moment inequality (See Lemma \ref{lemma_norm_op_summation_of_iid_matrix} in Appendix), we show that the dependency of the Hessian sample complexity in Theorem~\ref{th1_expectation} on dimension $d$ is in the order of $\log d$. 
\end{remark}

\begin{theorem}\label{th1_expectation}
    Let $F(\xv)$ satisfy Assumptions \ref{assump1} and \ref{assump:4} for a given $\alpha$ and the stochastic gradient and Hessian  satisfy Assumption \ref{assump2} for the same $\alpha$. Moreover, assume that an exact solver for sub-problem \eqref{sub-problem} exists. Then Algorithm \ref{algorithm} outputs a point $\xv_{T}$ such that $\mathbb{E}[F(\xv_{T})]-F(\xv^{*})\le \epsilon$ after $T$ iterations, where
    \begin{enumerate}[(i)]
    \vspace{-1.5mm}
        \item if $\alpha\in[1,3/2)$,
   $T=\mathcal{O}(\epsilon^{-\frac{3-2\alpha}{2\alpha}})$, with access to the following numbers of samples of the stochastic gradient and Hessian per iteration:
    \begin{align}\label{eq009}
       n_{1}\ge\frac{{C_{g}}^{2/\alpha}}{C^{6/\alpha}}\cdot\frac{4^{2/\alpha}\sigma_{1}^{2/\alpha}}{\epsilon^{2/\alpha}},\quad n_{2}\ge\frac{{C'_{H}}^{1/\alpha}}{C^{3/\alpha}}\cdot\frac{4^{1/\alpha}\sigma_{2,\alpha}^{2/\alpha}}{\epsilon^{1/\alpha}},
    \end{align}
    where $C'_{H}$ is defined in \eqref{def_C'_{H}} and depends on $\log(d)$.
    \item if $\alpha=3/2$, 
     $T=\mathcal{O}\left(\log(1/\epsilon)\right)$ with the same numbers of samples per iteration as in \eqref{eq009}.
\end{enumerate}
\end{theorem}


\subsection{SCRN under gradient dominance property with $\alpha\in(3/2,2]$} \label{subsec_SCRN_PL_alpha_(3/2,2)}

\begin{definition}[Bernstein's condition for matrices]\label{bernstein condition_assump}
A zero-mean symmetric random matrix $\Xv$ satisfies the Bernstein condition with parameter $b>0$ if 
\begin{align}\label{bernstein condition_equation}
    \mathbb{E}[\Xv^{k}]\preccurlyeq \frac{1}{2}\,k!\,b^{k-2}\,\var(\Xv),\quad\text{for}\,\,k=3,4,\ldots
\end{align}
where $\var(\Xv):=\mathbb{E}[\Xv^{2}]-(\mathbb{E}[\Xv])^{2}$.
\end{definition}
\begin{assumption}\label{assump2.5}
We assume that the symmetric version of each centered gradient estimator $\Gv(\xv,\xi):=\begin{bmatrix}\boldsymbol{0}_{1\times 1} &\gv(\xv,\xi)^{T}\\
\gv(\xv,\xi)&\boldsymbol{0}_{d\times d}
\end{bmatrix}$ where $\gv(\xv,\xi):=\nabla f(\xv,\xi)-\nabla F(\xv)$ and each centered Hessian estimator $\Hv(\xv,\xi):=\nabla^{2}f(\xv,\xi)-\nabla^{2}F(\xv)$
satisfy Bernstein's condition \eqref{bernstein condition_equation} with parameters $M_{1}$ and $M_{2}$, respectively.
\end{assumption}


\begin{remark}\label{remark_bounded_gradient_Hessian} It is noteworthy that most previous work  analyzing SCRN \citep{tripuraneni2018stochastic,zhou2019stochastic,wang2019stochastic} assumed that  centered gradient and centered Hessian estimators are bounded, i.e., $
    \|\nabla f(\xv,\xi)-\nabla F(\xv)\|_{2}\overset{a.s.}{\le} M_{1},\|\nabla^{2} f(\xv,\xi)-\nabla^{2} F(\xv)\|\overset{a.s.}{\le} M_{2}$.
This is a stronger assumption than Assumption ~\ref{assump2.5} as it implies Bernstein's condition \eqref{bernstein condition_equation} for $\gv(\xv,\xi)$ and $\Hv(\xv,\xi)$. 

\end{remark} 
\begin{theorem}\label{th1}
Suppose that $F(\xv)$ satisfies Assumptions \ref{assump1}, \ref{assump:4}, the stochastic gradient and Hessian satisfy Assumption \ref{assump2} (with $\alpha=1$) and Assumption \ref{assump2.5}, and there exists an exact solver for sub-problem \eqref{sub-problem}. Then, Algorithm \ref{algorithm}, with probability $1-\delta$, outputs a solution $\xv_{T}$ such that $F(\xv_{T})-F(\xv^{*})\le \epsilon$ after $T=\mathcal{O}\left(\log\left(\log(1/\epsilon)\right)\right)$ iterations with the following numbers of samples for the stochastic gradient and Hessian, respectively per iteration:
\begin{align}
    &n_{1}\ge \frac{8}{3}\max\left(\frac{{\tilde{C}^{1/\alpha}}M_{1}}{\epsilon^{1/\alpha}},\frac{\tilde{C}^{2/\alpha}\sigma_{1}^{2}}{\epsilon^{2/\alpha}}\right)\log\left(\frac{4(T+1)d}{\delta}\right),\\
    &n_{2}\ge \frac{8}{3}\max\left(\frac{{\tilde{C}^{1/(2\alpha)}}M_{2}}{{\epsilon^{1/(2\alpha)}}},\frac{{\tilde{C}^{1/\alpha}}\sigma_{2,1}^{2}}{{\epsilon^{1/\alpha}}}\right)\log\left(\frac{4(T+1)d}{\delta}\right),
\end{align}
where 
$\tilde{C}=1+\frac{\tau_{F}}{2}\left(\frac{M+L_{2}+4}{2}\right)^{\alpha}$.
\end{theorem}
\begin{remark}
It is noteworthy that the assumptions in Theorem \ref{th1} suffice to obtain the sample complexities in Theorem \ref{th1_expectation} for $1\leq \alpha\leq3/2$ up to a logarithmic factor. However, the assumptions of Theorem \ref{th1} are stronger than those in Theorem \ref{th1_expectation}. 
\end{remark}

\begin{remark}
Theorem \ref{th1} implies that the sample complexity of Algorithm \ref{algorithm} for $\alpha=2$ is $\mathcal{O}\left({\log\log\log(1/\epsilon)\cdot\log(\log(1/\epsilon))}{/\epsilon}\right)$.
\end{remark}
We summarized the sample complexity of SCRN in Theorems \ref{th1_expectation} and \ref{th1} and the best-known sample complexity of SGD under gradient dominance property in Table \ref{Table_comparison_SGD_SCRN}. The optimal sample complexity of SGD under gradient dominance property with $\alpha=2$ is discussed in detail in \citep[Section 5.2]{khaled2020better}. For the general case of $\alpha\in[1,2]$, the sample complexity of SGD under gradient dominance property was derived in \citep[Theorem 10]{fontaine2021convergence}. This rate is achieved with a time-varying step-size.
In the fourth column of Table \ref{Table_comparison_SGD_SCRN}, we provide the improvement of the sample complexity of SCRN with respect to that of SGD. The improvement is $\mathcal{O}\left({\epsilon^{-{1}/{(2\alpha)}}}\right)$ for $\alpha\in[1,3/2)$ and is $\mathcal{\tilde{O}}({\epsilon^{-{2}/{\alpha}+1}})$ for $\alpha\in[3/2,2]$, which are decreasing functions of $\alpha$. The largest improvement is for $\alpha=1$ and is $\mathcal{O}({\epsilon^{-0.5}})$. 

\begin{remark}\label{remark 5}
The special case  $\alpha=2$ generalizes strong convexity. For strongly convex objectives, the dependence of sample complexity of SGD on $\tau_F$ is $\mathcal{O}(\tau_{F}^{2}/\epsilon)$ \citep[Corollary 2]{khaled2020better} 
while the sample complexity of SCRN is $\mathcal{O}(\tau^{7/4}_{F}/\epsilon)$. This is an improvement by a factor of $\tau_F^{1/4}$. In Appendix \ref{app:sim_synthetic}, we provide simulation results comparing the performance of SCRN and SGD over synthetic functions by varying $\alpha$ and $\tau_F$.
\end{remark}

\begin{remark}\label{inexact_subsolver}
In our analysis, we assumed that we have access to the exact solution of sub-problem in \eqref{sub-problem}. Although no closed-form solution nor exact solver exists for this sub-problem, there are algorithms that approximate the exact solution with high probability \cite{agarwal2016finding,carmon2016gradient}. We emphasize that solving the sub-problem in \eqref{sub-problem} requires extra computation, but extra gradient and Hessian queries are not needed. In particular, Carmon and Duchi \cite{carmon2016gradient} proposed a perturbed GD-based algorithm that returns an approximate solution $\tilde{\mathbf{\Delta}}_{t}$ such that $m_{t}(\tilde{\mathbf{\Delta}}_{t})\le_{1-\delta'} m_{t}(\mathbf{\Delta}_{t})+\epsilon'$ 
with $\mathcal{O}({\log(1/\delta')}/\epsilon')$ iterations, for any given $\delta',\epsilon'>0$. 
\end{remark}
 In Appendix \ref{append_inexact_subsolver}, we prove the following lemma:
\begin{lemma}\label{lemm22}
    Theorems \ref{th1_expectation} and \ref{th1} are still true if we were to use an inexact sub-solver which returned an approximate solution $\tilde{\mathbf{\Delta}}_{t}$ such that $\|\nabla m_{t}(\tilde{\mathbf{\Delta}}_{t})\|\le \epsilon^{1/\alpha}$ ($\epsilon^{1/\alpha}$-stationary point). Moreover, under some mild assumptions (same as those in \cite{carmon2016gradient}), a GD-based algorithm indeed returns such a solution in $\mathcal{O}({\epsilon^{-2/\alpha}})$ iterations.
\end{lemma}
\begin{remark}
In the iterations of GD-based algorithm,  instead of directly computing the Hessian matrix (which could be computationally expensive in high dimensions), we could compute Hessian-vector products by running Pearlmutter's algorithm \citep{pearlmutter1994fast}.
Thus, the total computational complexity of evaluating gradients and Hessian vector products of SCRN is in the order of $\mathcal{O}(d\epsilon^{-3/\alpha})$ while the one of SGD is $\mathcal{O}(d\epsilon^{-4/\alpha+1})$, where $d$ is the dimension of $\xv$. For $\alpha\in[1,3/2)$, it can be seen that the computational complexity of SGD is less than the one of SCRN by a factor of $\mathcal{O}(\epsilon^{-{1}/{(2\alpha)}})$. For $\alpha\in[3/2,2]$, this factor is $\mathcal{\tilde{O}}(\epsilon^{-(1-{1}/{\alpha})})$.
\end{remark}
\vspace{-3mm}
\subsection{Further improvements}
\vspace{-1mm}
In our analysis of SCRN, we set the batch sizes such that the stochastic error terms $C_g\|\nabla F(\xv_t)-\gv_{t}\|^\alpha$ and $C_H\|\nabla^2F(\xv_t)-\Hv_{t}\|^{2\alpha}$ in \eqref{eq_recursion_ineq} are in the order of $\epsilon$ (either in expectation or with high probability). However, for $1\leq \alpha< 3/2$, it is just needed to make sure that the error terms at iteration $t$ are  $\mathcal{O}(t^{-(2\alpha)/(3-2\alpha)})$, which equals the convergence rate of the function values $F(\xv_{t})-F(\xv^*)$ (see Lemma \ref{lemma:rec_eq} in Appendix \ref{app:vr SCRN}). Moreover, as stated in the following theorem, incorporating time-varying batch sizes in conjunction with variance reduction improves sample complexity results. 

\begin{assumption}\label{assump:individual:main}
We assume that $f(\xv,\xi)$ satisfies $L'_1$-average smoothness and $L'_2$-average Hessian Lipschitz continuity, i.e., $\mathbb{E}[\|\nabla f(\xv,\xi) -\nabla f(\yv,\xi)\|^{2}] \leq L'^{2}_{1}\|\xv-\yv\|^{2}$ and $\mathbb{E}[\|\nabla^2 f(\xv,\xi)-\nabla^2 f(\yv,\xi)\|^{2}] \leq L'^{2}_{2}\|\xv-\yv\|^{2}$ for all $\xv,\yv \in \mathbb{R}^d$.
\end{assumption}

\begin{theorem}\label{th_vr_SCRN_Pl_alpha=1}
Suppose that $F(\xv)$ satisfies the gradient dominance property with $\alpha=1$. Assume that Assumptions \ref{assump1},  \ref{assump2}, and \ref{assump:individual:main} for $\alpha=1$ hold. Then variance-reduced SCRN (See Algorithm \ref{algorithm2} in Appendix \ref{app:vr SCRN}) achieves $\epsilon$-global stationary point in expectation by  making $\mathcal{O}(\epsilon^{-2})$ stochastic gradients and ${\mathcal{O}}(\epsilon^{-1})$ stochastic Hessian on average, and the iteration complexity is $\mathcal{O}(\epsilon^{-1/2})$.
\end{theorem}
We will see in the next section that there are a class of objective functions in RL setting that satisfies a weak version of gradient dominance with $\alpha=1$ and Lipschitz gradient property and then we can take advantage of variance-reduced SCRN (See Remark \ref{remark_vr_scrn_RL}). 
\vspace{-2mm}
\section{SCRN under weak gradient dominance property in RL Setting}\label{sec:application in RL}
\vspace{-1.5mm}
In this section, we showcase practical relevance of our result in Section \ref{subsec_SCRN_PL_alpha_(1,3/2)} by applying SCRN to model-free RL. Specifically, in Theorem \ref{th_RL_SCRN_PL_alpha=1}, we prove that as long as the expected return satisfies the weak version of gradient dominance property (Assumption \ref{relaxed weak PL}), SCRN improves upon the best-known sample complexity of stochastic policy gradient (SPG) \citep{yuan2021general} by a factor of $\mathcal{O}(1/\sqrt{\epsilon})$ (see Appendix \ref{discussion on PL_alpha=1_RL} for the related work).


\vspace{-1mm}
\subsection{RL Setup}
\vspace{-1mm}
Consider a discrete Markov decision process (MDP) $\mathcal{M} = (\mathcal{S},\mathcal{A},P,R,\rho,\gamma)$, where $\mathcal{S}$ is the state space and $\mathcal{A}$ is the action space. $P(s'|s,a)$ denotes the probability of state transition from $s$ to $s'$ after taking action $a$ and $R(\cdot,\cdot) : \mathcal{S}\times \mathcal{A}\to [-R_{\max},R_{\max}]$ is a bounded reward function, where $R_{\max}$ is a positive scalar. $\rho$ represents the initial distribution on state space $\mathcal{S}$ and $\gamma\in (0, 1)$ is the discount factor.

The parametric policy $\pi_{\theta}$ is a probability distribution over $\mathcal{S}\times \mathcal{A}$ with  parameter $\theta\in \mathbb{R}^{d}$, and $\pi_{\theta}(a|s)$ denotes the probability of taking action $a$ at a given state $s$. Let $\tau = \{s_t , a_t \}_{t\ge0} \sim p(\tau|\pi_{\theta})$ be a trajectory generated by the policy $\pi_{\theta}$, where
$
p(\tau|\pi_{\theta}):=\rho(s_{0})\prod_{t=0}^{\infty}\pi_{\theta}(a_{t}|s_{t})P(s_{t+1}|s_{t},a_{t}).
$
The expected return of $\pi$ is defined as
$
J(\pi_{\theta}):=\mathbb{E}_{\tau\sim p(\cdot|\pi_{\theta})}\left[\sum_{t=0}^{\infty}\gamma^{t}R(s_{t},a_{t})\right].
$
       In the sequel, we consider a set of parameterized policies $\{\pi_{\theta}:\theta\in\mathbb{R}^{d}\}$, with the assumption that $\pi_{\theta}$
 is differentiable with respect to $\theta$. For ease of presentation, we denote $J(\pi_{\theta})$ by $J(\theta)$. 

The goal of policy-based RL is to find $\theta^*=\arg\max_{\theta}J(\theta)$. However, in many cases, $J(\theta)$ is a non-concave function and instead we settle for obtaining an $\epsilon$-FOSP,  $\hat{\theta}$, such that $\|\nabla J(\hat{\theta})\|\leq \epsilon$.
It can be shown that:
$
  \nabla J(\theta)=\mathbb{E}\left[\sum_{h=0}^{\infty} \left(\sum_{t=h}^{\infty} \gamma^t R(s_t,a_t)\right) \nabla \log\pi_{\theta}(a_h|s_h)\right]$.
In practice, the full gradient cannot be computed due to the infinite horizon length. Instead, it is commonly truncated to a length $\mathsf{H}$ horizon as follows.
    $\nabla J_{\mathsf{H}}(\theta)=\mathbb{E}[\sum_{h=0}^{\mathsf{H}-1} \Psi _h(\tau) \nabla \log\pi_{\theta}(a_h|s_h)]$,
where $\Psi_h(\tau)=\sum_{t=h}^{\mathsf{H}-1} \gamma^t R(s_t,a_t)$. 

Assume that we sample $m$ trajectories $\tau^{i}=\{s^{i}_{t},a^{i}_{t}\}_{t\ge0}$, $1\le i\le m$, and then compute $\hat{\nabla}_{m}J(\theta)=\frac{1}{m}\sum_{i=1}^{m}\sum_{h=0}^{\mathsf{H}-1} \Psi _h(\tau^{i}) \nabla \log\pi_{\theta}(a^{i}_h|s^{i}_h)$, which is an unbiased estimator for $\nabla J_{\mathsf{H}}(\theta)$. The vanilla SPG method is based on the following update: $\theta\gets \theta+\eta \hat{\nabla}_{m}J(\theta),$ where $\eta$ is the learning rate. 

It can be shown that the Hessian matrix of $J_{\mathsf{H}}(\theta)$ can be obtained as follows \citep[Appendix 7.2]{shen2019hessian}:
$
    \nabla^2 J_{\mathsf{H}}(\theta)=\mathbb{E}[\nabla \Phi(\theta;\tau)\nabla \log p(\tau|\pi_{\theta})^T+ \nabla^2\Phi(\theta;\tau)],$
where $\Phi(\theta;\tau)=\sum_{h=0}^{\mathsf{H}-1}\sum_{t=h}^{\mathsf{H}-1} \gamma^t r(s_t,a_t) \log \pi_{\theta}(a_h|s_h)$. As a result, for trajectories $\tau^{i}=\{s^{i}_{t},a^{i}_{t}\}_{t\ge0}$, $1\le i\le m$, $\hat{\nabla}^{2}_{m}J(\theta)=\frac{1}{m}\sum_{i=1}^{m}\nabla \Phi(\theta;\tau^{i})\nabla \log p(\tau^{i}|\pi_{\theta})^T+ \nabla^2\Phi(\theta;\tau^{i})$ is an unbiased estimator of Hessian matrix $\nabla^{2}J_{\mathsf{H}}(\theta)$.

\vspace{-2mm}
\subsection{Sample Complexity of SCRN}
In our analysis, we consider the recently introduced relaxed weak gradient dominance property with $\alpha=1$ \citep{yuan2021general}.
\begin{assumption}[Weak gradient dominance property with $\alpha=1$]\label{relaxed weak PL}
    $J$ satisfies the weak gradient dominance property if for all $\theta\in \mathbb{R}^{d}$, there exist $\tau_{J}>0$ and $\epsilon'>0$ such that
    \begin{align}
        \epsilon'+\tau_{J}\|\nabla J(\theta)\|\ge J^{*}-J(\theta),
    \end{align}
    where $J^*:=\max_{\theta} J(\theta)$.
\end{assumption}
\begin{remark}\label{motive_for_weak_PL_alpha=1}
Two commonly made assumptions in the literature: non-degenerate Fisher matrix \citep{liu2020improved,ding2021global} and transferred compatible function approximation error \citep{wang2019neural,agarwal2021theory} imply Assumption \ref{relaxed weak PL}. See Appendix \ref{discussion on PL_alpha=1_RL} for more details.

\end{remark}

Further, we also make Lipschitz and smooth policy (LS) assumptions which are widely adopted in the analysis of vanilla policy gradient (PG) \citep{zhang2020global} as well as variance-reduced PG methods, e.g. in \citep{shen2019hessian}.
\begin{assumption}[LS]\label{LS} There exist constants $G_1,G_2>0$ such that for every state $s\in\mathcal{S}$, the gradient and Hessian of $\log\pi_{\theta}(\cdot|s)$ satisfy $\|\nabla_{\theta}\log\pi_{\theta}(a|s)\|\le G_1$ and $\|\nabla^{2}_{\theta}\log\pi_{\theta}(a|s)\|\le G_2$.

\end{assumption}
\begin{lemma}\label{LS_yields_truncation}
Under Assumption \ref{LS}, we have $\|\nabla J(\theta)-\nabla J_{\mathsf{H}}(\theta)\|\le D_{g}\gamma^{\mathsf{H}}$ and $\|\nabla^{2} J(\theta)-\nabla^{2} J_{\mathsf{H}}(\theta)\|\le D_{H}\gamma^{\mathsf{H}},$ where $D_{g}=\frac{G_1R_{\max}}{1-\gamma}\sqrt{\frac{1}{1-\gamma}+\mathsf{H}}$ and $D_{H}=\frac{R_{\max}(G_2+G_{1}^{2})}{1-\gamma}\left(\mathsf{H}+\frac{1}{1-\gamma}\right)$. 
\end{lemma}

\begin{assumption}[Lipschitz Hessian]\label{Lip Hes}
There exists a constant $\bar{L}_{2}$ such that the Hessian of $\log\pi_{\theta}(a|s)$ satisfies
\begin{align}\label{Lip_Hessian_eq}
    \|\nabla^{2}\log\pi_{\theta}(a|s)-\nabla^{2}\log\pi_{\theta'}(a|s)\|\le \bar{L}_{2}\|\theta-\theta'\|.
\end{align}
\end{assumption}
The Lipschitz Hessian assumption is commonly used to find SOSP in policy gradient algorithms \citep{yang2020sample}. For the Gaussian policy \eqref{gaussian_policy}, $\nabla^{2}\log\pi_{\theta}(a|s)$ reduces to the matrix $-\phi(s)\phi(s)^{T}/\sigma^{2}$, which is a constant function of $\theta$ and thus satisfies condition \eqref{Lip_Hessian_eq}.
Soft-max policy also satisfies this assumption (See appendix \ref{soft-max_lip_hessian}).

\begin{theorem}\label{th_RL_SCRN_PL_alpha=1}
For a policy $\pi_{\theta}$ satisfying Assumptions \ref{LS}, \ref{Lip Hes}, and the corresponding objective function $J(\theta)$ satisfying Assumption \ref{relaxed weak PL}, SCRN outputs the solution $\theta_{T}$ such that $J^{*}-\mathbb{E}[J(\theta_{T})]\le \epsilon+\epsilon'$ and the sample complexity (the number of observed state-action pairs) is: $T\times m\times \mathsf{H}=\tilde{\mathcal{O}}(\epsilon^{-2.5})$ for $\epsilon'=0$ and $T\times m\times \mathsf{H}=\tilde{\mathcal{O}}(\epsilon^{-0.5}\epsilon'^{-2})$ for $\epsilon'>0$.
\end{theorem}
\begin{remark}
 Under weak gradient dominance property with $\alpha=1$ (Assumption \ref{relaxed weak PL}), it has been shown that the sample complexity of SPG is  $\tilde{\mathcal{O}}(\epsilon^{-3})$ in case of $\epsilon'=0$ and $\tilde{\mathcal{O}}(\epsilon^{-1}\epsilon'^{-2})$ in case of $\epsilon'>0$ \citep[Theorem C.1]{yuan2021general}. Therefore, SCRN improves upon the best-known sample complexity of  SPG in both cases $\epsilon'=0$ and $\epsilon'>0$ by a factor of $\mathcal{O}(\epsilon^{-0.5})$.
\end{remark}
\begin{remark}
Having access to exact gradient and Hessian (the deterministic case), under Assumption \ref{relaxed weak PL}, PG algorithm achieves  global convergence (i.e., $J^{*}-J(\theta_{T})\le \epsilon$) with $\tilde{\mathcal{O}}(\epsilon^{-1})$ iterations \cite{yuan2021general} while CRN requires $\tilde{\mathcal{O}}(\epsilon^{-0.5})$ iterations.
\end{remark}

\begin{remark}\label{remark_vr_scrn_RL}
Under the same assumptions as in Theorem \ref{th_RL_SCRN_PL_alpha=1} and Assumption \ref{assump:individual:main} for $\alpha=1$ and bounded variance of importance sampling weights (See Assumption \ref{bounded_var_IS}), a variance-reduced version of SCRN (See Algorithm~\ref{algorithm3}) achieves global convergence (i.e., $J^{*}-\mathbb{E}[J(\theta_{T})]\le \epsilon$) with a sample complexity of $\tilde{\mathcal{O}}(\epsilon^{-2})$. See Appendix \ref{vr-scrn_RL_append} for a proof.
\end{remark}
\vspace{-4mm}
\section{Experiments}
\vspace{-2mm}






In this section, we evaluate the performance of SCRN in the two following RL settings. First, we consider grid world environments with finite state and action spaces, and next some robotic control tasks with continuous state and action spaces.  
The details of the implementations for all methods and some additional experiments appear in the appendix and the codes are available in the supplementary material.

\textbf{Environments with finite state and action spaces:}
We consider two grid world environments in our experiments: cliff walking \citep[Example 6.6]{sutton2018reinforcement}, and random mazes \cite{mazelab}. In cliff walking, the agent's aim is to reach a goal state from a start state, avoiding a region of cells called ``cliff''. The episode is terminated if the agent enters the cliff region, or the number of steps exceeds $100$ without reaching the goal.
Moreover, we consider a soft-max tabular policy in the experiments of this part. Indeed, it has been shown that a variant of gradient dominance property with $\alpha=1$ holds for soft-max tabular policy in environments with finite state and action spaces \cite{mei2020global}. 

\begin{figure*}[t]
    \centering
    \includegraphics[width=0.7\textwidth]{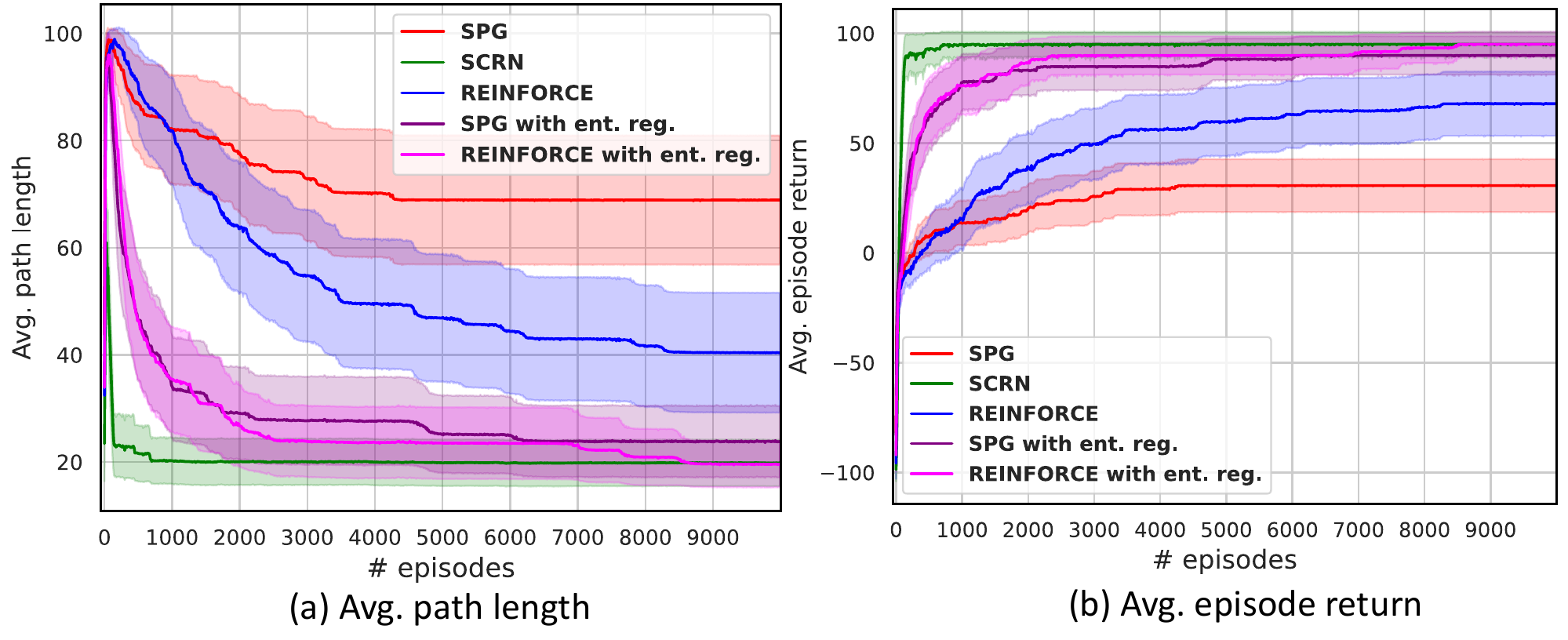}
    \caption{Comparison of SCRN with first-order methods in cliff walking environment.  The percentages of successful instances for SPG, SCRN, REINFORCE, SPG with entropy regularization, and REINFORCE with entropy regularization are $32.8\%$, $100\%$, $54.7\%$, $100\%$, and $92.2\%$, respectively.
    }
    \label{fig:cliff}
\end{figure*}
\vspace{-1.5mm}
\begin{figure}
    \centering
    \includegraphics[width=0.75\textwidth]{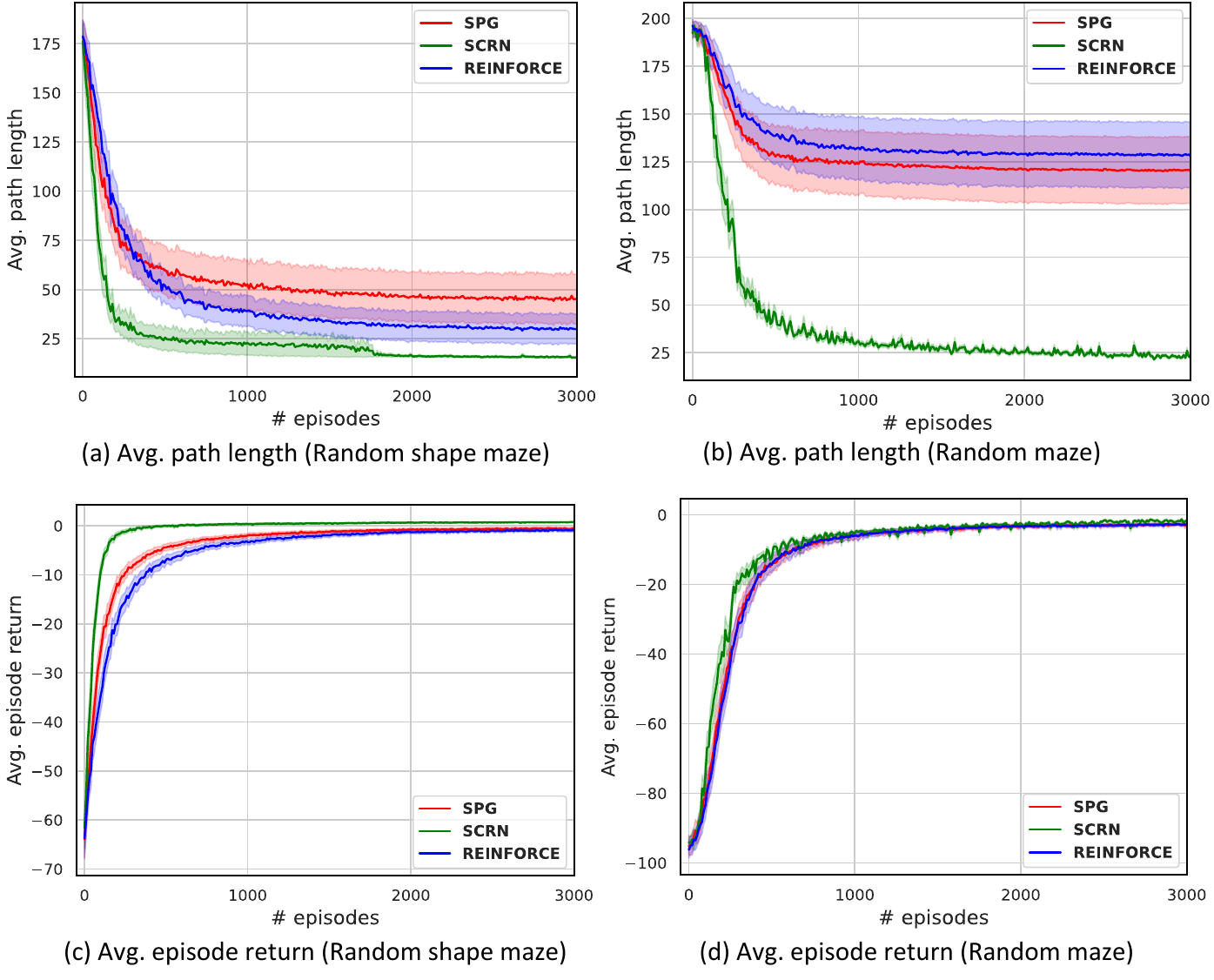}
    \caption{Comparison of SCRN with first-order methods in maze environments. In random shape maze, the percentages of successful instances for SPG, SCRN, and REINFORCE are $86\%$, $100\%$, and $95.3\%$, respectively. In random maze, the percentages of successful instances for SPG, SCRN, and REINFORCE are $45.3\%$, $97\%$, and $40.6\%$ respectively.}
    \label{fig:maze}
    \vspace{-1.5mm}
\end{figure}
We compare SCRN with two existing first-order methods: vanilla SPG and REINFORCE \cite{williams1992simple}. For the first-order methods, we use a time-varying learning rate and tune the parameters. To improve the performance of the first-order methods, we also add an entropy regularization term to the reward function.

In Fig. \ref{fig:cliff}, the average length of paths traversed by the agent and the average episode return are depicted against the number of episodes for each method. The results are averaged over 64 instances and the shaded region shows the $90\%$ confidence interval. 
As can be seen in Fig.~\ref{fig:cliff} (a), all the algorithms have a phase at the beginning during which the agent falls off the cliff in most episodes and the average length of paths are small. Then, the agent learns to avoid the cliff but could still not reach the goal in most cases. Finally, it finds a path to the goal and tries to reduce the path length. Note that SCRN finds the path very quickly and significantly outperforms the other algorithms. In fact, SPG and REINFORCE get stuck in the start state for some period of time, while SCRN easily escapes the flat plateau (for more details, please see the demonstrations in the supplementary file). The performance of SPG and REINFORCE  improves with entropy regularization, but SCRN still outperforms them. Moreover, SPG and REINFORCE fail to reach the goal in some instances while SCRN is successful in almost all instances.  In the captions of figures, for each algorithm, we also provide the percentage of instances in which the agent reached the goal. Please note that these percentages are obtained based on the parameters from the last update of each algorithm. 
\begin{figure*}[t]
\vspace{-1.5mm}
    \centering
    \includegraphics[width=0.75\textwidth]{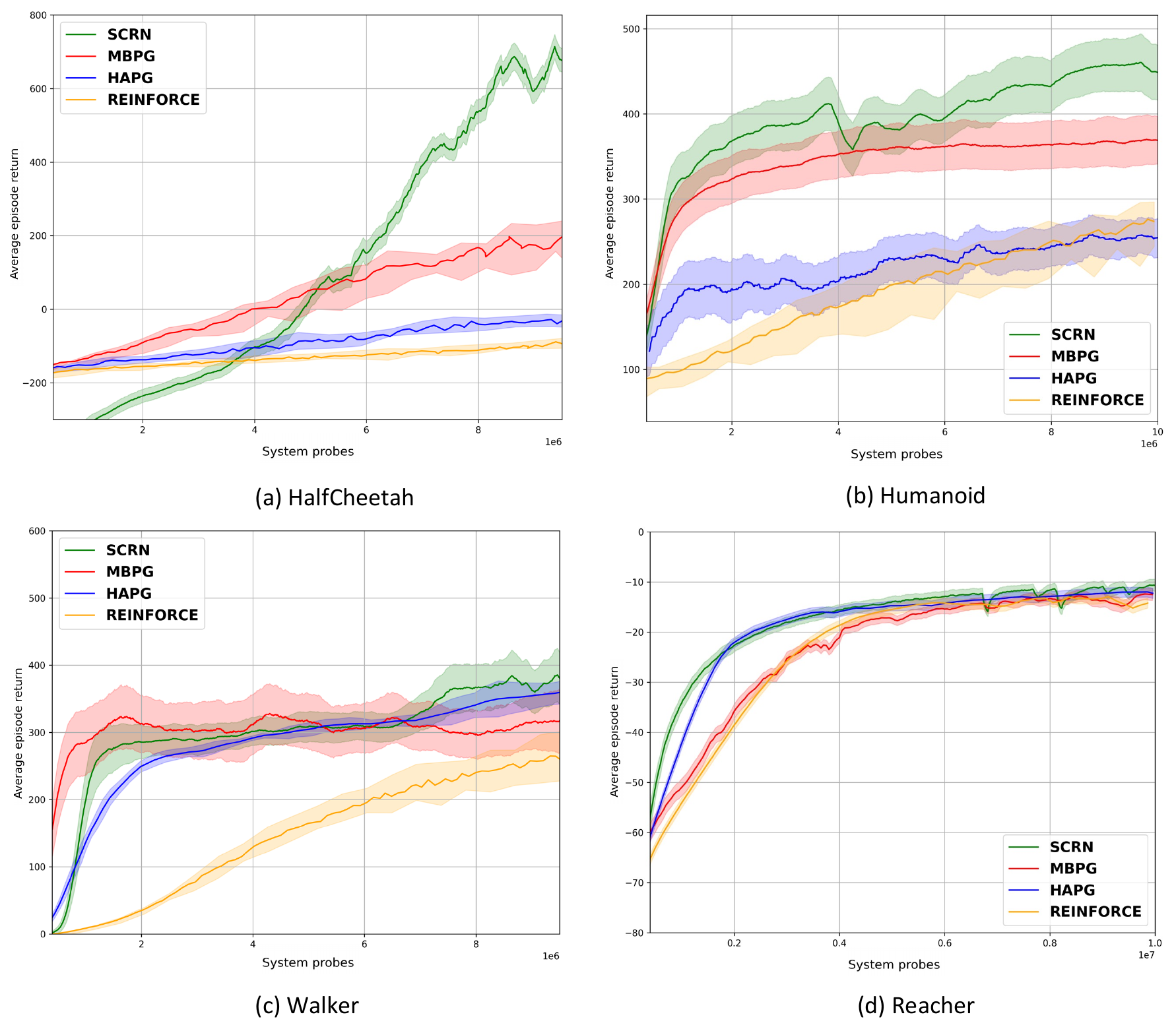}
    \caption{Comparison of SCRN with REINFORCE and variance-reduced SPG methods in MuJoCo environments.}
    \label{fig:MuJoCo}
    \vspace{-2mm}
\end{figure*}


Additionally, we studied the performance of the three aforementioned algorithms on a random maze and a random shape maze \cite{mazelab}. 
In the random shape maze, random shape blocks are placed on a grid and the agent tries to reach the goal state finding the shortest path. 
As shown in Fig. \ref{fig:maze}, SCRN again outperforms the first-order methods in both environments. In Appendix \ref{app:further resutls}, we also provided results for SPG and REINFROCE with entropy regularization which have slightly better performance.

\textbf{Environments with continuous state and action spaces:}
We consider the following control tasks in MuJoCo simulator \cite{todorov2012mujoco}: Walker, Humanoid, Reacher, and HalfCheetah. We compare SCRN with  first-order methods such as REINFORCE, and two state-of-the-art representatives of variance-reduced PG methods, HAPG \cite{shen2019hessian} and MBPG \cite{huang2020momentum}, both with guaranteed convergence to $\epsilon$-FOSP in general non-convex settings. HAPG uses second order information (Hessian vector products) for variance reduction and MBPG is a recent work based on STORM, a batch-free state-of-the-art variance reduction approach \cite{cutkosky2019momentum}.

We report average episode return against system probes as our performance measure. That is, the number of observed state-action pairs (see Fig. \ref{fig:MuJoCo}).
At each point, we run the trained policy $10$ times and compute the empirical estimate of the mean and the $90\%$ confidence interval of the episode return. As seen in Fig. \ref{fig:MuJoCo}, SCRN outperforms the other methods, especially in more complex environments such as HalfCheetah and Humanoid. In Appendix \ref{app:further resutls}, we also provided results for variance-reduced SCRN which improves upon SCRN in Humanoid and Reacher environments.
\vspace{-4mm}
\section{Conclusion}
\vspace{-3mm}
We studied the performance of SCRN for objectives satisfying the gradient dominance property for $1\leq \alpha\leq2$, which holds in various machine learning applications. We showed that SCRN improves the best-known sample complexity of SGD. The largest improvement is in the case of $\alpha=1$. Moreover, for $\alpha=1$, the average sample complexity of SCRN can be reduced to $\mathcal{O}(\epsilon^{-2})$ by utilizing a variance reduction method with time-varying batch sizes. A weak version of gradient dominance for $\alpha=1$ is satisfied in some policy-based RL settings. In the RL setting, we showed that SCRN achieves the same improvement over SPG under the weak version of gradient dominance property for $\alpha=1$. 
\vspace{-0.5cm}

\section{Acknowledgement}
\vspace{-0.3cm}
The authors would like to thank Mohammadsadegh Khorasani for conducting the experiments for continuous state-action environments. This research is partially supported by National Centre of Competence in Research (NCCR), grant agreement no. 51NF40\_180545.

\bibliographystyle{plainnat}
\bibliography{reference}

\begin{thebibliography}{51}
\providecommand{\natexlab}[1]{#1}
\providecommand{\url}[1]{\texttt{#1}}
\expandafter\ifx\csname urlstyle\endcsname\relax
  \providecommand{\doi}[1]{doi: #1}\else
  \providecommand{\doi}{doi: \begingroup \urlstyle{rm}\Url}\fi

\bibitem[Agarwal et~al.(2021)Agarwal, Kakade, Lee, and
  Mahajan]{agarwal2021theory}
Alekh Agarwal, Sham~M Kakade, Jason~D Lee, and Gaurav Mahajan.
\newblock On the theory of policy gradient methods: Optimality, approximation,
  and distribution shift.
\newblock \emph{Journal of Machine Learning Research}, 22\penalty0
  (98):\penalty0 1--76, 2021.

\bibitem[Agarwal et~al.(2016)Agarwal, Allen-Zhu, Bullins, Hazan, and
  Ma]{agarwal2016finding}
Naman Agarwal, Zeyuan Allen-Zhu, Brian Bullins, Elad Hazan, and Tengyu Ma.
\newblock Finding approximate local minima for nonconvex optimization in linear
  time.
\newblock \emph{arXiv preprint arXiv:1611.01146}, 2016.

\bibitem[Arjevani et~al.(2020)Arjevani, Carmon, Duchi, Foster, Sekhari, and
  Sridharan]{arjevani2020second}
Yossi Arjevani, Yair Carmon, John~C Duchi, Dylan~J Foster, Ayush Sekhari, and
  Karthik Sridharan.
\newblock Second-order information in non-convex stochastic optimization: Power
  and limitations.
\newblock In \emph{Conference on Learning Theory}, pages 242--299. PMLR, 2020.

\bibitem[Carmon and Duchi(2016)]{carmon2016gradient}
Yair Carmon and John~C Duchi.
\newblock Gradient descent efficiently finds the cubic-regularized non-convex
  newton step.
\newblock \emph{arXiv preprint arXiv:1612.00547}, 2016.

\bibitem[Cartis et~al.(2011)Cartis, Gould, and Toint]{cartis2011adaptive}
Coralia Cartis, Nicholas~IM Gould, and Philippe~L Toint.
\newblock Adaptive cubic regularisation methods for unconstrained optimization.
  part i: motivation, convergence and numerical results.
\newblock \emph{Mathematical Programming}, 127\penalty0 (2):\penalty0 245--295,
  2011.

\bibitem[Cayci et~al.(2021)Cayci, He, and Srikant]{cayci2021linear}
Semih Cayci, Niao He, and R~Srikant.
\newblock Linear convergence of entropy-regularized natural policy gradient
  with linear function approximation.
\newblock \emph{arXiv preprint arXiv:2106.04096}, 2021.

\bibitem[Chen et~al.(2011)Chen, Gittens, and Tropp]{chen2011masked}
RY~Chen, A~Gittens, and JA~Tropp.
\newblock The masked sample covariance estimator: An analysis via matrix
  concentration inequalities. arxiv e-print.
\newblock \emph{arXiv preprint arXiv:1109.1637}, 2011.

\bibitem[Cutkosky and Orabona(2019)]{cutkosky2019momentum}
Ashok Cutkosky and Francesco Orabona.
\newblock Momentum-based variance reduction in non-convex sgd.
\newblock \emph{Advances in neural information processing systems}, 32, 2019.

\bibitem[Ding et~al.(2021)Ding, Zhang, and Lavaei]{ding2021global}
Yuhao Ding, Junzi Zhang, and Javad Lavaei.
\newblock On the global convergence of momentum-based policy gradient.
\newblock \emph{arXiv preprint arXiv:2110.10116}, 2021.

\bibitem[Fontaine et~al.(2021)Fontaine, De~Bortoli, and
  Durmus]{fontaine2021convergence}
Xavier Fontaine, Valentin De~Bortoli, and Alain Durmus.
\newblock Convergence rates and approximation results for sgd and its
  continuous-time counterpart.
\newblock In \emph{Conference on Learning Theory}, pages 1965--2058. PMLR,
  2021.

\bibitem[Foster et~al.(2018)Foster, Sekhari, and Sridharan]{foster2018uniform}
Dylan~J Foster, Ayush Sekhari, and Karthik Sridharan.
\newblock Uniform convergence of gradients for non-convex learning and
  optimization.
\newblock \emph{Advances in Neural Information Processing Systems}, 31, 2018.

\bibitem[Hardt and Ma(2016)]{hardt2016identity}
Moritz Hardt and Tengyu Ma.
\newblock Identity matters in deep learning.
\newblock \emph{arXiv preprint arXiv:1611.04231}, 2016.

\bibitem[Huang et~al.(2020)Huang, Gao, Pei, and Huang]{huang2020momentum}
Feihu Huang, Shangqian Gao, Jian Pei, and Heng Huang.
\newblock Momentum-based policy gradient methods.
\newblock In \emph{International Conference on Machine Learning}, pages
  4422--4433. PMLR, 2020.

\bibitem[Kakade and Langford(2002)]{kakade2002approximately}
Sham Kakade and John Langford.
\newblock Approximately optimal approximate reinforcement learning.
\newblock In \emph{In Proc. 19th International Conference on Machine Learning}.
  Citeseer, 2002.

\bibitem[Karimi et~al.(2016)Karimi, Nutini, and Schmidt]{karimi2016linear}
Hamed Karimi, Julie Nutini, and Mark Schmidt.
\newblock Linear convergence of gradient and proximal-gradient methods under
  the polyak-{\l}ojasiewicz condition.
\newblock In \emph{Joint European Conference on Machine Learning and Knowledge
  Discovery in Databases}, pages 795--811. Springer, 2016.

\bibitem[Khaled and Richt{\'a}rik(2020)]{khaled2020better}
Ahmed Khaled and Peter Richt{\'a}rik.
\newblock Better theory for sgd in the nonconvex world.
\newblock \emph{arXiv preprint arXiv:2002.03329}, 2020.

\bibitem[Kohler and Lucchi(2017)]{kohler2017sub}
Jonas~Moritz Kohler and Aurelien Lucchi.
\newblock Sub-sampled cubic regularization for non-convex optimization.
\newblock In \emph{International Conference on Machine Learning}, pages
  1895--1904. PMLR, 2017.

\bibitem[Li and Yuan(2017)]{li2017convergence}
Yuanzhi Li and Yang Yuan.
\newblock Convergence analysis of two-layer neural networks with relu
  activation.
\newblock \emph{Advances in neural information processing systems}, 30, 2017.

\bibitem[Liu et~al.(2014)Liu, Wright, R{\'e}, Bittorf, and
  Sridhar]{liu2014asynchronous}
Ji~Liu, Steve Wright, Christopher R{\'e}, Victor Bittorf, and Srikrishna
  Sridhar.
\newblock An asynchronous parallel stochastic coordinate descent algorithm.
\newblock In \emph{International Conference on Machine Learning}, pages
  469--477. PMLR, 2014.

\bibitem[Liu et~al.(2020)Liu, Zhang, Basar, and Yin]{liu2020improved}
Yanli Liu, Kaiqing Zhang, Tamer Basar, and Wotao Yin.
\newblock An improved analysis of (variance-reduced) policy gradient and
  natural policy gradient methods.
\newblock \emph{Advances in Neural Information Processing Systems},
  33:\penalty0 7624--7636, 2020.

\bibitem[Mackey et~al.(2014)Mackey, Jordan, Chen, Farrell, and
  Tropp]{mackey2014matrix}
Lester Mackey, Michael~I Jordan, Richard~Y Chen, Brendan Farrell, and Joel~A
  Tropp.
\newblock Matrix concentration inequalities via the method of exchangeable
  pairs.
\newblock \emph{The Annals of Probability}, 42\penalty0 (3):\penalty0 906--945,
  2014.

\bibitem[Mei et~al.(2020)Mei, Xiao, Szepesvari, and Schuurmans]{mei2020global}
Jincheng Mei, Chenjun Xiao, Csaba Szepesvari, and Dale Schuurmans.
\newblock On the global convergence rates of softmax policy gradient methods.
\newblock In \emph{International Conference on Machine Learning}, pages
  6820--6829. PMLR, 2020.

\bibitem[Necoara et~al.(2019)Necoara, Nesterov, and Glineur]{necoara2019linear}
Ion Necoara, Yu~Nesterov, and Francois Glineur.
\newblock Linear convergence of first order methods for non-strongly convex
  optimization.
\newblock \emph{Mathematical Programming}, 175\penalty0 (1):\penalty0 69--107,
  2019.

\bibitem[Nesterov(2003)]{nesterov2003introductory}
Yurii Nesterov.
\newblock \emph{Introductory lectures on convex optimization: A basic course},
  volume~87.
\newblock Springer Science \& Business Media, 2003.

\bibitem[Nesterov and Polyak(2006)]{nesterov2006cubic}
Yurii Nesterov and Boris~T Polyak.
\newblock Cubic regularization of newton method and its global performance.
\newblock \emph{Mathematical Programming}, 108\penalty0 (1):\penalty0 177--205,
  2006.

\bibitem[Nguyen et~al.(2019)Nguyen, Nguyen, and van Dijk]{nguyen2019tight}
Phuong~Ha Nguyen, Lam Nguyen, and Marten van Dijk.
\newblock Tight dimension independent lower bound on the expected convergence
  rate for diminishing step sizes in sgd.
\newblock \emph{Advances in Neural Information Processing Systems}, 32, 2019.

\bibitem[Pearlmutter(1994)]{pearlmutter1994fast}
Barak~A Pearlmutter.
\newblock Fast exact multiplication by the hessian.
\newblock \emph{Neural computation}, 6\penalty0 (1):\penalty0 147--160, 1994.

\bibitem[Polyak(1963)]{polyak1963gradient}
Boris~Teodorovich Polyak.
\newblock Gradient methods for minimizing functionals.
\newblock \emph{Zhurnal Vychislitel'noi Matematiki i Matematicheskoi Fiziki},
  3\penalty0 (4):\penalty0 643--653, 1963.

\bibitem[Shen et~al.(2019)Shen, Ribeiro, Hassani, Qian, and
  Mi]{shen2019hessian}
Zebang Shen, Alejandro Ribeiro, Hamed Hassani, Hui Qian, and Chao Mi.
\newblock Hessian aided policy gradient.
\newblock In \emph{International conference on machine learning}, pages
  5729--5738. PMLR, 2019.

\bibitem[Song et~al.(2019)Song, Liu, and Jiang]{song2019inexact}
Chaobing Song, Ji~Liu, and Yong Jiang.
\newblock Inexact proximal cubic regularized newton methods for convex
  optimization.
\newblock \emph{arXiv preprint arXiv:1902.02388}, 2019.

\bibitem[Sutton and Barto(2018)]{sutton2018reinforcement}
Richard~S Sutton and Andrew~G Barto.
\newblock \emph{Reinforcement learning: An introduction}.
\newblock MIT press, 2018.

\bibitem[Todorov et~al.(2012)Todorov, Erez, and Tassa]{todorov2012mujoco}
Emanuel Todorov, Tom Erez, and Yuval Tassa.
\newblock Mujoco: A physics engine for model-based control.
\newblock In \emph{2012 IEEE/RSJ international conference on intelligent robots
  and systems}, pages 5026--5033. IEEE, 2012.

\bibitem[Tripuraneni et~al.(2018)Tripuraneni, Stern, Jin, Regier, and
  Jordan]{tripuraneni2018stochastic}
Nilesh Tripuraneni, Mitchell Stern, Chi Jin, Jeffrey Regier, and Michael~I
  Jordan.
\newblock Stochastic cubic regularization for fast nonconvex optimization.
\newblock \emph{Advances in neural information processing systems}, 31, 2018.

\bibitem[Wainwright(2019)]{wainwright2019high}
Martin~J Wainwright.
\newblock \emph{High-dimensional statistics: A non-asymptotic viewpoint},
  volume~48.
\newblock Cambridge University Press, 2019.

\bibitem[Wang et~al.(2019{\natexlab{a}})Wang, Cai, Yang, and
  Wang]{wang2019neural}
Lingxiao Wang, Qi~Cai, Zhuoran Yang, and Zhaoran Wang.
\newblock Neural policy gradient methods: Global optimality and rates of
  convergence.
\newblock \emph{arXiv preprint arXiv:1909.01150}, 2019{\natexlab{a}}.

\bibitem[Wang et~al.(2019{\natexlab{b}})Wang, Zhou, Liang, and
  Lan]{wang2019stochastic}
Zhe Wang, Yi~Zhou, Yingbin Liang, and Guanghui Lan.
\newblock Stochastic variance-reduced cubic regularization for nonconvex
  optimization.
\newblock In \emph{The 22nd International Conference on Artificial Intelligence
  and Statistics}, pages 2731--2740. PMLR, 2019{\natexlab{b}}.

\bibitem[Williams(1992)]{williams1992simple}
Ronald~J Williams.
\newblock Simple statistical gradient-following algorithms for connectionist
  reinforcement learning.
\newblock \emph{Machine learning}, 8\penalty0 (3):\penalty0 229--256, 1992.

\bibitem[Xu et~al.(2020{\natexlab{a}})Xu, Gao, and Gu]{xu2020improved}
Pan Xu, Felicia Gao, and Quanquan Gu.
\newblock An improved convergence analysis of stochastic variance-reduced
  policy gradient.
\newblock In \emph{Uncertainty in Artificial Intelligence}, pages 541--551.
  PMLR, 2020{\natexlab{a}}.

\bibitem[Xu et~al.(2020{\natexlab{b}})Xu, Roosta, and Mahoney]{xu2020second}
Peng Xu, Fred Roosta, and Michael~W Mahoney.
\newblock Second-order optimization for non-convex machine learning: An
  empirical study.
\newblock In \emph{Proceedings of the 2020 SIAM International Conference on
  Data Mining}, pages 199--207. SIAM, 2020{\natexlab{b}}.

\bibitem[Yang et~al.(2020)Yang, Zheng, and Pan]{yang2020sample}
Long Yang, Qian Zheng, and Gang Pan.
\newblock Sample complexity of policy gradient finding second-order stationary
  points.
\newblock \emph{arXiv preprint arXiv:2012.01491}, 2020.

\bibitem[Yuan et~al.(2021)Yuan, Gower, and Lazaric]{yuan2021general}
Rui Yuan, Robert~M Gower, and Alessandro Lazaric.
\newblock A general sample complexity analysis of vanilla policy gradient.
\newblock \emph{arXiv preprint arXiv:2107.11433}, 2021.

\bibitem[Zhang et~al.(2013)Zhang, Jiang, and Luo]{zhang2013linear}
Haibin Zhang, Jiaojiao Jiang, and Zhi-Quan Luo.
\newblock On the linear convergence of a proximal gradient method for a class
  of nonsmooth convex minimization problems.
\newblock \emph{Journal of the Operations Research Society of China},
  1\penalty0 (2):\penalty0 163--186, 2013.

\bibitem[Zhang et~al.(2020{\natexlab{a}})Zhang, Koppel, Bedi, Szepesvari, and
  Wang]{zhang2020variational}
Junyu Zhang, Alec Koppel, Amrit~Singh Bedi, Csaba Szepesvari, and Mengdi Wang.
\newblock Variational policy gradient method for reinforcement learning with
  general utilities.
\newblock \emph{Advances in Neural Information Processing Systems},
  33:\penalty0 4572--4583, 2020{\natexlab{a}}.

\bibitem[Zhang et~al.(2021)Zhang, Ni, Szepesvari, Wang,
  et~al.]{zhang2021convergence}
Junyu Zhang, Chengzhuo Ni, Csaba Szepesvari, Mengdi Wang, et~al.
\newblock On the convergence and sample efficiency of variance-reduced policy
  gradient method.
\newblock \emph{Advances in Neural Information Processing Systems}, 34, 2021.

\bibitem[Zhang et~al.(2022)Zhang, Xiao, and Zhang]{zhang2022adaptive}
Junyu Zhang, Lin Xiao, and Shuzhong Zhang.
\newblock Adaptive stochastic variance reduction for subsampled newton method
  with cubic regularization.
\newblock \emph{INFORMS Journal on Optimization}, 4\penalty0 (1):\penalty0
  45--64, 2022.

\bibitem[Zhang et~al.(2020{\natexlab{b}})Zhang, Koppel, Zhu, and
  Basar]{zhang2020global}
Kaiqing Zhang, Alec Koppel, Hao Zhu, and Tamer Basar.
\newblock Global convergence of policy gradient methods to (almost) locally
  optimal policies.
\newblock \emph{SIAM Journal on Control and Optimization}, 58\penalty0
  (6):\penalty0 3586--3612, 2020{\natexlab{b}}.

\bibitem[Zhou and Gu(2020)]{zhou2020stochastic}
Dongruo Zhou and Quanquan Gu.
\newblock Stochastic recursive variance-reduced cubic regularization methods.
\newblock In \emph{International Conference on Artificial Intelligence and
  Statistics}, pages 3980--3990. PMLR, 2020.

\bibitem[Zhou et~al.(2018{\natexlab{a}})Zhou, Xu, and Gu]{zhou2018stochastic}
Dongruo Zhou, Pan Xu, and Quanquan Gu.
\newblock Stochastic variance-reduced cubic regularized newton methods.
\newblock In \emph{International Conference on Machine Learning}, pages
  5990--5999. PMLR, 2018{\natexlab{a}}.

\bibitem[Zhou et~al.(2019)Zhou, Xu, and Gu]{zhou2019stochastic}
Dongruo Zhou, Pan Xu, and Quanquan Gu.
\newblock Stochastic variance-reduced cubic regularization methods.
\newblock \emph{J. Mach. Learn. Res.}, 20\penalty0 (134):\penalty0 1--47, 2019.

\bibitem[Zhou et~al.(2018{\natexlab{b}})Zhou, Wang, and
  Liang]{zhou2018convergence}
Yi~Zhou, Zhe Wang, and Yingbin Liang.
\newblock Convergence of cubic regularization for nonconvex optimization under
  kl property.
\newblock \emph{Advances in Neural Information Processing Systems}, 31,
  2018{\natexlab{b}}.

\bibitem[Zuo(2018)]{mazelab}
Xingdong Zuo.
\newblock mazelab: A customizable framework to create maze and gridworld
  environments.
\newblock \url{https://github.com/zuoxingdong/mazelab}, 2018.

\end{thebibliography}

\newpage
\section*{Checklist}


\begin{enumerate}

\item For all authors...
\begin{enumerate}
  \item Do the main claims made in the abstract and introduction accurately reflect the paper's contributions and scope?
    \answerYes{}
  \item Did you describe the limitations of your work?
    \answerYes{In all the statements in the paper (theorems, lemmas,...),  we explicitly mentioned what are the assumptions under which the statements holds.}
  \item Did you discuss any potential negative societal impacts of your work?
    \answerNo{This is a theoretical work and it does not have any meaningful negative societal impacts.}
  \item Have you read the ethics review guidelines and ensured that your paper conforms to them?
    \answerYes{}
\end{enumerate}

\item If you are including theoretical results...
\begin{enumerate}
  \item Did you state the full set of assumptions of all theoretical results?
    \answerYes{}
        \item Did you include complete proofs of all theoretical results?
    \answerYes{}
\end{enumerate}

\item If you ran experiments...
\begin{enumerate}
  \item Did you include the code, data, and instructions needed to reproduce the main experimental results (either in the supplemental material or as a URL)?
    \answerYes{The codes with the instructions are available as the supplementary material.}
  \item Did you specify all the training details (e.g., data splits, hyperparameters, how they were chosen)?
    \answerYes{We discuss them in the appendix.}
        \item Did you report error bars (e.g., with respect to the random seed after running experiments multiple times)?
    \answerYes{We run our experiments on multiple instances and reported confidence intervals.}
        \item Did you include the total amount of compute and the type of resources used (e.g., type of GPUs, internal cluster, or cloud provider)?
    \answerYes{We mentioned them in the details of experiments in the appendix.}
\end{enumerate}

\item If you are using existing assets (e.g., code, data, models) or curating/releasing new assets...
\begin{enumerate}
  \item If your work uses existing assets, did you cite the creators?
    \answerYes{We cited them in the experiments section.}
  \item Did you mention the license of the assets?
    \answerNA{The codes that were used are publicly available.}
  \item Did you include any new assets either in the supplementary material or as a URL?
    \answerYes{We provided the implementations of our algorithms in a supplemental material.}
  \item Did you discuss whether and how consent was obtained from people whose data you're using/curating?
    \answerNA{}
  \item Did you discuss whether the data you are using/curating contains personally identifiable information or offensive content?
    \answerNA{}
\end{enumerate}

\item If you used crowdsourcing or conducted research with human subjects...
\begin{enumerate}
  \item Did you include the full text of instructions given to participants and screenshots, if applicable?
    \answerNA{}
  \item Did you describe any potential participant risks, with links to Institutional Review Board (IRB) approvals, if applicable?
    \answerNA{}
  \item Did you include the estimated hourly wage paid to participants and the total amount spent on participant compensation?
    \answerNA{}
\end{enumerate}

\end{enumerate}


\newpage
\appendix

\section{Appendix}
\subsection{Proofs of Section \ref{subsec_SCRN_PL_alpha_(1,3/2)}}
We first provide some lemmas and then prove Theorem \ref{th1_expectation}.

\begin{lemma}\label{lemma_norm_op_summation_of_iid_matrix}
Let $\Yv_{1},\ldots,\Yv_{N}$ be centered symmetric random $d\times d$ matrices. Then
\begin{align}
    \left(\mathbb{E}\left\|\sum_{i=1}^{n}\Yv_{i}\right\|^{p}\right)^{\frac{1}{p}}\le 2\sqrt{e\cdot\max\{p,2\log(d)\}}\left(\mathbb{E}\left\|\sum_{i=1}^{n}\Yv_{i}^{2}\right\|^{p/2}\right)^{\frac{1}{p}}
\end{align}
\end{lemma}
The proof of Lemma \ref{lemma_norm_op_summation_of_iid_matrix} is given in Appendix \ref{proof_lemma_norm_op_summation_of_iid_matrix}.

In our setting, we are going to bound $ \mathbb{E}\left[\|\Hv_{t}-\nabla^{2}F(\xv_{t})\|^{2\alpha}\right]$ for $\alpha\ge 1$ using Lemma \ref{lemma_norm_op_summation_of_iid_matrix}
\begin{align}
    &\mathbb{E}\|\Hv_{t}-\nabla^{2}F(\xv)\|^{2\alpha}= \mathbb{E}\left\|\frac{1}{n_{2}}\sum_{i=1}^{n_{2}}(\nabla^{2}f(\xv,\xi_{i})-\nabla^{2}F(\xv))\right\|^{2\alpha}\nonumber\\
    &\le 2^{2\alpha}(e\cdot \max\{2\alpha,\log d\})^{\alpha}\mathbb{E}\left\|\frac{1}{n^{2}_{2}}\sum_{i=1}^{n_{2}}(\nabla^{2}f(\xv,\xi_{i})-\nabla^{2}F(\xv))^{2}\right\|^{\alpha}\label{10028}\\
    &=\frac{2^{2\alpha}(e\cdot \max\{2\alpha,\log d\})^{\alpha}}{n^{\alpha}_{2}}\mathbb{E}\left\|\frac{1}{n_{2}}\sum_{i=1}^{n_{2}}(\nabla^{2}f(\xv,\xi_{i})-\nabla^{2}F(\xv))^{2}\right\|^{\alpha}\nonumber\\
    &\le \frac{2^{2\alpha}(e\cdot \max\{2\alpha,\log d\})^{\alpha}}{n^{\alpha}_{2}}\frac{1}{n_{2}}\sum_{i=1}^{n_{2}}\mathbb{E}\left\|(\nabla^{2}f(\xv,\xi_{i})-\nabla^{2}F(\xv))^{2}\right\|^{\alpha}\label{10029}\\
    &\le \frac{2^{2\alpha}(e\cdot \max\{2\alpha,\log d\})^{\alpha}}{n^{\alpha}_{2}}\frac{1}{n_{2}}\sum_{i=1}^{n_{2}}\mathbb{E}\left\|(\nabla^{2}f(\xv,\xi_{i})-\nabla^{2}F(\xv))\right\|^{2\alpha}\label{10029.5}\\
    &\le \frac{2^{2\alpha}(e\cdot \max\{2\alpha,\log d\})^{\alpha}}{n^{\alpha}_{2}}\cdot\sigma^{2}_{2,\alpha}\label{10030}
\end{align}
where \eqref{10028} comes from Lemma \ref{lemma_norm_op_summation_of_iid_matrix} and \eqref{10029} is derived by the convexity of $\|X\|^{\alpha}$.  \eqref{10029.5} is obtained by $\|AB\|\le \|A\|\|B\|$ for the square matrices $A,B$. The last inequality is obtained from Assumption \ref{assump2}.
\begin{lemma}
\begin{enumerate}[(i).]
    \item Assume that function $F$ satisfies Assumption \ref{assump1} (Lipschitz Hessian). Then the solution of sub-problem, ${{\bf\Delta}_{t}}$, in Algorithm \ref{algorithm} (line 7), satisfies the following conditions:

\begin{align}
 &\|\nabla F(\xv_{t}+{\bf {\Delta}}_{t})\|\le  \frac{M+L_{2}}{2}\|{\bf {\Delta}}_{t}\|^{2}+\|\nabla F(\xv_{t})-\gv_{t}\|+\frac{\|\nabla^{2}F(\xv_{t})-\Hv_{t}\|^2}{2}+\frac{\|{\bf {\Delta}}_{t}\|^2}{2}, \label{eq:1} \\
&\frac{3M-2L_2-8}{12}\|{{\bf{\Delta}_{t}}}\|^3\leq F(\xv_t)-F(\xv_{t}+{{\bf{\Delta}_{t}}})+\frac{2\|\nabla F(\xv_t)-\gv_{t}\|^{3/2}}{3}+ \frac{\|\nabla^2 F(\xv_t)-\Hv_{t}\|^{3}}{6}.\label{eq:2}
\end{align}
\item With the same assumption in (i) and using Assumption \ref{assump:4} (gradient dominance property) for $\alpha\in[1,\infty)$, we have 
\begin{align}
&F(\xv_{t}+{\bf {\Delta}}_{t})-F(\xv^*)\leq \nonumber\\
&C(F(\xv_t)-F(\xv_{t}+{\bf {\Delta}}_{t}))^{2\alpha/3}+C_g\|\nabla F(\xv_t)-\gv_{t}\|^\alpha+C_{H}\|\nabla^2F(\xv_t)-\Hv_{t}\|^{2\alpha},\label{eq:3}
\end{align}
where $C,C_g,C_{H}>0$ are some constants depending on $M,L_2,$ and $\tau_F$.
\end{enumerate}
\label{lemma:two_ineq}
\end{lemma}

\begin{proof}

Proof of \eqref{eq:1}: We know from Lipschitzness of Hessian of $F$ that
\begin{align}
\|\nabla F(\xv_{t}+{\bf {\Delta}}_{t})-\nabla F(\xv_{t})-\nabla^{2}F(\xv_{t}){\bf {\Delta}}_{t}\|\le \frac{L_{2}}{2}\|{\bf {\Delta}}_{t}\|^{2}.
\end{align}
Moreover, ${\bf {\Delta}}_{t}$ is an optimal solution of the problem in Algorithm 1 (line 10) and therefore, it satisfies the optimality condition: $\gv_{t}+\Hv_{t}{\bf {\Delta}}_{t}+\frac{M}{2}\|{\bf {\Delta}}_{t}\|{\bf {\Delta}}_{t}=0$. Therefore, we have:
\begin{align}
    \|\gv_{t}+\Hv_{t}{\bf {\Delta}}_{t}\|=\frac{M}{2}\|{\bf {\Delta}}_{t}\|^{2}.
\end{align}
By summing the above two equations and using triangle inequality, we have:
\begin{align}
    \|\nabla F(\xv_{t}+{\bf {\Delta}}_{t})\|&\le \frac{M+L_{2}}{2}\|{\bf {\Delta}}_{t}\|^{2}+\|\nabla F(\xv_{t})-\gv_{t}\|+\|(\nabla^{2}F(\xv_{t})-\Hv_{t}){\bf {\Delta}}_{t}\|\nonumber\\
    &\le \frac{M+L_{2}}{2}\|{\bf {\Delta}}_{t}\|^{2}+\|\nabla F(\xv_{t})-\gv_{t}\|+\|\nabla^{2}F(\xv_{t})-\Hv_{t}\|\|{\bf {\Delta}}_{t}\|\\
    &\leq \frac{M+L_{2}}{2}\|{\bf {\Delta}}_{t}\|^{2}+\|\nabla F(\xv_{t})-\gv_{t}\|+\frac{\|\nabla^{2}F(\xv_{t})-\Hv_{t}\|^2}{2}+\frac{\|{\bf {\Delta}}_{t}\|^2}{2},
\end{align}
where the last inequality is due to Young's inequality. 
\\Proof of \eqref{eq:2}: For the sub-problem, we have from \citep{nesterov2006cubic} that
\begin{align}\label{10023}
    \gv_{t}^{T}{\bf {\Delta}}_{t}+{\bf {\Delta}}_{t}^{T}\Hv_{t}{\bf {\Delta}}_{t}+\frac{M}{2}\|{\bf {\Delta}}_{t}\|^{3}=0,\quad \Hv_{t}+\frac{M\|{{\bf {\Delta}}_{t}}\|}{2}I_{d\times d}\succeq 0
\end{align}
which yields 
\begin{equation}\label{10024}
    \gv_{t}^{T}{\bf {\Delta}}_{t}\le0.
\end{equation}
From Lipschitzness of Hessian, we have the following:
\begin{align}
    &F(\xv_{t}+{\bf {\Delta}}_{t})\le F(\xv_{t})+\langle \nabla F(\xv_{t}), {{\bf {\Delta}}_{t}}\rangle +\frac{1}{2}\langle {{\bf {\Delta}}_{t}}, \nabla^{2}F(\xv_{t}){{\bf {\Delta}}_{t}}\rangle+\frac{L_{2}}{6}\|{\bf {\Delta}}_{t}\|^{3}\\
    &=F(\xv_{t})+\langle \nabla F(\xv_{t})-\gv_{t}, {{\bf {\Delta}}_{t}}\rangle +\frac{1}{2}\langle {{\bf {\Delta}}_{t}}, (\nabla^{2}F(\xv_{t})-\Hv_{t}){{\bf {\Delta}}_{t}}\rangle+\gv_{t}^{T}{\bf {\Delta}}_{t}+\frac{1}{2}{\bf {\Delta}}_{t}^{T}\Hv_{t}{\bf {\Delta}}_{t}+\frac{L_{2}}{6}\|{\bf {\Delta}}_{t}\|^{3}\nonumber\\
    &\overset{(a)}{\le} F(\xv_{t})+\langle \nabla F(\xv_{t})-\gv_{t}, {{\bf {\Delta}}_{t}}\rangle +\frac{1}{2}\langle {{\bf {\Delta}}_{t}}, (\nabla^{2}F(\xv_{t})-\Hv_{t}){{\bf {\Delta}}_{t}}\rangle+\frac{2L_{2}-3M}{12}\|{{\bf {\Delta}}_{t}}\|^{3}\\
    &\overset{(b)}{\le}F(\xv_{t})+\|\nabla F(\xv_{t})-\gv_{t}\|_{2} \|{{\bf {\Delta}}_{t}}\|_{2} +\frac{1}{2} \|\nabla^{2}F(\xv_{t})-\Hv_{t}\|\|{{\bf {\Delta}}_{t}}\|_{2}^{2}+\frac{2L_{2}-3M}{12}\|{{\bf {\Delta}}_{t}}\|^{3}\\
    &\overset{(c)}\leq F(\xv_t)+\frac{2\|\nabla F(\xv_t)-\gv_{t}\|^{3/2}}{3}+ \frac{\|\nabla^2 F(\xv_t)-\Hv_{t}\|^{3}}{6}+\frac{8+2L_{2}-3M}{12}\|{{\bf {\Delta}}_{t}}\|^{3},
\end{align}
where (a) comes from \eqref{10023} and \eqref{10024}, (b) comes from Cauchy-Schwartz inequality and (c) is due to Young's inequality. 
\\Proof of \eqref{eq:3}: Based on gradient dominance property and  \eqref{eq:1}, we have:
\begin{equation}\label{10044}
    \begin{split}
        &(F(\xv_{t}+{\bf {\Delta}}_{t})-F(\xv^*))\leq \tau_F\|\nabla F(\xv_t+{\bf {\Delta}}_{t})\|^\alpha\\
        &\leq \tau_F\Big(\frac{M+L_{2}+1}{2}\|{\bf {\Delta}}_{t}\|^{2}+\|\nabla F(\xv_{t})-\gv_{t}\|+\frac{\|\nabla^{2}F(\xv_{t})-\Hv_{t}\|^2}{2}\Big)^\alpha\\
        &\leq  3^{\alpha-1}\tau_F\Big[\Big(\frac{M+L_2+1}{2}\Big)^\alpha\|{\bf {\Delta}}_{t}\|^{2\alpha}+\|\nabla F(\xv_{t})-\gv_{t}\|^{\alpha}+\frac{\|\nabla^{2}F(\xv_{t})-\Hv_{t}\|^{2\alpha}}{2^{\alpha}}\Big]\\
    \end{split}
\end{equation}
where we used $(a+b+c)^\alpha\leq 3^{\alpha-1}(a^\alpha+b^\alpha+c^\alpha)$, $\forall a,b,c>0$ in the last two above inequalities.

Substituting \eqref{eq:2} into \eqref{10044}, we have:
\begin{small}
\begin{equation}
    \begin{split}
        &F(\xv_{t}+{\bf {\Delta}}_{t})-F(\xv^*)\leq\\ &3^{\alpha-1}\tau_F\Big[ \Big(\frac{M+L_2+1}{2}\Big)^\alpha\Big(\frac{12}{3M-2L_2-8}\Big)^{2\alpha/3}\Big(F(\xv_{t})-F(\xv_{t}+{\bf {\Delta}}_{t})+\frac{2\|\nabla F(\xv_t)-\gv_{t}\|^{3/2}}{3}+ \\
       &\qquad\qquad\frac{\|\nabla^2 F(\xv_t)-\Hv_{t}\|^{3}}{6}\Big)^{2\alpha/3}+ \|\nabla F(\xv_t)-\gv_{t}\|^\alpha+\frac{\|\nabla^2F(\xv_t)-\Hv_{t}\|^{2\alpha}}{2^{\alpha}}\Big]\\
     &\qquad\qquad\leq C(F(\xv_t)-F(\xv_{t}+{\bf {\Delta}}_{t}))^{2\alpha/3}+C_g\|\nabla F(\xv_t)-\gv_{t}\|^\alpha+C_{H}\|\nabla^2F(\xv_t)-\Hv_{t}\|^{2\alpha},
    \end{split}
\end{equation}
\end{small}
where in the last inequality we used $(a+b+c)^{2\alpha/3}\le3^{2\alpha/3-1/3}(a^{2\alpha/3}+b^{2\alpha/3}+c^{2\alpha/3}) $ which is derived by the following inequalities: \[(a+b+c)^{2\alpha}\leq 3^{2\alpha-1}(a^{2\alpha}+b^{2\alpha}+c^{2\alpha})\le 3^{2\alpha-1}(a^{2\alpha/3}+b^{2\alpha/3}+c^{2\alpha/3})^{3}\]
for any $a,b,c\in \mathbb{R^{+}}$ and,
\begin{equation}\label{eq20}
    \begin{split}
      C&=3^{(5\alpha-4)/3}\tau_F \Big(\frac{M+L_2+1}{2}\Big)^\alpha\Big(\frac{12}{3M-2L_2-8}\Big)^{2\alpha/3}
       \\
     C_g&=2^{2\alpha/3}\times3^{\frac{5\alpha-7}{3}}\tau_F \Big(\frac{M+L_2+1}{2}\Big)^\alpha\Big(\frac{12}{3M-2L_2-8}\Big)^{2\alpha/3}+3^{\alpha-1}\tau_{F}\\
     C_{H}&=2^{-2\alpha/3}\times3^{\frac{5\alpha-7}{3}}\tau_F \Big(\frac{M+L_2+1}{2}\Big)^\alpha\Big(\frac{12}{3M-2L_2-8}\Big)^{2\alpha/3}+3^{\alpha-1}2^{-\alpha}\tau_{F}
    \end{split}
\end{equation}
and the proof is complete.
\end{proof}
\textbf{Theorem 1.}
\textit{Let $F(\xv)$ satisfy Assumptions \ref{assump1} and \ref{assump:4} for a given $\alpha$ and the stochastic gradient and Hessian  satisfy Assumption \ref{assump2} for the same $\alpha$. Moreover, assume that an exact solver for sub-problem \eqref{sub-problem} exists. Then Algorithm \ref{algorithm} outputs a point $\xv_{T}$ such that $\mathbb{E}[F(\xv_{T})]-F(\xv^{*})\le \epsilon$ after $T$ iterations, where
    \begin{enumerate}[(i)]
        \item if $\alpha\in[1,3/2)$,
   $T=\mathcal{O}(\epsilon^{-\frac{3-2\alpha}{2\alpha}})$, with access to the following numbers of samples of the stochastic gradient and Hessian per iteration:
    \begin{align}\label{eq009}
       n_{1}\ge\frac{{C_{g}}^{2/\alpha}}{C^{6/\alpha}}\cdot\frac{2^{2/\alpha}\sigma_{1}^{2/\alpha}}{\epsilon^{2/\alpha}},\quad n_{2}\ge\frac{{C'_{H}}^{1/\alpha}}{C^{3/\alpha}}\cdot\frac{2^{1/\alpha}\sigma_{2,\alpha}^{2/\alpha}}{\epsilon^{1/\alpha}},
    \end{align}
    where $C'_{H}$ is defined in \eqref{def_C'_{H}} and depends on $\log(d)$.
    \item if $\alpha=3/2$, 
     $T=\mathcal{O}\left(\log(1/\epsilon)\right)$ with the same numbers of samples per iteration as in \eqref{eq009}.
\end{enumerate}}
\subsubsection{Proof of Theorem \ref{th1_expectation}}\label{proof_theorem_SCRN_pL_alpha=1}

\textbf{Part (i)}: By Lemma \ref{lemma:two_ineq}, we have:
\begin{align*}
       &F(\xv_{t}+{\bf {\Delta}}_{t})-F(\xv^*)\leq C(F(\xv_t)-F(\xv_{t}+{\bf {\Delta}}_{t}))^{2\alpha/3}+C_g\|\nabla F(\xv_t)-\gv_{t}\|^\alpha+C_{H}\|\nabla^2F(\xv_t)-\Hv_{t}\|^{2\alpha}
\end{align*}
Taking expectation of both sides given $\xv_{t}$
\begin{align}
&\mathbb{E}[F(\xv_{t}+{\bf {\Delta}}_{t})-F(\xv^*)]\le\nonumber\\ &C(\mathbb{E}F(\xv_t)-\mathbb{E}F(\xv_{t}+{\bf {\Delta}}_{t}))^{2\alpha/3}+C_g\mathbb{E}\|\nabla F(\xv_t)-\gv_{t}\|^\alpha+C_{H}\mathbb{E}\|\nabla^2F(\xv_t)-\Hv_{t}\|^{2\alpha}\nonumber\\
&{\leq} C(\mathbb{E}F(\xv_t)-\mathbb{E}F(\xv_{t}+{\bf {\Delta}}_{t}))^{2\alpha/3}+C_g\frac{\sigma_{1}^{\alpha}}{n_{1}^{\alpha/2}}+C_{H}\mathbb{E}\|\nabla^2F(\xv_t)-\Hv_{t}\|^{2\alpha}\label{10047}\\
&{\le} C(\mathbb{E}F(\xv_t)-\mathbb{E}F(\xv_{t}+{\bf {\Delta}}_{t}))^{2\alpha/3}+C_g\frac{\sigma_{1}^{\alpha}}{n_{1}^{\alpha/2}}+C_{H}\frac{2^{2\alpha}(e\cdot \max\{2\alpha,\log d\})^{\alpha}}{n^{\alpha}_{2}}\cdot\sigma^{2}_{2,\alpha}\label{10048}.
\end{align}
where \eqref{10047} comes from the following facts that, $$\mathbb{E}[\|\nabla F(\xv_t)-\gv_{t}\|^{\alpha}]\le (\mathbb{E}[\|\nabla F(\xv_t)-\gv_{t}\|^{2}])^{\alpha/2}\le \frac{\sigma_{1}^{\alpha}}{n_{1}^{\alpha/2}},$$
which is comes from Jensen's inequality. $\mathbb{E}[A^{2\alpha/3}]\le (\mathbb{E}[A])^{2\alpha/3}$. Inequality \eqref{10048} 
is derived by Equation \eqref{10030} for bounding the Hessian error.
Let define
\begin{align}\label{def_C'_{H}}
    C'_{H}:=C_{H}2^{2\alpha}(e\cdot \max\{2\alpha,\log d\})^{\alpha}
\end{align}
and $P(n_{1},n_{2}):=\frac{C_g\sigma_{1}}{n_{1}^{\alpha/2}}+\frac{C'_{H}\sigma_{2,\alpha}^{2}}{n_{2}^{\alpha}}$ and $\delta_{t}:=\frac{\mathbb{E}[F(\xv_{t})]-F(\xv^*)}{C^{3/(3-2\alpha)}}-\frac{P(n_{1},n_{2})}{C^{3/(3-2\alpha)}}$. Then rewrite Inequality \eqref{10048} as follows,
\begin{align}
   \delta_{t+1}\le (\delta_{t}-\delta_{t+1})^{2\alpha/3}.
\end{align}

First we show that $\delta_{t}\to 0$ when $t\to \infty$ and then with good choices of batch sizes, $P(n_{1},n_{2})\to 0$. Let $h(t):=\frac{2\alpha}{3-2\alpha}t^{1-\frac{3}{2\alpha}}$ . To show $\delta_{t}\to 0$, We have two cases: 

Case (1): Assume that $\delta_{t+1}\ge {2}^{-\frac{2\alpha}{3}}\delta_{t}$ for some $t$. Then
\begin{align}
    &h(\delta_{t+1})-h(\delta_{t})=\int_{\delta_{t}}^{\delta_{t+1}}\frac{d}{dt}h(t)dt=\int^{\delta_{t}}_{\delta_{t+1}}t^{-\frac{3}{2\alpha}}dt\ge (\delta_{t}-\delta_{t+1})\delta_{t}^{-\frac{3}{2\alpha}}\nonumber\\
    &\ge (\delta_{t}-\delta_{t+1})\frac{1}{2}\delta_{t+1}^{-\frac{3}{2\alpha}}\ge \frac{1}{2}
\end{align}
Case (2): If we have $\delta_{t+1}\le {2}^{-\frac{2\alpha}{3}}\delta_{t}$ for some $t\ge {0}$, $\delta_{t+1}^{1-\frac{3}{2\alpha}}\ge 2^{\frac{3-2\alpha}{3}}\delta_{t}^{1-\frac{3}{2\alpha}}$.
\begin{align}
    &h(\delta_{t+1})-h(\delta_{t})=\frac{2\alpha}{3-2\alpha}(\delta_{t+1}^{1-\frac{3}{2\alpha}}-\delta_{t}^{1-\frac{3}{2\alpha}})\ge \frac{2\alpha}{3-2\alpha}(2^{\frac{3-2\alpha}{3}}-1)\delta_{t}^{1-\frac{3}{2\alpha}}\nonumber\\
    &\ge \frac{2\alpha}{3-2\alpha}(2^{\frac{3-2\alpha}{3}}-1)\delta_{0}^{1-\frac{3}{2\alpha}}.
\end{align}
Let define $D:=\min\{\frac{1}{2},\frac{2\alpha}{3-2\alpha}(2^{\frac{3-2\alpha}{3}}-1)\delta_{0}^{1-\frac{3}{2\alpha}}\}$. Then we have
\[
h(\delta_{t+1})-h(\delta_{t})\ge D
\]
which implies 
\[
h(\delta_{T})\ge \sum_{t=1}^{T}h(\delta_{t})-h(\delta_{t-1})\ge D\cdot T.
\]
Then we get
\begin{align}\label{eq_recursion_exp_alpha<3/2}
    \delta_{T}\le \left(\frac{2\alpha}{3-2\alpha}\right)^{\frac{2\alpha}{3-2\alpha}}\frac{1}{(DT)^{\frac{2\alpha}{3-2\alpha}}}.
\end{align}
Hence, $F(x_{T})-F(x^{*})$ converges to a stationary point $P(n_{1},n_{2})$ at a rate of $\mathcal{O}\left(\frac{1}{T^{\frac{2\alpha}{3-2\alpha}}}\right)$. We can choose $n_{1}\ge\frac{4^{2/\alpha}C^{2/\alpha}_{g}\sigma_{1}^{2/\alpha}}{C^{\frac{2}{\alpha}\cdot\frac{3}{3-2\alpha}}\epsilon^{2/\alpha}}$, and $n_{2}\ge\frac{4^{1/\alpha}C'^{1/\alpha}_{H}\sigma_{2,\alpha}^{2/\alpha}}{C^{\frac{1}{\alpha}\cdot\frac{3}{3-2\alpha}}\epsilon^{1/\alpha}}$ to have $P(n_{1},n_{2})\le C^{3/(3-2\alpha)}\frac{\epsilon}{2}$.
With $T\ge\frac{2\alpha}{3-2\alpha}\frac{1}{D(\epsilon/2)^{\frac{3-2\alpha}{2\alpha}}}$, we have $\delta_{T}\le \epsilon/2$ and then $ {\mathbb{E}[F(\xv_{T})]-F(\xv^*)}\le C^{\frac{3}{3-2\alpha}}\epsilon$.  Finally, the total sample complexity for having $ {\mathbb{E}[F(\xv_{T})]-F(\xv^*)}\le\epsilon$ would be
\[
T\cdot (n_{1}+n_{2})=\left(\frac{2\alpha}{3-2\alpha}D\frac{C^{3/(2\alpha)}}{(\epsilon/2)^{\frac{3-2\alpha}{2\alpha}}}\right)\left(\frac{4^{2/\alpha}C^{2/\alpha}_{g}\sigma_{1}^{2/\alpha}}{\epsilon^{2/\alpha}}+\frac{4^{1/\alpha}C'^{1/\alpha}_{H}\sigma_{2,\alpha}^{2/\alpha}}{\epsilon^{1/\alpha}}\right)=\mathcal{O}\left(\frac{1}{\epsilon^{\frac{7}{2\alpha}-1}}\right).
\]

\textbf{Part (ii)}:
By Lemma \ref{lemma:two_ineq}, we have:
\begin{align}
       F(\xv_{t}+{\bf {\Delta}}_{t})-F(\xv^*)\leq C(F(\xv_t)-F(\xv_{t}+{\bf {\Delta}}_{t}))+C_g\|\nabla F(\xv_t)-\gv_{t}\|^{3/2}+C_{H}\|\nabla^2F(\xv_t)-\Hv_{t}\|^{3}.
\end{align}
Taking expectation of both sides given $\xv_{t}$ and using the same arguments as in \eqref{10047} and \eqref{10048}, we get
\begin{align}
 \mathbb{E}[F(\xv_{t}+{\bf {\Delta}}_{t})-F(\xv^*)]\leq C(\mathbb{E}F(\xv_t)-\mathbb{E}F(\xv_{t}+{\bf {\Delta}}_{t}))+C_g\frac{\sigma_{1}^{\alpha}}{n_{1}^{3/4}}+C'_{H}\frac{\sigma_{2,1.5}^{2}}{n_{2}^{3/2}}.
\end{align}
Let define $\delta_{t}:=\mathbb{E}[F(\xv_{t})]-F(\xv^*)-C_g\frac{\sigma_{1}}{n_{1}^{3/4}}-C'_{H}\frac{\sigma_{2,1.5}^{2}}{n_{2}^{3/2}}$ and rewrite the above inequality as follows,
\begin{align}
   \delta_{t+1}\le C(\delta_{t}-\delta_{t+1}).
\end{align}
We get $\delta_{T}\le\left( \frac{C}{C+1}\right)^{T}\delta_{0}$. Thus, $F(x_{T})-F(x^{*})$ converges to a stationary point
\[
{C_g}\frac{\sigma_{1}}{n_{1}^{3/4}}+{C'_{H}}\frac{\sigma_{2,1.5}^{2}}{n_{2}^{3/2}}
\]
at a rate of $\left(\frac{C}{C+1}\right)^{T}\cdot\delta_{0}$.
In order to have $ {\mathbb{E}[F(\xv_{T})]-F(\xv^*)}\le\epsilon$, we can choose $T\ge \frac{\log(2\delta_{0}/\epsilon)}{\log\left(\frac{C+1}{C}\right)}$, $n_{1}\ge\frac{C^{4/3}_{g}}{C^{4}}\frac{4^{4/3}\sigma_{1}^{4/3}}{\epsilon^{4/3}}$, and $n_{2}\ge\frac{C'^{2/3}_{H}}{C^{2}}\frac{4^{2/3}\sigma_{2,1.5}^{4/3}}{\epsilon^{2/3}}$. Finally, the total sample complexity would be
\[
T\cdot (n_{1}+n_{2})=\frac{\log\left(\frac{2\delta_{0}}{\epsilon}\right)}{\log\left(\frac{C+1}{C}\right)}\cdot\left(\frac{C^{4/3}_{g}}{C^{4}}\frac{4^{4/3}\sigma_{1}^{4/3}}{\epsilon^{4/3}}+\frac{C'^{2/3}_{H}}{C^{2}}\frac{4^{2/3}\sigma_{2,1.5}^{4/3}}{\epsilon^{2/3}}\right).
\]
\subsubsection{Proof of Lemma \ref{lemma_norm_op_summation_of_iid_matrix}}\label{proof_lemma_norm_op_summation_of_iid_matrix}
The proof is mainly adapted from the symmetrization argument \citep{wainwright2019high,chen2011masked}. Consider $\Yv'_{i}$ as an independent copy of $\Yv_{i}$ for $i=1,\ldots,n$. Then
\begin{align}
    &\mathbb{E}\left\|\sum_{i=1}^{n}\Yv_{i}\right\|^{p}=\mathbb{E}\left\|\sum_{i=1}^{n}\mathbb{E}_{\Yv'_{i}}(\Yv_{i}-\Yv'_{i})\right\|^{p}\nonumber\\
    &\le \mathbb{E}_{\Yv}\mathbb{E}_{\Yv'}\left\|\sum_{i=1}^{n}(\Yv_{i}-\Yv'_{i})\right\|^{p}=\mathbb{E}\left\|\sum_{i=1}^{n}\epsilon_{i}(\Yv_{i}-\Yv'_{i})\right\|^{p}\label{10057}\\
    &\le \mathbb{E}\left[2^{p-1}\left(\left\|\sum_{i=1}^{n}\epsilon_{i}\Yv_{i}\right\|^{p}+\left\|\sum_{i=1}^{n}\epsilon_{i}\Yv'_{i}\right\|^{p}\right)\right]=2^{p}\mathbb{E}\left\|\sum_{i=1}^{n}\epsilon_{i}\Yv_{i}\right\|^{p}\label{10058}
\end{align}
where \eqref{10057} comes from the fact that $\Yv_{i}-\Yv'_{i}$ has the same distribution as $\Yv'_{i}-\Yv_{i}$ and $\epsilon_{i}$'s are Rademacher random variables (i.e. $\mathbb{P}[\epsilon_{i}=-1]=\mathbb{P}[\epsilon_{i}=1]=\frac{1}{2}$). \eqref{10058} comes from the inequalities $\|\Av-\Hv\|\le \|\Av\|+\|\Hv\|$ and $\left(\frac{a+b}{2}\right)^{q}\le \frac{a^{q}+b^{q}}{2}$. Hence,
\begin{align*}
    \left[\mathbb{E}\left\|\sum_{i=1}^{n}\Yv_{i}\right\|^{p}\right]^{\frac{1}{p}}\le 2\left[\mathbb{E}\left\|\sum_{i=1}^{n}\epsilon_{i}\Yv_{i}\right\|^{p}\right]^{\frac{1}{p}}
\end{align*}
Let define the Schatten $p$-norm of a matrix $\Av$ as $\|\Av\|_{p}:=\left(\Tr[(\Av^{T}\Av)^{p/2}]\right)^{\frac{1}{p}}:=\left(\sum_{i\ge1}s^{p}_{i}(A)\right)^{\frac{1}{p}}$ where $s_{i}(A)$'s are the singular value of $\Av$. Schatten $\infty$-norm of a matrix is its operator norm by the definition ($\|\Av\|=\|\Av\|_{\infty}$). Then for $q\ge p$, we have
\begin{align}
     &\left[\mathbb{E}\left\|\sum_{i=1}^{n}\Yv_{i}\right\|^{p}\right]^{\frac{1}{p}}\le 2\left[\mathbb{E}\left\|\sum_{i=1}^{n}\epsilon_{i}\Yv_{i}\right\|^{p}\right]^{\frac{1}{p}}\nonumber\\
     &\le 2\left[\mathbb{E}\left\|\sum_{i=1}^{n}\epsilon_{i}\Yv_{i}\right\|_{q}^{p}\right]^{\frac{1}{p}}\le 2\left[\mathbb{E}_{\Yv}\left(\mathbb{E}_{\boldsymbol{\epsilon}}\left\|\sum_{i=1}^{n}\epsilon_{i}\Yv_{i}\right\|_{q}^{q}\right)^{p/q}\right]^{\frac{1}{p}}\label{10060}
\end{align}
where \eqref{10060} comes from the fact that $\left(\mathbb{E}\|X\|^{p}\right)^{1/p}\le \left(\mathbb{E}\|X\|^{q}\right)^{1/q}$ for $q\ge p$. 

The matrix Khintchine inequality is as follows:
\begin{lemma}\citep{mackey2014matrix}\label{matrix Khintchine inequality_lemma}
Suppose $q>2$ and consider the deterministic, symmetric matrices $\Av_{i},\,1\le i\le n$. Then
\begin{align*}
    \left(\mathbb{E}_{\boldsymbol{\epsilon}}\left\|\sum_{i=1}^{n}\epsilon_{i}\Av_{i}\right\|_{q}^{q}\right)^{1/q}\le \sqrt{q}\left\|\left[\sum_{i=1}^{n}\Av_{i}^{2}\right]^{1/2}\right\|_{q}
\end{align*}
\end{lemma}
Assume that $q=\max\{p,2\log d\}\ge p$. From \eqref{10060},
\begin{align}
    &\left[\mathbb{E}\left\|\sum_{i=1}^{n}\Yv_{i}\right\|^{p}\right]^{\frac{1}{p}}\le 2\left[\mathbb{E}_{\Yv}\left(\mathbb{E}_{\boldsymbol{\epsilon}}\left\|\sum_{i=1}^{n}\epsilon_{i}\Yv_{i}\right\|_{q}^{q}\right)^{p/q}\right]^{\frac{1}{p}}\nonumber\\
    &\le 2\sqrt{q}\left[\mathbb{E}_{\Yv}\left\|\left(\sum_{i=1}^{n}\Yv_{i}^{2}\right)^{1/2}\right\|_{q}^{p}\right]^{\frac{1}{p}}\label{10062}\\
    &\le 2\sqrt{q}\left[\mathbb{E}_{\Yv}\left(d^{\frac{1}{q}}\left\|\left(\sum_{i=1}^{n}\Yv_{i}^{2}\right)^{1/2}\right\|\right)^{p}\right]^{\frac{1}{p}}\label{10063}\\
    &\le 2\sqrt{eq}\left[\mathbb{E}_{\Yv}\left(\left\|\left(\sum_{i=1}^{n}\Yv_{i}^{2}\right)^{1/2}\right\|\right)^{p}\right]^{\frac{1}{p}}\label{10064}\\
    &= 2\sqrt{eq}\left[\mathbb{E}_{\Yv}\left(\left\|\sum_{i=1}^{n}\Yv_{i}^{2}\right\|\right)^{p/2}\right]^{\frac{1}{p}}\label{10065}
\end{align}
where \eqref{10062} is derived by Lemma \ref{matrix Khintchine inequality_lemma} and \eqref{10063} comes from $\|\Av\|_{q}\le d^{\frac{1}{q}}\|A\|$. \eqref{10064} is due to $d^{1/q}\le d^{\frac{1}{2\log d}}\le \sqrt{e}$ and \eqref{10065} follows from $\|\Av\|^{2}=\lambda_{\max}(\Av^{T}\Av)=\lambda_{\max}(\Av^{2})=\|\Av^{2}\|$ when $\Av$ is positive semi-definite ($\Av\succeq 0$).
\subsection{Proofs of Section \ref{subsec_SCRN_PL_alpha_(3/2,2)}}

We begin by providing some lemmas that will be used in the main proof.

\begin{lemma}\label{berns_lemma}
Under Assumptions \ref{assump2} and \ref{assump2.5}, we can adjust gradient and Hessian mini-batch sizes
\begin{align*}
    &n_{1}\ge \frac{8}{3}\max\left(\frac{M_{1}}{{\epsilon}},\frac{\sigma_{1}^{2}}{\epsilon^{2}}\right)\log\frac{2d}{\delta},\\
    &n_{2}\ge \frac{8}{3}\max\left(\frac{M_{2}}{\sqrt{\epsilon}},\frac{\sigma_{2}^{2}}{\epsilon}\right)\log\frac{2d}{\delta},
\end{align*}
such that with probability at least $1-\delta$,
\begin{align}
    &\|\gv_{t}-\nabla F(\xv_{t})\|\le \epsilon,\\
     &\|\Hv_{t}-\nabla^{2} F(\xv_{t})\|\le \sqrt{\epsilon}.
\end{align}
\end{lemma}
\begin{proof}[Proof of Lemma \ref{berns_lemma}]
Recall $\gv(\xv,\xi)=\nabla f(\xv,\xi)-\nabla F(\xv)$, $\Gv(\xv,\xi)=\begin{bmatrix}\boldsymbol{0}_{1\times 1} &\gv(\xv,\xi)^{T}\\
\gv(\xv,\xi)&\boldsymbol{0}_{d\times d}
\end{bmatrix}$, and $\Hv(\xv,\xi)=\nabla^{2}f(\xv,\xi)-\nabla^{2}F(\xv)$. 

We use the matrix Bernstein's inequality from \citep[Theorem 6.17]{wainwright2019high}
to control the estimation error in the stochastic gradients and stochastic Hessians under Assumptions \ref{assump2} and \ref{assump2.5}. For a given $\xv_{t}$,
\begin{equation}
\mathbb{P}\left[\left\|\frac{1}{n_{1}}\sum_{i=1}^{n_{1}}\Gv(\xv_{t},\xi_{i})\right\|\ge t \right]\le 2d\exp\left(\frac{-n_{1}t^{2}}{2(\sigma_{g}^{2}+M_{1}t)}\right)\le 2d\exp\left(-\frac{n_{1}}{8}\min\left\{\frac{t}{M_{1}},\frac{t^{2}}{\sigma_{g}^{2}}\right\}\right),
\end{equation}
where $M_{1}$ is the parameter of Bernstein's condition for $\Gv(\xv_{t},\xi_{i})$ in Assumption \ref{assump2.5} and
\[
\sigma_{g}^{2}:=\max\left\{\left\|\frac{1}{n_{1}}\sum_{i=1}^{n_{1}}\mathbb{E}[\gv(\xv_{t},\xi_{i})\gv(\xv_{t},\xi_{i})^{T}]\right\|_{2},\left\|\frac{1}{n_{1}}\sum_{i=1}^{n_{1}}\mathbb{E}[\gv(\xv_{t},\xi_{i})^{T}\gv(\xv_{t},\xi_{i})]\right\|_{2}\right\}.
\]
Note that $\sigma_{g}^{2}\le \sigma_{1}^{2}$ where $\sigma_{1}^{2}$ is the variance parameter in Assumption \ref{assump2}.
\begin{equation}
\mathbb{P}\left[\left\|\frac{1}{n_{2}}\sum_{i=1}^{n_{2}}\Hv(\xv_{t},\xi_{i})\right\|\ge t \right]\le 2d\exp\left(\frac{-n_{2}t^{2}}{2(\sigma_{2}^{2}+M_{2}t)}\right)\le 2d\exp\left(-\frac{n_{2}}{8}\min\left\{\frac{t}{M_{2}},\frac{t^{2}}{\sigma_{2}^{2}}\right\}\right).
\end{equation}
where $M_{2}$ is the parameter of Bernstein's condition for $\Hv(\xv_{t},\xi_{i})$ in Assumption \ref{assump2.5} and
\[
\sigma_{H}^{2}:=\left\|\frac{1}{n_{1}}\sum_{i=1}^{n_{1}}\mathbb{E}[\Hv^{2}(\xv_{t},\xi_{i})]\right\|.
\]
Note that $\sigma_{H}^{2}\le \sigma_{2}^{2}$ where $\sigma_{2}^{2}$ is the parameter for the bounded variance condition \ref{assump2}. 
First, we claim that for every $\vv\in \mathbb{R}^{d}$ we have
\begin{align}\label{lin_algebra_lemma}
    \|\vv\|_{2}=\left\|\begin{bmatrix}\boldsymbol{0}_{1\times 1}&\vv^{T}\\
\vv&\boldsymbol{0}_{d\times d}\end{bmatrix}\right\|
\end{align}

Using Inequality \eqref{lin_algebra_lemma}, we can obtain $\gv_{t}-\nabla F(\xv_{t})\le\frac{1}{n_{1}}\sum_{i=1}^{n}\Gv(\xv_{t},\xi_{i})$. Note that $\Hv_{t}-\nabla^{2}F(\xv_{t})=\frac{1}{n_{2}}\sum_{i=1}^{n}\Hv(\xv_{t},\xi_{i})$. Hence,
\begin{align*}
    &\mathbb{P}\left[\left\|\gv_{t}-\nabla F(\xv_{t})\right\|\ge t \right]\le 2d\exp\left(-\frac{n_{1}}{8}\min\left\{\frac{t}{M_{1}},\frac{t^{2}}{\sigma_{1}^{2}}\right\}\right)\\
    &\mathbb{P}\left[\left\|\Hv_{t}-\nabla^{2} F(\xv_{t})\right\|\ge t \right]\le 2d\exp\left(-\frac{n_{2}}{8}\min\left\{\frac{t}{M_{2}},\frac{t^{2}}{\sigma_{2}^{2}}\right\}\right)
\end{align*}

In other words, for
\begin{align}
     &n_{1}\ge 8\max\left(\frac{M_{1}}{{\epsilon}},\frac{\sigma_{1}^{2}}{\epsilon^{2}}\right)\log\frac{2d}{\delta},\\
    &n_{2}\ge 8\max\left(\frac{M_{2}}{\sqrt{\epsilon}},\frac{\sigma_{2}^{2}}{\epsilon}\right)\log\frac{2d}{\delta},
\end{align}
we have:
\begin{align}
    &\|\gv_{t}-\nabla F(\xv_{t})\|\le_{1-\delta} \epsilon,\\
    &\|\Hv_{t}-\nabla^{2} F(\xv_{t})\|\le_{1-\delta}\sqrt{\epsilon}.
\end{align}
The claims of Lemma \ref{berns_lemma} are established.

Proof of \eqref{lin_algebra_lemma}: For every symmetric matrix $\Av$ we have $|\lambda_{\max}(\Av)|=\sigma_{\max}(\Av)$. For matrix $\Av=\begin{bmatrix}\boldsymbol{0}_{1\times 1}&\vv^{T}\\
\vv&\boldsymbol{0}_{d\times d}\end{bmatrix}$ we have $\det(\lambda I-\Av)=\lambda^{d}-\lambda^{d-2}\|\vv\|^{2}=0$. Thus, $\lambda_{\max}=\|\vv\|$.


\end{proof}

\begin{lemma}\label{lemma1}
Under Assumptions \ref{assump1}, \ref{assump2}, and \ref{assump2.5} and for
\begin{align}
    &n_{1}\ge 8\max\left(\frac{M_{1}}{\epsilon^{1/\alpha}},\frac{\sigma_{1}^{2}}{\epsilon^{2/\alpha}}\right)\log\frac{2d}{\delta},\label{eq:n1}\\
    &n_{2}\ge 8\max\left(\frac{M_{2}}{{\epsilon^{1/(2\alpha)}}},\frac{\sigma_{2}^{2}}{\epsilon^{1/\alpha}}\right)\log\frac{2d}{\delta}
    \label{eq:n2}
\end{align}
where $\alpha\ge0$, we have for Algorithm \ref{algorithm}:
\begin{align}
    &\|\nabla F(\xv_{t+1})\|\le_{1-2\delta} \frac{3}{2}{\epsilon^{1/\alpha}}+\frac{M+L_{2}+1}{2}\|{\bf \Delta}_{t}\|^{2}.
\end{align}
\end{lemma}
\begin{proof}[Proof of Lemma \ref{lemma1}]
We know from Lipschitzness of Hessian of $F$ that
\begin{align}
\|\nabla F(\xv_{t+1})-\nabla F(\xv_{t})-\nabla^{2}F(\xv_{t})(\xv_{t+1}-\xv_{t})\|\le \frac{L_{2}}{2}\|\xv_{t+1}-\xv_{t}\|^{2}.
\end{align}
Moreover, ${\bf \Delta}_{t}$ is an optimal solution of the problem in \eqref{sub-problem} and therefore, it satisfies the optimality condition: $\gv_{t}+\Hv_{t}{\bf \Delta}_{t}+\frac{M}{2}\|{\bf \Delta}_{t}\|{\bf \Delta}_{t}=0$. Therefore, we have:
\begin{align}
    \|\gv_{t}+\Hv_{t}{\bf \Delta}_{t}\|=\frac{M}{2}\|{\bf \Delta}_{t}\|^{2}.
\end{align}
By summing the above two equations and using triangle inequality, we have:
\begin{align}
    \|\nabla F(\xv_{t+1})\|&\le \frac{M+L_{2}}{2}\|{\bf \Delta}_{t}\|^{2}+\|\nabla F(\xv_{t})-\gv_{t}\|+\|(\nabla^{2}F(\xv_{t})-\Hv_{t})(\xv_{t+1}-\xv_{t})\|\nonumber\\
    &\le \frac{M+L_{2}}{2}\|{\bf \Delta}_{t}\|^{2}+\|\nabla F(\xv_{t})-\gv_{t}\|+\|\nabla^{2}F(\xv_{t})-\Hv_{t}\|\|{\bf \Delta}_{t}\|\\
    &\overset{(a)}{\le}_{1-2\delta}\frac{M+L_{2}}{2}\|{\bf \Delta}_{t}\|^{2}+\epsilon^{1/\alpha}+{\epsilon^{1/(2\alpha)}}\|{\bf \Delta}_{t}\|\\&\overset{(b)}{\le}_{1-2\delta} \frac{M+L_{2}+1}{2}\|{\bf \Delta}_{t}\|^{2}+\frac{3}{2}{\epsilon^{1/\alpha}}.
\end{align}
(a) Due to Lemma \ref{berns_lemma} if we consider the batch sizes considered in \eqref{eq:n1} and \eqref{eq:n2}.\\
(b) According to Young's inequality, we have: $ab\le \frac{a^{2}+b^{2}}{2} $.
\end{proof}

\begin{lemma}\label{lemma2}
Under the assumptions of Lemma \ref{lemma1}, we have
\begin{align}
        F(\xv_{t+1})\le_{1-2\delta} F(\xv_{t})+{{\epsilon^{1/\alpha}}} \|{\bf\Delta}_{t}\|_{2} +\frac{1}{2} {\epsilon^{1/(2\alpha)}}\|{\bf\Delta}_{t}\|_{2}^{2}+\frac{2L_{2}-3M}{12}\|{\bf\Delta}_{t}\|^{3}.
\end{align}
\end{lemma}
\begin{proof}[Proof of Lemma \ref{lemma2}] 
For the sub-problem mentioned in \eqref{sub-problem}, we have from \citep[Proposition 1]{nesterov2006cubic} that
\begin{align}\label{1023}
    \gv_{t}^{T}{\bf \Delta}_{t}+{\bf \Delta}_{t}^{T}\Hv_{t}{\bf \Delta}_{t}+\frac{M}{2}\|{\bf \Delta}_{t}\|^{3}=0,\quad \Hv_{t}+\frac{M\|{\bf\Delta}_{t}\|}{2}I_{d\times d}\succeq 0
\end{align}
which yields 
\begin{equation}\label{1024}
    \gv_{t}^{T}{\bf \Delta}_{t}\le0.
\end{equation}
From Lipschitzness of Hessian, we have the following:
\begin{align}
    &F(\xv_{t+1})\le F(\xv_{t})+\langle \nabla F(\xv_{t}), {\bf\Delta}_{t}\rangle +\frac{1}{2}\langle {\bf\Delta}_{t}, \nabla^{2}F(\xv_{t}){\bf\Delta}_{t}\rangle+\frac{L_{2}}{6}\|{\bf \Delta}_{t}\|^{3}\\
    &=F(\xv_{t})+\langle \nabla F(\xv_{t})-\gv_{t}, {\bf\Delta}_{t}\rangle +\frac{1}{2}\langle {\bf\Delta}_{t}, (\nabla^{2}F(\xv_{t})-\Hv_{t}){\bf\Delta}_{t}\rangle+\gv_{t}^{T}{\bf \Delta}_{t}+\frac{1}{2}{\bf \Delta}_{t}^{T}\Hv_{t}{\bf \Delta}_{t}+\frac{L_{2}}{6}\|{\bf \Delta}_{t}\|^{3}\nonumber\\
    &\overset{(a)}{\le} F(\xv_{t})+\langle \nabla F(\xv_{t})-\gv_{t}, {\bf\Delta}_{t}\rangle +\frac{1}{2}\langle {\bf\Delta}_{t}, (\nabla^{2}F(\xv_{t})-\Hv_{t}){\bf\Delta}_{t}\rangle+\frac{2L_{2}-3M}{12}\|{\bf\Delta}_{t}\|^{3}\\
    &\overset{(b)}{\le}F(\xv_{t})+\|\nabla F(\xv_{t})-\gv_{t}\|_{2} \|{\bf\Delta}_{t}\|_{2} +\frac{1}{2} \|\nabla^{2}F(\xv_{t})-\Hv_{t}\|\|{\bf\Delta}_{t}\|_{2}^{2}+\frac{2L_{2}-3M}{12}\|{\bf\Delta}_{t}\|^{3}\\
    &\overset{(c)}{\le}_{1-2\delta} F(\xv_{t})+{\epsilon^{1/\alpha}} \|{\bf\Delta}_{t}\|_{2} +\frac{1}{2} {\epsilon^{1/(2\alpha)}}\|{\bf\Delta}_{t}\|_{2}^{2}+\frac{2L_{2}-3M}{12}\|{\bf\Delta}_{t}\|^{3}
\end{align}
where (a) comes from \eqref{1024} and \eqref{1023}. Inequality (b) comes from Cauchy-Schwartz inequality and (c) is derived by Lemma \ref{berns_lemma}. 
\end{proof}
\textbf{Theorem 2.}\textit{ When $F(\xv)$ satisfies Assumptions \ref{assump1}, \ref{assump:4}, the stochastic gradient and Hessian satisfy Assumption \ref{assump2} (with $\alpha=1$) and Assumption \ref{assump2.5}, and there exists an exact solver for sub-problem \eqref{sub-problem}, Algorithm \ref{algorithm}, with probability $1-\delta$, outputs a solution $\xv_{T}$ such that $F(\xv_{T})-F(\xv^{*})\le \epsilon$ after $T=\mathcal{O}\left(\log\left(\log(1/\epsilon)\right)\right)$ iterations with the following numbers of samples for the stochastic gradient and Hessian, respectively per iteration:
\begin{align}
    &n_{1}\ge \frac{8}{3}\max\left(\frac{{\tilde{C}^{1/\alpha}}M_{1}}{\epsilon^{1/\alpha}},\frac{\tilde{C}^{2/\alpha}\sigma_{1}^{2}}{\epsilon^{2/\alpha}}\right)\log\left(\frac{4(T+1)d}{\delta}\right),\\
    &n_{2}\ge \frac{8}{3}\max\left(\frac{{\tilde{C}^{1/(2\alpha)}}M_{2}}{{\epsilon^{1/(2\alpha)}}},\frac{{\tilde{C}^{1/\alpha}}\sigma_{2,1}^{2}}{{\epsilon^{1/\alpha}}}\right)\log\left(\frac{4(T+1)d}{\delta}\right),
\end{align}
where 
$\tilde{C}=1+\frac{\tau_{F}}{2}\left(\frac{M+L_{2}+4}{2}\right)^{\alpha}$.}
\subsubsection{Proof of Theorem \ref{th1}}\label{proof_theorem_SCRN_pL_alpha=2}
Now, we are ready to prove the main theorem. 
Suppose that $t=N_{0}$ is the first time that $\|{\bf \Delta}_{t}\|\le \sqrt[2\alpha]{\epsilon}$. In other words, for $t=0,\ldots,N_{0}-1$, we have $\|{\bf \Delta}_{t}\|> \sqrt[2\alpha]{\epsilon}$. The following analysis is valid for $t=0,\ldots,N_{0}-1$. From Lemma \ref{lemma1}, for 
\begin{align}
     &n_{1}\ge 8\max\left(\frac{M_{1}}{{\epsilon^{1/\alpha}}},\frac{\sigma_{1}^{2}}{\epsilon^{2/\alpha}}\right)\log\frac{2d}{\delta}\label{sample_comp_n1_proof_th1}\\
    &n_{2}\ge 8\max\left(\frac{M_{2}}{{\epsilon^{1/(2\alpha)}}},\frac{\sigma_{2}^{2}}{{\epsilon^{1/\alpha}}}\right)\log\frac{2d}{\delta}\label{sample_comp_n2_proof_th1},
\end{align}
we have:
\begin{align}
          \|\nabla F(\xv_{t+1})\|&\le_{1-2\delta} \frac{3}{2}\sqrt[\alpha]{\epsilon}+\frac{M+L_{2}+1}{2}\|{\bf \Delta}_{t}\|^{2}\label{eq_139}\\
       &\le_{1-2\delta} \frac{3}{2}\|{\bf\Delta}_{t}\|^{2}+\frac{M+L_{2}+1}{2}\|{\bf \Delta}_{t}\|^{2}.
\end{align}
Let $A:=\frac{3}{2}+\frac{M+L_{2}+1}{2}$. Then, we can rewrite the above inequality in the following form:
\begin{align}
    \frac{1}{A}\|\nabla F(\xv_{t+1})\|\le_{1-2\delta}\|{\bf \Delta}_{t}\|^{2}.
    \label{eq:A}
\end{align}
Using Lemma \ref{lemma2}, we have:
\begin{align}
   F(\xv_{t+1})&\le_{1-2\delta} F(\xv_{t})+{\epsilon^{1/\alpha}} \|{\bf\Delta}_{t}\|_{2} +\frac{1}{2} {\epsilon^{1/(2\alpha)}}\|{\bf\Delta}_{t}\|_{2}^{2}+\frac{2L_{2}-3M}{12}\|{\bf\Delta}_{t}\|^{3}\\
    &\le_{1-2\delta} F(\xv_{t})-\left(\frac{3M-2L_{2}}{12}-\frac{3}{2}\right)\|{\bf\Delta}_{t}\|^{3}\\
    &\le_{1-4\delta}F(\xv_{t})-\left(\frac{3M-2L_{2}}{12}-\frac{3}{2}\right)\frac{1}{A^{3/2}}\|\nabla F(\xv_{t+1})\|^{3/2},
\end{align}
where the last inequality is according to \eqref{eq:A}.
Let $B:=\left(\frac{3M-2L_{2}}{12}-\frac{3}{2}\right)/A^{3/2}$.  Using gradient dominance property, we have:
\begin{align}
    (F(\xv_{t+1})-F(\xv^{*}))\le_{1-4\delta} \tau_{F}\frac{(F(\xv_{t})-F(\xv_{t+1}))^{2\alpha/3}}{B^{2\alpha/3}}.
\end{align}
Let $h_{t}:=\frac{\tau_{F}^{\frac{3}{2\alpha-3}}}{B^{\frac{2\alpha}{2\alpha-3}}}[F(\xv_{t})-F(\xv^{*})]$. Then,
\begin{align}
   h_{t+1}^{3/2\alpha}\le_{1-4\delta}h_{t}-h_{t+1}
\end{align}
We define $K_{N_{0}}:=\max_{0\le t\le N_{0}-1}h_{t}$. Therefore, 
\[
1+\frac{1}{K_{N_{0}}^{\frac{2\alpha-3}{2\alpha}}}\le 1+\frac{1}{h^{\frac{2\alpha-3}{2\alpha}}_{t+1}}\le_{1-4\delta}\frac{h_{t}}{h_{t+1}}.
\]
Finally, we can find $1\le N_{1}\le N_{0}$ such that 
\begin{align}
   h_{N_{1}} \le_{1-4N_{1}\delta}{h_{0}}\cdot\left(1+\frac{1}{K_{N_{0}}^{\frac{2\alpha-3}{2\alpha}}}\right)^{-N_{1}}\le \frac{1}{2}.
\end{align}
From above inequality, we imply
\begin{align}\label{eq10475}
    N_{1}\ge \frac{\log 2h_{0}}{\log\left(1+\frac{1}{K_{N_{0}}^{\frac{2\alpha-3}{2\alpha}}}\right)}
\end{align}
For iteration $t\ge N_{1}$, we use the following recursion: $h_{t+1}\le_{1-4\delta}h^{2\alpha/3}_{t}$. Thus
\begin{align}\label{eq1047}
    h_{N_{0}}\le_{1-4(N_{0}-N_{1})\delta}h^{(\frac{2\alpha}{3})^{N_{0}-N_{1}}}_{N_{1}}\le_{1-4N_{0}\delta}\left(\frac{1}{2}\right)^{(\frac{2\alpha}{3})^{N_{0}-N_{1}}}
\end{align}
On the other hand, from \eqref{eq_139}, gradient dominance property, and stopping criterion $\|{\bf \Delta}\|_{2}\le \sqrt[2\alpha]{\epsilon}$, we have
\begin{equation}\label{eq1048}
    F(\xv_{N_{0}})-F(\xv^{*})\le \tau_{F}\|\nabla F(\xv_{N_{0}})\|^{\alpha}\le_{1-4\delta}\tau_{F}\left(\frac{M+L_{2}+4}{2}\right)^{\alpha}\cdot \epsilon.
\end{equation}
Summing up \eqref{eq1047} and \eqref{eq1048}, we get
\begin{align}
    F(\xv_{N_{0}})-F(\xv^{*}) \le_{1-(4N_{0}+4)\delta}\frac{B^{\frac{2\alpha}{2\alpha-3}}}{\tau_{F}^{\frac{3}{2\alpha-3}}}\left(\frac{1}{2}\right)^{(\frac{2\alpha}{3})^{N_{0}-N_{1}}}+\frac{\tau_{F}}{2}\left(\frac{M+L_{2}+4}{2}\right)^{\alpha}\cdot \epsilon.
\end{align}
First denote $\tilde{C}:=1+\frac{\tau_{F}}{2}\left(\frac{M+L_{2}+4}{2}\right)^{\alpha}$ and $D:=\frac{B^{\frac{2\alpha}{2\alpha-3}}}{\tau_{F}^{\frac{3}{2\alpha-3}}}$. Hence, we have to choose $N_{0}\ge N_{1}+\frac{\log\left(\frac{\log D+\log\frac{1}{\epsilon}}{\log 2}\right)}{\log(2\alpha/3)}$, in order to guarantee that $F(\xv_{N_{0}})-F(\xv^{*})\leq \tilde{C}\epsilon$ with probability at least $1-(4N_{0}+4)\delta$. By plugging $\delta'=4(N_0+1)\delta$ and $\epsilon'=\tilde{C}\epsilon$ in the sample complexities $n_{1}, n_{2}$ in \eqref{sample_comp_n1_proof_th1} and \eqref{sample_comp_n2_proof_th1}, we can get the desired result in Theorem \ref{th1} and the proof is complete.

\subsubsection{Inexact cubic sub-solver (Proof of Lemma \ref{lemm22})}\label{append_inexact_subsolver}
In this part, we show that under some mild assumptions, we can use a gradient descant based algorithm to find an approximate solution $\tilde{\mathbf{\Delta}}_{t}$ of sub-problem in \eqref{sub-problem} such that Theorem \ref{th1_expectation} and Theorem \ref{th1} still hold. To do so, we provide a new version of \eqref{eq:3} for $\tilde{\mathbf{\Delta}}_{t}$ and then show that the order of error terms in the recursion is not changed. Recall that the sub-problem \eqref{sub-problem} has the following form:
\begin{align}
  \min_{{\bf\Delta}\in\mathbb{R}^{d}}m_{t}({\bf\Delta}):=\langle \gv_{t}, {\bf\Delta}\rangle+\frac{1}{2}\langle {\bf\Delta}, \Hv_{t}{\bf\Delta}\rangle+\frac{M}{6}\|{\bf\Delta}\|^{3}.
\end{align}
We need the following assumption:
\begin{assumption}\label{assum10}
Hessian estimators satisfy $\|\Hv_{t}\|\le\beta$ for all $1\le t\le T$ where $\beta$ is a positive constant.
\end{assumption}

Consider any iteration $t\in[1,T]$ of Algorithm \ref{algorithm} and in what follows, the iteration $t$ is fixed. Thus, we drop the subscription $t$ for simplifying the notations.  We use the gradient descent algorithm as our sub-solver with the following update and initial point $\tilde{{\bf {\Delta}}}^{(0)}$:
\[
\tilde{{\bf {\Delta}}}^{(k+1)}=\tilde{{\bf {\Delta}}}^{(k)}-\eta \nabla m(\tilde{{\bf {\Delta}}}^{k})=\left(I-\eta\Hv-\eta \frac{M}{2}\|\tilde{{\bf {\Delta}}}^{(k)}\|\right)\tilde{{\bf {\Delta}}}^{(k)}-\eta\gv.
\]
We define the following quantity 
\[
R_c\triangleq-\frac{\gv_{t}^{T}\Hv\gv}{2M\|\gv_{t}\|^{2}}+\sqrt{\frac{\|\gv\|}{M}+\left(\frac{\gv^{T}\Hv\gv}{2M\|\gv_{t}\|^{2}}\right)^{2}},
\]
which is the Cauchy radius \citep{carmon2016gradient} and 
\[
R\triangleq\frac{\|\Hv\|}{2M}+\sqrt{\left(\frac{\|\Hv\|}{2M}\right)^{2}+\frac{\|\gv\|}{M}},
\]
which is an upper bound on $\|{\bf {\Delta}}\|$ (See \citep[Claim 2.1]{carmon2016gradient}).
We make the following assumptions.
\begin{assumption}\label{assum8}
The step-size $\eta$ satisfies $0\le\eta\le \frac{1}{4(\|\Hv\|+MR)}$.
\end{assumption}
\begin{assumption}\label{assum9}
The initialization $\tilde{{\bf {\Delta}}}^{(0)}$ satisfies $\tilde{{\bf {\Delta}}}^{(0)}=-r\frac{\gv}{\|\gv\|}$, with $0\le r\le R_{c}$.
\end{assumption}
In \citep[Lemma 2.3]{carmon2016gradient}, Carmon and Duchi showed that with the Assumptions \ref{assum8} and \ref{assum9}, $\gv^{T}\tilde{{\bf {\Delta}}}^{(k)}\le 0$ and in \citep[Lemma 2.2]{carmon2016gradient}, they showed that $\nabla m(\tilde{{\bf {\Delta}}}^{(k)})^{T}\tilde{{\bf {\Delta}}}^{(k)}\le 0$ for all $1\le k\le K$. Assumptions \ref{assum10}, \ref{assum9}, and \ref{assum8} implies
\[
\|\nabla^{2} m(\tilde{{\bf {\Delta}}}^{(k)})\|=\left\|\Hv+\frac{M\|\tilde{{\bf {\Delta}}}^{(k)}\|}{2}I_{d\times d}\right\|\le \|\Hv\|+\frac{M\|\tilde{{\bf {\Delta}}}^{(k)}\|}{2}\overset{(a)}{\le}\beta+\frac{MR}{2}
\]
where (a) comes from Assumption \ref{assum10} and $\|\tilde{{\bf {\Delta}}}^{(k)}\|\le R$ which is from \citep[Lemma 2.2]{carmon2016gradient}. Hence, the Lipschitzness of gradient of $m(\cdot)$ is established and the gradient descent algorithm obtains $\|\nabla m(\tilde{{\bf {\Delta}}}^{(K)})\|\le \epsilon'$ with $K= \mathcal{O}(\epsilon'^{-2})$ \cite{nesterov2003introductory}. Let $\tilde{{\bf {\Delta}}}_{t}=\tilde{{\bf {\Delta}}}^{(K)}$ be the inexact solution of above sub-solver.

From Lipschitzness of Hessian, we have the following:
\begin{align}
    &F(\xv_{t}+\tilde{{\bf {\Delta}}}_{t})\le F(\xv_{t})+\langle \nabla F(\xv_{t}), \tilde{{\bf {\Delta}}}_{t}\rangle +\frac{1}{2}\langle \tilde{{\bf {\Delta}}}_{t}, \nabla^{2}F(\xv_{t})\tilde{{\bf {\Delta}}}_{t}\rangle+\frac{L_{2}}{6}\|{\bf {\Delta}}_{t}\|^{3}\nonumber\\
    &=F(\xv_{t})+\langle \nabla F(\xv_{t})-\gv_{t}, \tilde{{\bf {\Delta}}}_{t}\rangle +\frac{1}{2}\langle \tilde{{\bf {\Delta}}}_{t}, (\nabla^{2}F(\xv_{t})-\Hv_{t})\tilde{{\bf {\Delta}}}_{t}\rangle+\frac{L_{2}}{6}\|\tilde{{\bf {\Delta}}}_{t}\|^{3}+\gv_{t}^{T}\tilde{{\bf {\Delta}}}_{t}+\frac{1}{2}\tilde{{\bf {\Delta}}}_{t}^{T}\Hv_{t}\tilde{{\bf {\Delta}}}_{t}\nonumber\\
    &= F(\xv_{t})+\langle \nabla F(\xv_{t})-\gv_{t}, \tilde{{\bf {\Delta}}}_{t}\rangle +\frac{1}{2}\langle \tilde{{\bf {\Delta}}}_{t}, (\nabla^{2}F(\xv_{t})-\Hv_{t})\tilde{{\bf {\Delta}}}_{t}\rangle+\frac{2L_{2}-3M}{12}\|\tilde{{\bf {\Delta}}}_{t}\|^{3}+\frac{1}{2}\gv_{t}^{T}\tilde{{\bf {\Delta}}}_{t}\nonumber\\
    &+\frac{1}{2}\nabla m(\tilde{{\bf {\Delta}}}_{t})^{T}\tilde{{\bf {\Delta}}}_{t}\nonumber\\
    &\overset{(a)}{\le} F(\xv_{t})+\langle \nabla F(\xv_{t})-\gv_{t}, \tilde{{\bf {\Delta}}}_{t}\rangle +\frac{1}{2}\langle \tilde{{\bf {\Delta}}}_{t}, (\nabla^{2}F(\xv_{t})-\Hv_{t})\tilde{{\bf {\Delta}}}_{t}\rangle+\frac{2L_{2}-3M}{12}\|\tilde{{\bf {\Delta}}}_{t}\|^{3}\nonumber\\
    &\overset{(b)}{\le}F(\xv_{t})+\|\nabla F(\xv_{t})-\gv_{t}\|_{2} \|\tilde{{\bf {\Delta}}}_{t}\|_{2} +\frac{1}{2} \|\nabla^{2}F(\xv_{t})-\Hv_{t}\|\|\tilde{{\bf {\Delta}}}_{t}\|_{2}^{2}+\frac{2L_{2}-3M}{12}\|\tilde{{\bf {\Delta}}}_{t}\|^{3}\nonumber\\
    &\overset{(c)}\leq F(\xv_t)+\frac{2\|\nabla F(\xv_t)-\gv_{t}\|^{3/2}}{3}+ \frac{\|\nabla^2 F(\xv_t)-\Hv_{t}\|^{3}}{6}+\frac{2L_{2}-3M+8}{12}\|\tilde{{\bf {\Delta}}}_{t}\|^{3}\label{100042},
\end{align}
where (a) comes from the fact that $\gv^{T}\tilde{{\bf {\Delta}}}_{t}\le 0$ and $\nabla m(\tilde{{\bf {\Delta}}}_{t})^{T}\tilde{{\bf {\Delta}}}_{t}\le 0$ and (b) comes from Cauchy-Schwartz inequality and (c) is due to Young's inequality. 

 We know from Lipschitzness of Hessian of $F$ that
\begin{align}
\|\nabla F(\xv_{t}+\tilde{{\bf {\Delta}}}_{t})-\nabla F(\xv_{t})-\nabla^{2}F(\xv_{t})\tilde{{\bf {\Delta}}}_{t}\|\le \frac{L_{2}}{2}\|\tilde{{\bf {\Delta}}}_{t}\|^{2}.
\end{align}
Moreover, the gradient descent returns  $\tilde{{\bf {\Delta}}}_{t}$ which is $\epsilon'$-approximate first-order stationary point:
$\|\nabla m(\tilde{{\bf {\Delta}}}_{t})\|=\|\gv_{t}+\Hv_{t}\tilde{{\bf {\Delta}}}_{t}+\frac{M}{2}\|\tilde{{\bf {\Delta}}}_{t}\|\tilde{{\bf {\Delta}}}_{t}\|\le \epsilon'$. Therefore, we have:
\begin{align}
    \|\gv_{t}+\Hv_{t}{\tilde{\bf {\Delta}}}_{t}\|\le\epsilon'+ \frac{M}{2}\|{\tilde{\bf {\Delta}}}_{t}\|^{2}.
\end{align}
By summing the above two equations and using triangle inequality, we have:
\begin{align}
    \|\nabla F(\xv_{t}+\tilde{{\bf {\Delta}}}_{t})\|&\le \frac{M+L_{2}}{2}\|\tilde{{\bf {\Delta}}}_{t}\|^{2}+\|\nabla F(\xv_{t})-\gv_{t}\|+\|(\nabla^{2}F(\xv_{t})-\Hv_{t})\tilde{{\bf {\Delta}}}_{t}\|+\epsilon'\nonumber\\
    &\le \frac{M+L_{2}}{2}\|\tilde{{\bf {\Delta}}}_{t}\|^{2}+\|\nabla F(\xv_{t})-\gv_{t}\|+\|\nabla^{2}F(\xv_{t})-\Hv_{t}\|\|\tilde{{\bf {\Delta}}}_{t}\|+\epsilon'\\
    &\leq \frac{M+L_{2}}{2}\|\tilde{{\bf {\Delta}}}_{t}\|^{2}+\|\nabla F(\xv_{t})-\gv_{t}\|+\frac{\|\nabla^{2}F(\xv_{t})-\Hv_{t}\|^2}{2}+\frac{\|\tilde{{\bf {\Delta}}}_{t}\|^2}{2}+\epsilon'\label{100043},
\end{align}

Based on gradient dominance property and  \eqref{100043}, we have:
\begin{equation}\label{100044}
    \begin{split}
        &(F(\xv_{t}+\tilde{{\bf {\Delta}}}_{t})-F(\xv^*))\leq \tau_F\|\nabla F(\xv_t+\tilde{{\bf {\Delta}}}_{t})\|^\alpha\\
        &\leq \tau_F\Big(\frac{M+L_{2}+1}{2}\|\tilde{{\bf {\Delta}}}_{t}\|^{2}+\|\nabla F(\xv_{t})-\gv_{t}\|+\frac{\|\nabla^{2}F(\xv_{t})-\Hv_{t}\|^2}{2}+\epsilon'\Big)^\alpha\\
        &\leq  4^{\alpha-1}\tau_F\Big[\Big(\frac{M+L_2+1}{2}\Big)^\alpha\|\tilde{{\bf {\Delta}}}_{t}\|^{2\alpha}+\|\nabla F(\xv_{t})-\gv_{t}\|^{\alpha}+\frac{\|\nabla^{2}F(\xv_{t})-\Hv_{t}\|^{2\alpha}}{2^{\alpha}}+\epsilon'^{\alpha}\Big]\\
    \end{split}
\end{equation}
where we used $(a+b+c+d)^\alpha\leq 4^{\alpha-1}(a^\alpha+b^\alpha+c^\alpha+d^{\alpha})$, $\forall a,b,c,d>0$ in the last inequality.

Substituting \eqref{100042} into \eqref{100044}, we have:
\begin{small}
\begin{equation}
    \begin{split}
        &F(\xv_{t}+\tilde{{\bf {\Delta}}}_{t})-F(\xv^*)\leq\\ &4^{\alpha-1}\tau_F\Big[ \Big(\frac{M+L_2+1}{2}\Big)^\alpha\Big(\frac{12}{3M-2L_2-6}\Big)^{2\alpha/3}\Big(F(\xv_{t})-F(\xv_{t}+\tilde{{\bf {\Delta}}}_{t})+\frac{2\|\nabla F(\xv_t)-\gv_{t}\|^{3/2}}{3}+ \\
       &\qquad\qquad\frac{\|\nabla^2 F(\xv_t)-\Hv_{t}\|^{3}}{6}\Big)^{2\alpha/3}+ \|\nabla F(\xv_t)-\gv_{t}\|^\alpha+\frac{\|\nabla^2F(\xv_t)-\Hv_{t}\|^{2\alpha}}{2^{\alpha}}+\epsilon'^{\alpha}\Big]\\
     &\qquad\qquad\leq C(F(\xv_t)-F(\xv_{t}+\tilde{{\bf {\Delta}}}_{t}))^{2\alpha/3}+C_g\|\nabla F(\xv_t)-\gv_{t}\|^\alpha+C_{H}\|\nabla^2F(\xv_t)-\Hv_{t}\|^{2\alpha}+C_{\epsilon'}\epsilon'^{\alpha},
    \end{split}
\end{equation}
\end{small}
where in the last inequality we used $(a+b+c)^{2\alpha/3}\le3^{2\alpha/3-1/3}(a^{2\alpha/3}+b^{2\alpha/3}+c^{2\alpha/3}) $ which is derived by the following inequalities: \[(a+b+c)^{2\alpha}\leq 3^{2\alpha-1}(a^{2\alpha}+b^{2\alpha}+c^{2\alpha})\le 3^{2\alpha-1}(a^{2\alpha/3}+b^{2\alpha/3}+c^{2\alpha/3})^{3}\]
for any $a,b,c\in \mathbb{R^{+}}$ and,
\begin{equation}\label{eq200}
    \begin{split}
      C&=4^{\alpha-1}3^{(2\alpha-1)/3}\tau_F \Big(\frac{M+L_2+1}{2}\Big)^\alpha\Big(\frac{12}{3M-2L_2-8}\Big)^{2\alpha/3}
       \\
     C_g&=2^{8\alpha/3-2}\times3^{-\frac{1}{3}}\tau_F \Big(\frac{M+L_2+1}{2}\Big)^\alpha\Big(\frac{12}{3M-2L_2-8}\Big)^{2\alpha/3}+2^{\alpha-2}\tau_{F}\\
     C_{H}&=2^{4\alpha/3-2}\times3^{-\frac{1}{3}}\tau_F \Big(\frac{M+L_2+1}{2}\Big)^\alpha\Big(\frac{12}{3M-2L_2-8}\Big)^{2\alpha/3}+2^{\alpha-2}\tau_{F}\\
     C_{\epsilon'}&=4^{\alpha-1}\tau_{F}
    \end{split}
\end{equation}
If $\epsilon'=\epsilon^{1/\alpha}$, we have the following recursion inequality.
\begin{align}
    \delta_{t+1}\le C(\delta_{t}-\delta_{t+1})^{2\alpha/3}+C_g\|\nabla F(\xv_t)-\gv_{t}\|^\alpha+C_{H}\|\nabla^2F(\xv_t)-\Hv_{t}\|^{2\alpha}+C_{\epsilon'}\epsilon.
\end{align}
Recall that in the analysis of Sections \ref{subsec_SCRN_PL_alpha_(1,3/2)} and \ref{subsec_SCRN_PL_alpha_(3/2,2)}, we keep the error terms $\|\nabla F(\xv_t)-\gv_{t}\|^\alpha$ and $\|\nabla^2F(\xv_t)-\Hv_{t}\|^{2\alpha}$ in the order of $\epsilon$ (in expectation or with high probability) and note that the effect of inexactness in sub-solver appears in the term $C_{\epsilon'}\epsilon$ with the same order. Hence, we still have a valid analysis of the global optimum when using an inexact sub-solver based on gradient descent. It is important to note that when using the inexact sub-solver, finding a local minimum solution of sub-problem is sufficient for the global convergence analysis of SCRN and despite second order convergence analysis of SCRN \cite{carmon2016gradient,tripuraneni2018stochastic, zhou2019stochastic}, its global convergence analysis does not need to approximate the exact solution of sub-solver, i.e., ${\bf {\Delta}}_{t}$.

\subsubsection{Simulations of SCRN and SGD on synthetic functions satisfying gradient dominance property}\label{app:sim_synthetic}

\begin{example}
We consider function $x^{p/q}:[-1,1]\to[0,1]$, satisfying gradient dominance property with $\alpha=\frac{p}{p-q}\in (1,\infty)$ where $p$ is an even positive integer and $q<p$ is a positive integer. For this choice of objective function, we compare average global optimization error ($\mathbb{E}[F(x)]-F(x^*)$) of SCRN and SGD for $1<\alpha<3/2$ in Figure \ref{F_alpha<2}. We observe that the improvement of SCRN upon SGD is increasing when $\alpha$ decreases from $4/3$ to $7/6$. 
\end{example}

\begin{figure}
    \centering
    \includegraphics[width=1\textwidth]{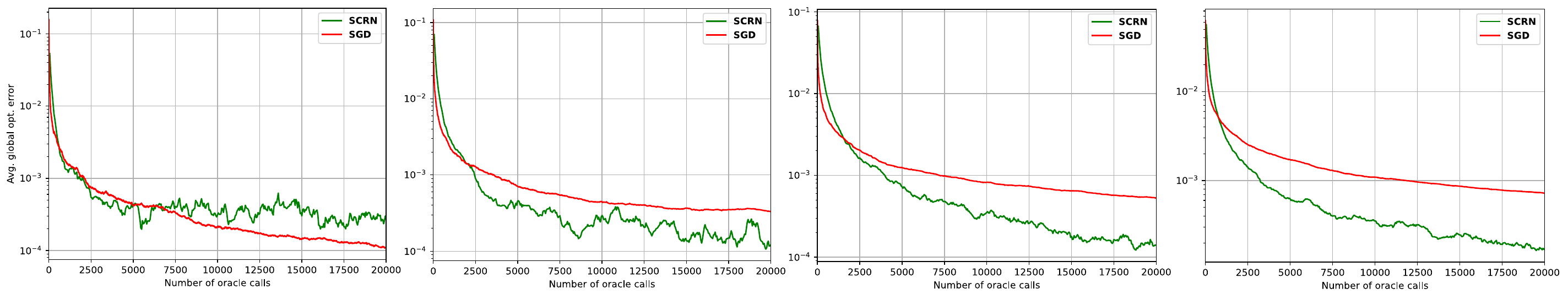}
    \caption{Performance of SCRN versus SGD with adaptive step-size for different values of $\alpha$. From left to right, average global optimization error of SCRN and SGD versus number of oracle calls for $\alpha=4/3,5/4,6/5,7/6$.}
    
    \label{F_alpha<2}
\end{figure}
\begin{example}
As we discussed in Remark \ref{F_tua_varying}, SCRN improves SGD in terms of dependency on $\tau_{F}$. We consider the function $F(x)=x^{2}+a \sin^{2}(x)$ for $0\le a\le 4.60\bar{3}$. This choice of function satisfy gradient dominance property of $\alpha=2$ with different values of $\tau_{F}$. When $a$ goes to $4.60\bar{3}$, $\tau_{F}$ goes to infinity. By tuning the parameter $a$, we run SCRN and SGD with the best adaptive step-size on function $F(x)$ for four different values of $\tau_{F}$ in Figure \ref{F_tua_varying}. As can be seen, SCRN outperforms SGD for large values of $\tau_F$.
\end{example}
\begin{figure}
    \centering
    \includegraphics[width=1\textwidth]{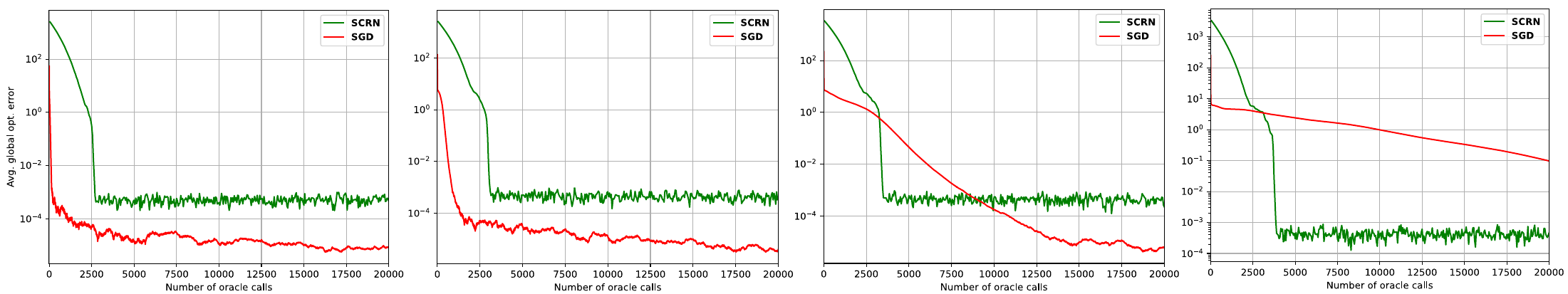}
    \caption{Performance of SCRN versus SGD on a function satisfying gradient dominance property with $\alpha=2$ for different values of $\tau_{F}$. From left to right, average global optimization error versus number of oracle calls for $\tau_{F}\approx8, 32, 85, 190$.}
    \label{F_tua_varying}
\end{figure}

\subsection{variance-reduced SCRN under gradient dominance property with $\alpha=1$} \label{app:vr SCRN}

We first provide two lemmas that are used in analyzing variance-reduced version of SCRN. They are stated for more general case $1\le \alpha<3/2$ in the following. We will use them in the proof of Theorem \ref{th_vr_SCRN_Pl_alpha=1} for the case $\alpha=1$.

\begin{lemma}\label{lemma:rec_eq_noerror}
Consider the following recursive inequality for $1\leq \alpha<3/2$:
\begin{equation}\label{eq:rec_eq}
    \delta_{t+1}\leq (\delta_{t}-\delta_{t+1})^{2\alpha/3}.
\end{equation}
Then,
\begin{itemize}
    \item (i) As far as $\delta_t\geq 1$, we have: $\delta_t\leq (1/2)^t \delta_0$.
    \item (ii) For any pair $(t,t')$ such that $t'\leq t$: 
    \begin{equation}
        \delta_t\leq \beta^{\beta} \left(\beta \delta_{t'}^{-1/\beta}+D(t-t')\right)^{-\beta},
    \end{equation}
    where $\beta:=\frac{2\alpha}{3-2\alpha}$ and $D:= \min\{\frac{1}{2},\beta(2^{\frac{3-2\alpha}{3}}-1)\delta_{t'}^{-1/\beta}\} $.
\end{itemize}
\end{lemma}
\begin{proof}
For part $(i)$, as long as $\delta_t\geq 1$, we have:
\begin{equation}
\begin{split}
    \frac{\delta_t}{\delta_{t-1}}\leq \frac{1}{1+\delta_t^{-1+3/(2\alpha)}}\leq \frac{1}{2},
\end{split}
\end{equation}
where the first inequality is due to \eqref{eq:rec_eq} and the second inequality comes from the fact that $\delta_t\geq 1$. Therefore, $\prod_{k=1}^t \delta_k/\delta_{k-1}\leq (1/2)^t$ which we can imply that: $\delta_t\leq (1/2)^t \delta_0$.

For part $(ii)$, let $h(t):=\frac{2\alpha}{3-2\alpha}t^{1-\frac{3}{2\alpha}}$. Consider the following two cases: 

Case (1): Suppose that $\delta_{t+1}\ge {2}^{-\frac{2\alpha}{3}}\delta_{t}$ for all $t\ge t'$. Then
\begin{align}
    &h(\delta_{t+1})-h(\delta_{t})=\int_{\delta_{t}}^{\delta_{t+1}}\frac{d}{dt}h(t)dt=\int^{\delta_{t}}_{\delta_{t+1}}t^{-\frac{3}{2\alpha}}dt\ge (\delta_{t}-\delta_{t+1})\delta_{t}^{-\frac{3}{2\alpha}}\nonumber\\
    &\ge (\delta_{t}-\delta_{t+1})\frac{1}{2}\delta_{t+1}^{-\frac{3}{2\alpha}}\ge \frac{1}{2}.
\end{align}
Case (2): Suppose that $\delta_{t+1}\le {2}^{-\frac{2\alpha}{3}}\delta_{t}$ for all $t\ge t'$. Then, we have: $\delta_{t+1}^{1-\frac{3}{2\alpha}}\ge 2^{\frac{3-2\alpha}{3}}\delta_{t}^{1-\frac{3}{2\alpha}}$. Therefore,
\begin{align}
    &h(\delta_{t+1})-h(\delta_{t})=\frac{2\alpha}{3-2\alpha}(\delta_{t+1}^{1-\frac{3}{2\alpha}}-\delta_{t}^{1-\frac{3}{2\alpha}})\ge \frac{2\alpha}{3-2\alpha}(2^{\frac{3-2\alpha}{3}}-1)\delta_{t}^{1-\frac{3}{2\alpha}}\nonumber\\
    &\ge \frac{2\alpha}{3-2\alpha}(2^{\frac{3-2\alpha}{3}}-1)\delta_{t'}^{1-\frac{3}{2\alpha}}.
\end{align}
Let $D:=\min\{\frac{1}{2},\frac{2\alpha}{3-2\alpha}(2^{\frac{3-2\alpha}{3}}-1)\delta_{t'}^{1-\frac{3}{2\alpha}}\}$. Then we have
\[
h(\delta_{t+1})-h(\delta_{t})\ge D
\]
which implies 
\[
h(\delta_{t})-h(\delta_{t'})= \sum_{k=t'+1}^{t}h(\delta_{k})-h(\delta_{k-1})\ge D\cdot (t-t').
\]
Then we get
\begin{align}\label{eq_recursion_exp_alpha<3/2_VR}
    \delta_{t}\le \left(\frac{2\alpha}{3-2\alpha}\right)^{\frac{2\alpha}{3-2\alpha}}\frac{1}{(h(\delta_{t'})+D(t-t'))^{\frac{2\alpha}{3-2\alpha}}}.
\end{align}
\end{proof}

\begin{lemma}\label{lemma:rec_eq}
Consider the following recursive inequality for $1\leq \alpha< 3/2$:
\begin{equation}\label{eq:key_rec}
    \delta_{t}\leq C(\delta_{t-1}-\delta_{t})^{2\alpha/3}+\frac{C'}{(\lceil t/S\rceil S)^{\beta}}, \qquad \forall t\geq 1,
\end{equation}
where $\beta:=\frac{2\alpha}{3-2\alpha}$, $C$, and $C'$ are some positive constants and $S$ is a positive integer. Then,
\begin{equation}
    \delta_t\leq  
  \frac{C_{\delta}C^{3/(3-2\alpha)}}{(\lfloor t/S\rfloor S)^{\beta}}+\frac{C'}{(\lceil t/S\rceil S)^{\beta}},
\end{equation}

 for all $t\geq S$ where $C_{\delta}$ satisfies $C_{\delta}^{-1/\beta}\left(1+A^{1/(1+\beta)}\right)\leq \frac{1}{2\beta}$ where $A:=\beta C'/(C^{3/(3-2\alpha)}C_{\delta})$ and $S>\max\{C_{\delta}^{1/\beta}(1+A)^{1/\beta},C_{\delta}^{1/\beta}(t_0-2\beta)/(C_{\delta}^{1/\beta}-2\beta)\}$, and $t_0\leq   \lceil \log(\delta_0/C^{3/(3-2\alpha)})/\log(2)\rceil$.
 
\end{lemma}

\begin{remark}
In the analysis of SCRN where $\delta_{t}=\mathbb{E}[F(\xv_{t})]-F(\xv^{*})$, and $C'=C_{g}+C_{H}$ which are chosen in \eqref{eq20}, if $C_{\delta}=(2\beta)^{\beta+1}C'$, then the condition $C_{\delta}^{-1/\beta}\left(1+\left(\frac{\beta C'}{C^{3/(3-2\alpha)}C_{\delta}}\right)^{1/(1+\beta)}\right)\leq \frac{1}{2\beta}$ turns to
\[
1+\frac{1}{C\cdot 2^{\frac{\beta}{1+\beta}}\beta^{\frac{\beta}{1+\beta}}}\le \beta^{1/\beta}C'^{1/\beta}
\]
which is satisfied for large $M$. Hence, after $T$ iterations, if the errors of gradient and Hessian are
\[
\mathbb{E}[\|\nabla F(\xv_t)-\gv_{t}\|^\alpha]\le\frac{1}{(\lceil t/S\rceil S)^{\beta}},\quad \mathbb{E}[\|\nabla^2F(\xv_t)-\Hv_{t}\|^{2\alpha}]\le\frac{1}{(\lceil t/S\rceil S)^{\beta}},
\]
then
\begin{align}
     \delta_T\leq  
  \frac{(2\beta)^{\beta+1}(C_{g}+C_{H})C^{3/(3-2\alpha)}}{(\lfloor T/S\rfloor S)^{\beta}}+\frac{C_{g}+C_{H}}{(\lceil T/S\rceil S)^{\beta}}.
\end{align}
\end{remark}
\begin{proof}
We define $\delta^k_t:=(\delta_t-C'/(kS)^{\beta})/C^{3/(3-2\alpha)}$ for $(k-1)S\leq t\leq kS$ and $k\geq 1$. Therefore, we have the following recursive inequality: $\delta^k_{t}\leq (\delta^k_{t}-\delta^k_{t-1})^{2\alpha/3}$ for $(k-1)S+1\leq t\leq kS$. By induction on $k$, we will show that $\delta^k_{kS}\leq C_{\delta}/(kS)^{\beta}$ for all $k\geq 1$. We first prove the statement for the base case $k=1$. Suppose $t_0$ is the first iteration such that $\delta^1_{t_0}<1$. From Lemma \ref{lemma:rec_eq_noerror} (part (i)), we know that $t_0\leq \lceil \log(\delta_0/C^{3/(3-2\alpha)})/\log(2)\rceil$. Moreover, from part (ii) in Lemma \ref{lemma:rec_eq_noerror}, by setting $t'=t_0$, we have:
\begin{equation}
\begin{split}
    \delta^1_{S}&\leq\beta^{\beta} \left(\beta (\delta^1_{t_0})^{-1/\beta}+D(S-t_0)\right)^{-\beta}\\
    &\leq \beta^{\beta} \left(\beta +(S-t_0)/2\right)^{-\beta}\\
    &\leq \frac{C_{\delta}}{S^{\beta}}, 
\end{split}
\end{equation}
where the second inequality is due to fact that $D=1/2$ as $\delta^1_{t_0}<1$ and
the third inequality comes from  the constraints on $C_{\delta}$ and $S$ in the statement of lemma.


Now, suppose that the statement holds up to some $k$. We will show that it also holds for $k+1$. From the induction hypothesis, we know that
\begin{equation}\label{eq:delta_k^k+1}
\begin{split}
    \delta_{kS}&\leq \frac{C^{3/(3-2\alpha)}C_{\delta}}{(kS)^{\beta}}+\frac{C'}{(kS)^{\beta}}\\
    \implies& \delta^{k+1}_{kS}\leq \frac{C_{\delta}}{(kS)^{\beta}}+\frac{C'}{C^{3/(3-2\alpha)}}\left(\frac{1}{(kS)^{\beta}}-\frac{1}{((k+1)S)^{\beta}}\right)\\
    &\quad\quad\leq \frac{C_{\delta}}{(kS)^{\beta}}\left(1+\frac{A}{k^{\beta+1}}\right),
\end{split}
\end{equation}
where  $A= \beta C'/(C^{3/(3-2\alpha)}C_{\delta})$ and the second inequality is due to definition of $\delta^{k+1}_{kS}$.

From the constraint on $C_{\delta}$ in the statement of lemma, we have:
\begin{equation}\label{eq:longimplications}
    \begin{split}
        &C_{\delta}^{-1/\beta}\left(1+A^{1/(1+\beta)}\right)\leq \frac{1}{2\beta}\\
        \overset{(a)}{\implies}&  C_{\delta}^{-1/\beta}\left(1+\frac{A^{1/\beta}}{k^{1/\beta}+A^{1/\beta}/k }\right)\leq \frac{1}{2\beta}\\
        \implies&  C_{\delta}^{-1/\beta}\left(k+1-\frac{k}{\left(1+A^{1/\beta}/k^{1/\beta+1}\right)}\right)\leq \frac{1}{2\beta} \\
        \overset{(b)}{\implies}&  C_{\delta}^{-1/\beta}\left(k+1-\frac{k}{\left(1+A/k^{\beta+1}\right)^{1/\beta}}\right)\leq \frac{1}{2\beta} \\
        \implies&  C_{\delta}^{-1/\beta}(k+1)S \leq\frac{kS}{(C_{\delta}(1+A/k^{\beta+1}))^{1/\beta}}+S/(2\beta),
\end{split}
\end{equation}
where $(a)$ is due to fact that $A^{1/(\beta(\beta+1))}<(\beta^{1/(1+\beta)}+\beta^{-\beta/(1+\beta)}) A^{1/(\beta(1+\beta))}\leq k^{1/\beta}+A^{1/\beta}/k$ for $\beta>1$ and $k>0$ and $(b)$ comes from $(a+b)^{1/\beta}\leq (a^{1/\beta}+b^{1/\beta})$ for $\beta>1$ and any $a,b>0$.

Now, from part (ii) of Lemma \ref{lemma:rec_eq_noerror}, we have:
\begin{equation}
\begin{split}
    \delta^{k+1}_{(k+1)S}&\leq \beta^{\beta} \left(\beta (\delta^{k+1}_{kS})^{-1/\beta}+DS\right)^{-\beta}\\
    &\overset{(a)}{\leq} \left(\frac{kS}{(C_\delta(1+A/k^{\beta+1}))^{1/\beta}}+S/(2\beta)\right)^{-\beta} \\
    &\overset{(b)}{\leq} \frac{C_{\delta}}{((k+1)S)^{\beta}},
\end{split}
\end{equation}
 where $(a)$ is according to \eqref{eq:delta_k^k+1} and the fact that $D=1/2$ (as $\delta^{k+1}_{kS}<1$ due to the constraint on $S$ in the statement of lemma) and $(b)$ comes from \eqref{eq:longimplications}.

From the definition of $\delta^{k}_{kS}$, we can imply that for all $k\geq 1$,
\begin{equation}
    \delta_{kS}\leq \frac{C^{3/(3-2\alpha)}C_{\delta}+C'}{(kS)^{\beta}}.
\end{equation}
Moreover, from $(\delta^k_{t})^{3/2\alpha}+\delta^k_{t}\leq \delta^k_{t-1}$ for $(k-1)S+1\leq t\leq kS$, we can imply that: $\delta^{k}_t\leq \delta^k_{(k-1)S}$ for $(k-1)S+1\leq t\leq kS$. Therefore,
\begin{equation}
    \delta_t\leq \frac{C_{\delta}C^{3/(3-2\alpha)}}{(\lfloor t/S\rfloor S)^{\beta}}+\frac{C'}{(\lceil t/S\rceil S)^{\beta}}\quad \text{mod}(t,S)\neq0.
\end{equation}



\end{proof}

The above lemma shows that it suffices to bound the expectations of error terms for stochastic gradient and Hessian by $\mathcal{O}(1/t^{\beta})$ and $\delta_t:=\mathbb{E}[F(\xv_t)]-F(\xv^*)$ still has the same convergence rate in Theorem \ref{th1_expectation}. In the following, we show that 
incorporating time-varying batch sizes in conjunction with variance reduction improves sample complexity results.

\begin{algorithm}
\caption{variance-reduced stochastic cubic regularized Newton method}\label{alg:cap1}
\textbf{Input:} Maximum number of iterations $T$, batch sizes $\{n^{t}_{g}\}_{t=1}^{T}$, $\{n^{t}_{H}\}_{t=1}^{T}$, the period length $S$, initial point $\xv_{0}$, and cubic penalty parameter $M$.
\begin{algorithmic}[1]
\State $t\gets 0$
\While{$t\le T$}
\State Sample index set $\mathcal{J}_{t}$ with $|\mathcal{J}_{t}|=n^{t}_{g}$; $\mathcal{I}_{t}$ with $|\mathcal{I}_{t}|=n^{t}_{H}$
\State \[
\vv_{t} \gets \begin{cases}
\nabla f_{\mathcal{J}_{t}}(\xv_{t}),\quad &\text{mod}(t,S)=0\\
\nabla f_{\mathcal{J}_{t}}(\xv_{t})-\nabla f_{\mathcal{J}_{t}}(\xv_{t-1})+\vv_{t-1},\quad&\text{else}
\end{cases}
\]
\State \[
\Uv_{t} \gets \begin{cases}
\nabla^2 f_{\mathcal{I}_{t}}(\xv_{t}),\quad &\text{mod}(t,S)=0\\
\nabla^2 f_{\mathcal{I}_{t}}(\xv_{t})-\nabla^2 f_{\mathcal{I}_{t}}(\xv_{t-1})+\Uv_{t-1},\quad&\text{else}
\end{cases}
\]
\State ${\bf\Delta}_{t} \gets \argmin_{{\bf\Delta}\in\mathbb{R}^{d}}\langle \vv_{t}, {\bf\Delta}\rangle+\frac{1}{2}\langle {\bf\Delta}, \Uv_{t}{\bf\Delta}\rangle+\frac{M}{6}\|{\bf\Delta}\|^{3}$
\State $\xv_{t+1}\gets \xv_{t}+{\bf \Delta}_{t}$
\State $t \gets t+1$
\EndWhile
\State\Return{$\xv_{t}$}
\end{algorithmic}
\label{algorithm2}
\end{algorithm}

We consider the following assumptions:
\begin{assumption}\label{assump:individual_smoothness}
We assume that $f(\xv,\xi)$ satisfies $L'_1$-average smoothness and $(L'_2,\alpha)$-average Hessian Lipschitz continuity, i.e., $\mathbb{E}[\|\nabla f(\xv,\xi) -\nabla f(\yv,\xi)\|^{2}] \leq L'^{2}_{1}\|\xv-\yv\|^{2}$ and $\mathbb{E}[\|\nabla^2 f(\xv,\xi)-\nabla^2 f(\yv,\xi)\|^{2\alpha}] \leq L'^{2\alpha}_{2}\|\xv-\yv\|^{2\alpha}$ for all $\xv,\yv \in \mathbb{R}^d$ and $1\le\alpha<3/2$.
\end{assumption}


It is noteworthy that the above assumption hold in policy-based RL setting that is considered in this paper.

Let define $\nabla f_{\mathcal{J}_{t}}(\xv_{t}):=\frac{1}{|\mathcal{J}_{t}|}\sum_{j\in\mathcal{J}_{t}}\nabla f(\xv_{t},\xi_{j})$ and $\nabla^{2} f_{\mathcal{I}_{t}}(\xv_{t}):=\frac{1}{|\mathcal{I}_{t}|}\sum_{i\in\mathcal{I}_{t}}\nabla^2 f(\xv_{t},\xi_{i})$.
By adapting the variance reduction method in \cite{zhou2020stochastic} (see Algorithm \ref{algorithm2}), we can have the following bound on the errors of estimated gradient $\vv_t$ and Hessian $\Uv_t$:

\begin{lemma} \label{lemma:vr}
Let $n_g^t$ and $n_H^t$ be the number of stochastic gradients and Hessian matrices that are taken at time $t$. Suppose that for a given $\epsilon>0$, we take the following number of samples at checkpoints:
\begin{equation}
    n_g^t\ge \frac{2\sigma_{1}^{2}}{\epsilon^{2/\alpha}}, n_H^t\ge \frac{2^{1/\alpha}(8e \cdot \log d)^{2}\cdot\sigma_{2,\alpha}^{2/\alpha}S^{1-1/\alpha}}{\epsilon^{1/\alpha}}.
\end{equation}
and at the other iterations, we take the following number of samples:
\begin{equation}\label{eq:bounds on samples of grad and hessian}
    n_g^t(\xv_{t},\xv_{t-1})\ge \frac{4 L'^{2}_{1} S \|\mathbf{\Delta}_{t-1}\|^2}{\epsilon^{2/\alpha}}, n_H^t(\xv_{t},\xv_{t-1})\ge \frac{4\times 2^{1/\alpha}(8e\cdot \log d)^{2}\cdot L'^{2}_{2} S \|\mathbf{\Delta}_{t-1}\|^{2}}{\epsilon^{1/\alpha}},
\end{equation}
Then, under Assumptions \ref{assump2} and \ref{assump:individual_smoothness} for a specific amount of $\alpha\in[1,3/2)$, we have:
\begin{equation}\label{eq:bounds on grad and hessian}
    \mathbb{E}[\|\nabla F(\xv_t)-\vv_t\|^{\alpha}]\leq \epsilon, \qquad \mathbb{E}[\|\nabla^2 F(\xv_t)-\Uv_t\|^{2\alpha}]\leq \epsilon.
\end{equation}
\end{lemma}
\begin{proof}
First for the gradient estimator $\vv_{t}$, we have $\vv_{t}-\nabla F(\xv_{t})=\sum_{k=\lfloor t/S\rfloor S}^{t}\uv_{k}$ such that
\[
\uv_{k}=\begin{cases}
\nabla f_{\mathcal{J}_{k}}(\xv_{k})-\nabla F(\xv_{k}),\quad &k=\lfloor t/S\rfloor S,\\
\nabla f_{\mathcal{J}_{k}}(\xv_{k})-\nabla f_{\mathcal{J}_{k}}(\xv_{k-1})-\nabla F(\xv_{k})+\nabla F(\xv_{k-1}),\quad& k>\lfloor t/S\rfloor S.
\end{cases}
\]
We know that $\mathbb{E}[\uv_{k}|\mathcal{G}_{k}]=0$ where $\mathcal{G}_{k}=\sigma(\xv_{0},\ldots,\xv_{k}, \uv_0,\cdots,\uv_{k-1})$. Conditioned on $\mathcal{G}_{k}$ and for $k>\lfloor t/S\rfloor S$, we have
\begin{align}
    &\mathbb{E}[\|\uv_{k}\|^{2}|\mathcal{G}_{k}]=\mathbb{E}\left\|\frac{1}{n^{k}_{g}}\sum_{i=1}^{n^{k}_{g}}\nabla f(\xv_{k},\xi_{i})-\nabla f(\xv_{k-1},\xi_{i})-\nabla F(\xv_{k})+\nabla F(\xv_{k-1})\right\|^{2}\nonumber\\
    &\overset{(a)}{=}\frac{1}{n^{k}_{g}}\mathbb{E}\left\|\nabla f(\xv_{k},\xi_{1})-\nabla f(\xv_{k-1},\xi_{1})-\nabla F(\xv_{k})+\nabla F(\xv_{k-1})\right\|^{2}\nonumber\\
    &\overset{(b)}{\le} \frac{2}{n^{k}_{g}}\mathbb{E}\|\nabla f(\xv_{k},\xi_{1})-\nabla f(\xv_{k-1},\xi_{1})\|^{2}+\frac{2}{n^{k}_{g}}\mathbb{E}\|\nabla F(\xv_{k})-\nabla F(\xv_{k-1})\|^{2}\nonumber\\
    &\le \frac{4L'^{2}_{1}}{n^{k}_{g}}\|\xv_{k}-\xv_{k-1}\|^{2}
\end{align}
 where (a) comes from $[\nabla f(\xv_{k},\xi_{i})-\nabla f(\xv_{k-1},\xi_{i})-\nabla F(\xv_{k})+\nabla F(\xv_{k-1})]$'s are i.i.d conditioned on $\mathcal{G}_k$ for $1\le i\le n^{k}_{g}$. Inequality (b) is from $(a+b)^{2}\le 2a^{2}+2b^{2}$ and (c) is derived by Assumption \ref{assump:individual_smoothness}. For $k=\lfloor t/S\rfloor S$, $\mathbb{E}[\|\uv_{k}\|^{2}]\le \frac{\sigma_{1}^{2}}{n^{k}_{g}}$. 
\begin{align}
    &\mathbb{E}[\|\nabla F(\xv_{t})-\vv_{t}\|^{2}]=\mathbb{E}\left\|\sum_{k=\lfloor t/S\rfloor S}^{t}\uv_{k}\right\|^{2}\le \sum_{k=\lfloor t/S\rfloor S}^{t}\mathbb{E}[\mathbb{E}[\|\uv_{k}\|^{2}|\mathcal{G}_{k}]]\nonumber\\
    &\le \frac{\sigma_{1}^{2}}{n^{\lfloor t/S\rfloor S}_{g}}+\sum_{k=\lfloor t/S\rfloor S+1}^{t}{4L'^{2}_{1}}\mathbb{E}\left[\frac{\|\xv_{k}-\xv_{k-1}\|^{2}}{n^{k}_{g}}\right]
\end{align}
where the first inequality comes from the fact that $\mathbb{E}[\uv_{k}^{T}\uv_{k'}]=\mathbb{E}[\uv_{k'}^{T}\mathbb{E}[\uv_{k}|\mathcal{G}_{k}]]=0$ for $k> k'$.

If we take $ n_g^k\ge  \frac{2\sigma_{1}^{2}}{\epsilon^{2/\alpha}}$ samples at checkpoints ($k=\lfloor t/S\rfloor S$) and and at the other iterations, we take $n_g^k\ge \frac{4 L'^{2}_{1} S \|\mathbf{\Delta}_{k-1}\|^2}{\epsilon^{2/\alpha}}$ samples, we have
\[
\mathbb{E}[\|\nabla F(\xv_{t})-\vv_{t}\|^{\alpha}]\le \left(\mathbb{E}[\|\nabla F(\xv_{t})-\vv_{t}\|^{2}]\right)^{\alpha/2}\le \epsilon.
\]
Now we give a proof for Hessian estimator which is similar to the gradient one.

First for the Hessian estimator $\Uv_{t}$, we have $\nabla^{2} F(\xv_{t})-\Uv_{t}=\sum_{k=\lfloor t/S\rfloor S}^{t}\Vv_{k}$ such that
\[
\Vv_{k}=\begin{cases}
\nabla^{2} f_{\mathcal{J}_{k}}(\xv_{k})-\nabla^{2} F(\xv_{k}),\quad &k=\lfloor t/S\rfloor S,\\
\nabla^{2} f_{\mathcal{J}_{k}}(\xv_{k})-\nabla^{2} f_{\mathcal{J}_{k}}(\xv_{k-1})-\nabla^{2} F(\xv_{k})+\nabla^{2} F(\xv_{k-1}),\quad& k>\lfloor t/S\rfloor S.
\end{cases}
\]
Let $\mathcal{H}_{k}=\sigma(\xv_0,\cdots,\xv_k,\Vv_{0},\ldots,\Vv_{k-1})$. Then, we have $\mathbb{E}[\Vv_{k}|\mathcal{H}_{k}]=0$. Conditioned on $\mathcal{H}_{k}$ and for $k>\lfloor t/S\rfloor S$, we have
\begin{align}
    &\mathbb{E}[\|\Vv_{k}\|^{2\alpha}|\mathcal{H}_{k}]=\mathbb{E}\left\|\frac{1}{n^{k}_{H}}\sum_{i=1}^{n^{k}_{H}}\nabla^{2} f(\xv_{k},\xi_{i})-\nabla^{2} f(\xv_{k-1},\xi_{i})-\nabla^{2} F(\xv_{k})+\nabla^{2} F(\xv_{k-1})\right\|^{2\alpha}\nonumber\\
    &\overset{(a)}{\le}(8e\log d)^{\alpha}\cdot\mathbb{E}\left\|\frac{1}{(n^{k}_{H})^{2}}\sum_{i=1}^{n^{k}_{H}}[\nabla^{2} f(\xv_{k},\xi_{i})-\nabla^{2} f(\xv_{k-1},\xi_{i})-\nabla^{2} F(\xv_{k})+\nabla^{2} F(\xv_{k-1})]^{2}\right\|^{\alpha}\nonumber\\
    &\overset{(b)}{\le}\frac{(8e\log d)^{\alpha}}{(n^{k}_{H})^{\alpha}}\cdot \frac{1}{n^{k}_{H}}\sum_{i=1}^{n^{k}_{H}}\mathbb{E}\left\|[\nabla^{2} f(\xv_{k},\xi_{i})-\nabla^{2} f(\xv_{k-1},\xi_{i})-\nabla^{2} F(\xv_{k})+\nabla^{2} F(\xv_{k-1})]^{2}\right\|^{\alpha}\nonumber\\
    &\overset{(c)}{\le}\frac{(8e\log d)^{\alpha}}{(n^{k}_{H})^{\alpha}}\mathbb{E}\left\|[\nabla^{2} f(\xv_{k},\xi_{1})-\nabla^{2} f(\xv_{k-1},\xi_{1})-\nabla^{2} F(\xv_{k})+\nabla^{2} F(\xv_{k-1})]\right\|^{2\alpha}\nonumber\\
    &\overset{(d)}{\le} \frac{2^{2\alpha-1}(8e\log d)^{\alpha}}{(n^{k}_{H})^{\alpha}}\mathbb{E}\|\nabla^{2} f(\xv_{k},\xi_{1})-\nabla^{2} f(\xv_{k-1},\xi_{1})\|^{2\alpha}+\frac{2^{2\alpha-1}(8e\log d)^{\alpha}}{(n^{k}_{H})^{\alpha}}\mathbb{E}\|\nabla^{2} F(\xv_{k})-\nabla^{2} F(\xv_{k-1})\|^{2\alpha}\nonumber\\
    &\overset{(e)}{\le} \frac{2^{2\alpha}(8e\log d)^{\alpha}(L'_{1})^{2\alpha}}{(n^{k}_{H})^{\alpha}}\|\xv_{k}-\xv_{k-1}\|^{2\alpha},
\end{align}
where (a) comes from Lemma \ref{lemma_norm_op_summation_of_iid_matrix} and (b) comes from Jensen's inequality for operator norm $\|\cdot\|^{\alpha}$. (c) is derived by $\|AB\|\le \|A\|\|B\|$ and the fact that $[\nabla^{2} f(\xv_{k},\xi_{i})-\nabla^{2} f(\xv_{k-1},\xi_{i})-\nabla^{2} F(\xv_{k})+\nabla^{2} F(\xv_{k-1})]$'s are i.i.d conditioned on $\mathcal{H}_{k}$ for $1\le i\le n^{k}_{g}$. Inequality (d) is from inequality $(a+b)^{2\alpha}\le 2^{2\alpha-1}(a^{2\alpha}+b^{2\alpha})$ and (e) is derived by Assumption \ref{assump:individual_smoothness}. For $k=\lfloor t/S\rfloor S$, $\mathbb{E}[\|\Vv_{k}\|^{2\alpha}]\le \frac{(8e\log d)^{\alpha}\cdot\sigma_{2,\alpha}^{2}}{(n^{k}_{H})^{\alpha}}$ from Lemma \ref{lemma_norm_op_summation_of_iid_matrix}. Then
\begin{align}
    &\mathbb{E}[\|\nabla^{2} F(\xv_{t})-\Uv_{t}\|^{2\alpha}]=\mathbb{E}\left\|\sum_{k=\lfloor t/S\rfloor S}^{t}\Vv_{k}\right\|^{2\alpha}\overset{(a)}{\le}(8e\log d)^{\alpha}\mathbb{E}\left\|\frac{S}{S}\cdot\sum_{k=\lfloor t/S\rfloor S}^{t}\Vv_{k}^{2}\right\|^{\alpha}
    \nonumber\\
    &\overset{(b)}{\le}{(8e\log d)^{\alpha}}{S^{\alpha-1}}\sum_{k=\lfloor t/S\rfloor S}^{t}\mathbb{E}[\|\Vv_{k}^{2}\|^{\alpha}]\overset{(c)}{\le}{(8e\log d)^{\alpha}}{S^{\alpha-1}}\sum_{k=\lfloor t/S\rfloor S}^{t}\mathbb{E}[\|\Vv_{k}\|^{2\alpha}]\nonumber\\
    &\le \frac{(8e\log d)^{2\alpha}\sigma_{2,\alpha}^{2}S^{\alpha-1}}{(n^{\lfloor t/S\rfloor S}_{H})^{\alpha}}+\sum_{k=\lfloor t/S\rfloor S+1}^{t}{2^{2\alpha}(8e\log d)^{2\alpha}S^{\alpha-1}(L'_{2})^{2\alpha}}\mathbb{E}\left[\frac{\|\xv_{k}-\xv_{k-1}\|^{2\alpha}}{(n^{k}_{H})^{\alpha}}\right],
\end{align}
where (a) is derived by Lemma \ref{lemma_norm_op_summation_of_iid_matrix}. (b) comes from Jensen's inequality for $\|\cdot\|^{\alpha}$ and (c) is from $\|AB\|\le\|A\|\|B\| $. 

If we take $ n_H^k\ge  \frac{2^{1/\alpha}(8e \cdot \log d)^{2}\cdot\sigma_{2,\alpha}^{2/\alpha}S^{1-1/\alpha}}{\epsilon^{1/\alpha}}$ samples at checkpoints ($k=\lfloor t/S\rfloor S$) and and at the other iterations, we take $n_H^k\ge \frac{4\times2^{1/\alpha}(8e\cdot \log d)^{2}\cdot L'^{2}_{2} S \|\mathbf{\Delta}_{k-1}\|^{2}}{\epsilon^{1/\alpha}}$ samples, we have
\[
\mathbb{E}[\|\nabla^{2} F(\xv_{t})-\Uv_{t}\|^{2\alpha}]\le \epsilon.
\]
\end{proof}

\textbf{Theorem 3.}
\textit{
Under gradient dominance property with $\alpha=1$, Assumptions \ref{assump1}, \ref{assump2} for $\alpha=1$, and \ref{assump:individual_smoothness} for $\alpha=1$, Algorithm \ref{algorithm2} can achieve $\epsilon$-global stationary point in expectation by querying $\mathcal{O}(\epsilon^{-2})$ stochastic gradients and ${\mathcal{O}}(\epsilon^{-{1}})$ stochastic Hessian on average.}

\begin{proof}
From Lemma \ref{lemma:two_ineq}, using the estimates of gradient and Hessian from $\vv_t$ and $\Uv_t$, we have for $\alpha=1$:
\begin{equation}
      F(\xv_{t}+{\bf {\Delta}}_{t})-F(\xv^*)\leq C(F(\xv_t)-F(\xv_{t}+{\bf {\Delta}}_{t}))^{2/3}+C_g\|\nabla F(\xv_t)-\vv_{t}\|+C_{H}\|\nabla^2F(\xv_t)-\Uv_{t}\|^{2}.
\end{equation}
By defining $\delta_t:=\mathbb{E}[F(\xv_t)]-F(\xv^*)$ and using Jensen's inequality, we can rewrite the above inequality as follows:
\begin{equation}
    \delta_{t+1}\leq C(\delta_t-\delta_{t+1})^{2/3}+C_g\mathbb{E}[\|\nabla F(\xv_t)-\vv_{t}\|]+C_{H}\mathbb{E}[\|\nabla^2F(\xv_t)-\Uv_{t}\|^{2}].
\end{equation}

Now, suppose for $(k-1)S<t\leq kS$, $k\geq 1$, we set $\epsilon$ to $(kS)^{-{2}}$ in $n_g^t$'s and $n_H^t$'s. Then, from Lemma \ref{lemma:vr}, for all $t\geq 0,$ 
\begin{equation}
    \delta_{t+1}\leq C(\delta_t-\delta_{t+1})^{2/3}+\frac{C_g+C_{H}}{(\lceil t/S\rceil S)^{2}}, 
\end{equation}
 and from Lemma \ref{lemma:rec_eq}, 
 \begin{align}
     \delta_T\leq  
  \frac{16(C_{g}+C_{H})C^{3}}{(\lfloor T/S\rfloor S)^{2}}+\frac{C_{g}+C_{H}}{(\lceil T/S\rceil S)^{2}}.
 \end{align}
 After
\begin{align}
T\ge\frac{\left[16(C_{g}+C_{H})C^{3}+C_{g}+C_{H}\right]^{1/2}}{\epsilon^{1/2}}
\end{align}
 iterations, $\delta_{T}\le\epsilon$.
 Furthermore, according to Lemma \ref{lemma:two_ineq} \eqref{eq:2}, we have:
 \begin{align}
\frac{3M-2L_2-8}{12}\|{{\bf{\Delta}}_{t}}\|^3&\leq F(\xv_t)-F(\xv_{t}+{{\bf{\Delta}}_{t}})+\frac{2\|\nabla F(\xv_t)-\gv_{t}\|^{3/2}}{3}+ \frac{\|\nabla^2 F(\xv_t)-\Hv_{t}\|^{3}}{6}
 \end{align}
 and using inequality $(a+b+c)^{2/3}\le 3^{2/3-1}(a^{2/3}+b^{2/3}+c^{2/3})$ for $a,b,c>0$ and then taking expectation, we have:
 \begin{equation}\label{eq00123}
 \begin{split}
    &\mathbb{E}[\|\mathbf{\Delta}_t\|^{2}]\overset{(a)}{\le}\\
    & C''[(\mathbb{E}[F(\xv_t)]-\mathbb{E}[F(\xv_{t}+{{\bf{\Delta}_{t}}})])^{2/3}+c_{g}\mathbb{E}\|\nabla F(\xv_t)-\gv_{t}\|+ c_{H}\mathbb{E}[\|\nabla^2 F(\xv_t)-\Hv_{t}\|^{2}]]\\
    &\implies \sum_{t=(k-1)S}^{kS-1} \mathbb{E}[\|\mathbf{\Delta}_t\|^2]\le C''\left[\sum_{t=(k-1)S}^{kS-1}(\delta_t-\delta_{t+1})^{2/3}+c_{g}\mathbb{E}\|\nabla F(\xv_t)-\gv_{t}\|+ c_{H}\mathbb{E}[\|\nabla^2 F(\xv_t)-\Hv_{t}\|^{2}]\right]\\
    &\qquad\qquad\qquad\qquad\qquad\overset{(b)}{\le}C''\left[S^{1/3}(\delta_{(k-1)S}-\delta_{kS})^{2/3}+\frac{(c_{g}+c_{H})}{(kS)^{\beta/ \alpha}}\right]\\
    &\qquad\qquad\overset{(c)}{\le}C''\left[S^{1/3}\left(\frac{16(C_{g}+C_{H})C^{3}}{((k-1) S)^{2}}+\frac{C_{g}+C_{H}}{((k -1)S)^{2}}\right)^{2/3}+\frac{(c_{g}+c_{H})}{(kS)^{2}}\right]=\mathcal{O}(S^{1/3}(kS)^{-4/3}),
\end{split}
\end{equation}
where $C''=\frac{12^{2/3}}{3^{1/3}(3M-2L_{2}-8)^{2/3}}$, $c_{g}=\frac{2}{3}$, and $c_{H}=\frac{1}{6}$ in (a), and in (b), we used Jensen's inequality for the first term in the sum:
\[
\frac{1}{S}\sum_{t=(k-1)S}^{kS-1}(\delta_t-\delta_{t+1})^{2/3}\le \left(\frac{1}{S}\sum_{t=(k-1)S}^{kS-1}\delta_t-\delta_{t+1}\right)^{2/3}=\left(\frac{1}{S}(\delta_{(k-1)S}-\delta_{kS})\right)^{2/3}
\]
and we set $n_g^t$'s and $n_H^t$'s such that 
\begin{align*}
    &\mathbb{E}[\|\nabla F(\xv_t)-\gv_{t}\|]\le \frac{1}{(kS)^{2}}\\
    &\mathbb{E}[\|\nabla^2 F(\xv_t)-\Hv_{t}\|^{2}]\le \frac{1}{(kS)^{2}}.
\end{align*}
Moreover, (c) is according to Lemma \ref{lemma:rec_eq}. 

Hence, according to Lemma \ref{lemma:vr} for the case $\alpha=1$, in order to obtain the following errors:
\[
\mathbb{E}[\|\nabla F(\xv_t)-\gv_{t}\|])\le \frac{1}{(kS)^{2}},\quad  \mathbb{E}[\|\nabla^2 F(\xv_t)-\Hv_{t}\|^{2}]\le \frac{1}{(kS)^{2}},
\]
the average sample complexity of gradient $\mathbb{E}\left[\sum_{t=1}^T n_g^t\right]$ must be in order of
\begin{align}\label{sample_g_VSCRN}
    \sum_{\mod(t,S)=0}\frac{2\sigma_{1}^{2}}{t^{-4}} +\sum_{\mod(t,S)\neq 0} \frac{4\cdot L'^{2}_{1} S \mathbb{E}[\|\mathbf{\Delta}_{t-1}\|^2]}{(\lceil t/S\rceil S)^{-4}}.
\end{align}
It is enough to have $\mathbb{E}\left[\sum_{t=1}^T n_g^t\right]$ be greater than the following upper bound for \eqref{sample_g_VSCRN}:
 \begin{equation}\label{sample_gradient_VRSCRN}
 \begin{split}
           &\sum_{\mod(t,S)=0}\frac{2\sigma_{1}^{2}}{t^{-4}} +\sum_{\mod(t,S)\neq 0} \frac{4\cdot L'^{2}_{1} S \mathbb{E}[\|\mathbf{\Delta}_{t-1}\|^2]}{(\lceil t/S\rceil S)^{4}}\\
          &= \sum_{k=1}^{\lceil T/S\rceil}2\sigma_{1}^{2}(kS)^{4} +\sum_{k=1}^{\lceil T/S\rceil}\sum_{t=(k-1)S+1}^{kS}4\cdot L'^{2}_{1} S (kS)^{4}\mathbb{E}[\|\mathbf{\Delta}_{t-1}\|^2] \\  
     &\overset{(a)}{\le}  2\sigma_{1}^{2}\sum_{k=1}^{\lceil T/S\rceil}(kS)^{4} +\sum_{k=1}^{\lceil T/S\rceil}4\cdot L'^{2}_{1} S (kS)^{4}C''\left[S^{1/3}\left(\frac{16(C_{g}+C_{H})C^{3}}{(k S)^{2}}+\frac{C_{g}+C_{H}}{(k S)^{2}}\right)^{2/3}+\frac{(c_{g}+c_{H})}{(kS)^{2}}\right]\\
     &\le 2\sigma_{1}^{2}\frac{T^{5}}{S}+4c\cdot L'^{2}_{1} C''\left[\left[\left(16(C_{g}+C_{H})C^{3}\right)+(C_{g}+C_{H})\right]^{2/3}+\frac{c_{g}+c_{H}}{(kS)^{2(1 - 2/3)}}\right]S^{1/3} T^{5-4/3}
 \end{split}
 \end{equation}
 where in (a), we used \eqref{eq00123}. If $S=\lfloor \frac{T}{q}\rfloor$ where $q$ is a constant integer, the average sample complexity of gradient, would be in the order of $\mathcal{O}(T^{4})$. As $T\ge \frac{\left[16(C_{g}+C_{H})C^{3}+C_{g}+C_{H}\right]^{1/2}}{\epsilon^{1/2}}$, the average sample complexity of gradient is at least
 \[
 \mathcal{O}\left(\frac{2\sigma_{1}^{2}\left[16(C_{g}+C_{H})C^{3}+C_{g}+C_{H}\right]^{2}}{\epsilon^{2}}\right)
 \]
 in order to get $\delta_{T}\le \epsilon$.
 
Moreover, for the average sample complexity of Hessian $\mathbb{E}\left[\sum_{t=1}^T n_H^t\right]$, we have:
 \begin{equation}\label{sample_Hessian_VRSCRN}
 \begin{split}
     &\sum_{\mod(t,S)=0}\frac{(8e \cdot \log d)^{2}\cdot\sigma_{2,1}^{2}}{(t)^{-2}} +\sum_{\mod(t,S)\neq 0} 
           \frac{2(8e\cdot \log d)^{2}\cdot L'^{2}_{2} S \mathbb{E}[\|\mathbf{\Delta}_{t-1}\|^{2}]}{(\lceil t/S\rceil S)^{-2}}\\
          &= \sum_{k=1}^{\lceil T/S\rceil}{(8e \cdot \log d)^{2}\cdot\sigma_{2,1}^{2}}{(kS)^{2}}+\sum_{k=1}^{\lceil T/S\rceil}\sum_{t=(k-1)S+1}^{kS} {2(8e\cdot \log d)^{2}\cdot L'^{2}_{2} S }{(k S)^{2}}\mathbb{E}[\|\mathbf{\Delta}_{t-1}\|^{2}]\\  
     &\le  {(8e \cdot \log d)^{2}\cdot\sigma_{2,1}^{2}S^{2}}\sum_{k=1}^{\lceil T/S\rceil}{k^{2}}+\\
     &\sum_{k=1}^{\lceil T/S\rceil} {2(8e\cdot \log d)^{2}\cdot L'^{2}_{2} S }{(k S)^{2}}C''\left[S^{1/3}\left(\frac{16(C_{g}+C_{H})C^{3}}{(k S)^{2}}+\frac{C_{g}+C_{H}}{(k S)^{2}}\right)^{2/3}+\frac{(c_{g}+c_{H})}{(kS)^{2}}\right]\\
     & \le  {(8e \cdot \log d)^{2}\cdot\sigma_{2,1}^{2}}{T^{3}}+\\
     & 2(8e\cdot \log d)^{2}\cdot L'^{2}_{2} S^{1/3}(T)^{3-4/3}C''\left([16(C_{g}+C_{H})C^{3}+(C_{g}+C_{H})]^{2/3}+\frac{c_{g}+c_{H}}{(kS)^{2(1-2/3)}}\right)
 \end{split}
 \end{equation}
 
By setting $S=\lfloor\frac{T}{q}
\rfloor$ where $q$ is a positive integer constant, the average sample complexity of Hessian, would be in the order of $\mathcal{O}(T^{2})$. As $T=\mathcal{O}(\epsilon^{-\frac{1}{2}})$, the average sample complexity of Hessian is in the order of $\mathcal{O}\left(\frac{1}{\epsilon}\right)$.

\end{proof}

\subsection{Proof of RL results}
Two important classes of policies:
1)  Scalar-action, fixed-variance Gaussian policy:
\begin{align}\label{gaussian_policy}
    \pi_{\theta}(a|s)=\frac{1}{\sigma\sqrt{2\pi}}\exp\left\{-\frac{1}{2}\left(\frac{a-\theta^{T}\phi(s)}{\sigma}\right)^{2}\right\}.
\end{align}
2) Softmax tabular policy:
\begin{align}\label{soft_max_policy}
    \pi_{\theta}(a|s)=\frac{\exp(\theta_{s,a})}{\sum_{a'\in\mathcal{A}}\exp(\theta_{s,a'})}
\end{align}
where $\theta=\{\theta_{s,a}\}_{(s,a)\in\mathcal{S}\times\mathcal{A}}$ are parameters of the policy.
\subsubsection{Discussion on gradient dominance property with $\alpha=1$ for Policies}\label{discussion on PL_alpha=1_RL}
\begin{assumption}[Fisher-non-degenerate.]\label{Fisher-non-deg}
For all $\theta\in \mathbb{R}^{d}$, there exists $\mu_{F}>0$ such that the Fisher information matrix $F_{\rho}(\theta)$ induced by policy $\pi_{\theta}$ and initial distribution $\rho$ satisfies 
\begin{align}
    F_{\rho}(\theta):=\mathbb{E}_{(s,a)\sim v_{\rho}^{\pi_{\theta}}}[\nabla_{\theta}\log\pi_{\theta}(a|s)\nabla_{\theta}\log\pi_{\theta}(a|s)^{T}]\succeq \mu_{F}I_{d\times d},
\end{align}
where $v_{\rho}^{\pi_{\theta}}(s,a):=(1-\gamma)\mathbb{E}_{s_{0}\sim \rho}\sum_{t=0}^{\infty}\gamma^{t}\mathbb{P}(s_{t}=s,a_{t}=a|s_{0},\pi_{\theta})$ is the state-action visitation measure.
\end{assumption}
This assumption is commonly used in the literature \citep{liu2020improved,ding2021global}. The Fisher-non-degenerate setting implicitly guarantees that the agent is able to explore the state-action
space under the considered policy class. Similar conditions of the Fisher-non-degeneracy is also required in other global optimum convergence framework (Assumption 6.5 in \citep{agarwal2021theory} on the relative condition number and Assumption 3 in \citep{cayci2021linear} on the regularity of the parametric model). Assumption \ref{Fisher-non-deg} is satisfied by a wide families of policies, including the Gaussian policy \eqref{gaussian_policy} and certain neural policy. Without the non-degenerate Fisher information matrix condition,
the global optimum convergence of more general parameterizations would be hard to analyze without introducing
the additional exploration procedures in the non-tabular setting \citep[Sec. 8]{ding2021global}.

\begin{assumption}\label{compat_approx}
For all $\theta\in \mathbb{R}^{d}$, there exists $\epsilon_{bias}>0$ such that the transferred compatible function approximation error with $(s,a)\sim v_{\rho}^{\pi_{\theta^{*}}}$ satisfies 
\begin{align}
    \mathbb{E}[(A^{\pi_{\theta}}(s,a)-(1-\gamma){u^{*}}^{T}\nabla_{\theta}\log\pi_{\theta}(a|s))^{2}]\le \epsilon_{bias},
\end{align}
where $v_{\rho}^{\pi_{\theta^{*}}}$ is the state-action distribution induced by an optimal policy, and $u^{*}=(F_{\rho}(\theta))^{\dagger}\nabla J(\theta)$.
\end{assumption}
This is also a common assumption \citep{wang2019neural,agarwal2021theory,yuan2021general,liu2020improved,ding2021global}. In particular, when $\pi_{\theta}$ is a soft-max tabular policy \eqref{soft_max_policy}, $\epsilon_{bias}$ is 0 \citep{ding2021global}; Moreover, it can be shown that when $\pi_{\theta}$ is a rich neural policy, $\epsilon_{bias}$ is small \citep{wang2019neural}. 

By Lemma 4.7 in \citep{ding2021global}, Assumptions \ref{Fisher-non-deg} and \ref{compat_approx} yield relaxed gradient dominance property with $\alpha=1$ (See Assumption \ref{relaxed weak PL}). We provide the proof in the sequel:

From performance difference lemma \citep{kakade2002approximately}, we have
\begin{align}\label{perf_diff_lemma}
    \mathbb{E}_{(s,a)\sim v_{\rho}^{\pi_{\theta^{*}}}}[A^{\pi_{\theta_{t}}}(s,a)]=(1-\gamma)(J_{\rho}(\theta^{*})-J_{\rho}(\theta_{t})),
\end{align}
and from Assumption \ref{compat_approx} and Jensen's inequality
\begin{align}\label{compatible_Jensen}
   \mathbb{E}[A^{\pi_{\theta}}(s,a)-(1-\gamma){u^{*}}^{T}\nabla_{\theta}\log\pi_{\theta}(a|s)]\le \sqrt{\epsilon_{bias}}.
\end{align}
Substituting  \eqref{perf_diff_lemma} into \eqref{compatible_Jensen}, we get
\begin{align}
   &J_{\rho}(\theta^{*})-J_{\rho}(\theta_{t})\le \frac{1}{1-\gamma}\sqrt{\epsilon_{bias}}+\mathbb{E}[{u^{*}}^{T}\nabla_{\theta}\log\pi_{\theta}(a|s)]\nonumber\\
   &\le \frac{1}{1-\gamma}\sqrt{\epsilon_{bias}}+\frac{G_{1}}{\mu_{F}}\|\nabla J_{\rho}(\theta_{t})\|,
\end{align}
where the last inequality comes from following Cauchy-Schwartz inequality
\begin{align*}
&\mathbb{E}[{u^{*}}^{T}\nabla_{\theta}\log\pi_{\theta}(a|s)]=\mathbb{E}[\nabla J(\theta)^{T}(F_{\rho}(\theta))^{\dagger}\nabla_{\theta}\log\pi_{\theta}(a|s)]\\
&\le \mathbb{E}\|\nabla J(\theta)\|_{2}\mathbb{E}\|(F_{\rho}(\theta))^{\dagger}\nabla_{\theta}\log\pi_{\theta}(a|s)\|_{2}\le \frac{1}{\mu_{F}}\mathbb{E}\|\nabla_{\theta}\log\pi_{\theta}(a|s)\|_{2}\mathbb{E}\|\nabla J(\theta)\|_{2}\le \frac{G_{1}}{\mu_{F}}\mathbb{E}\|\nabla J(\theta)\|_{2}
\end{align*}
where two last inequalities are from Assumptions \ref{Fisher-non-deg} and \ref{LS}. Thus, by setting $\epsilon':=\frac{1}{1-\gamma}\sqrt{\epsilon_{bias}}$ and $\tau_{J}:=\frac{G}{\mu_{F}}$,
\begin{align}
    J^{*}-J(\theta_{t})\le \tau_{J}\|\nabla J(\theta_{t})\|_{2}+\epsilon'.
\end{align}

In the past few years, several work have attempted to establish that  PG and SPG converge  to a global optimal point for various classes of policies \citep{zhang2020variational, zhang2021convergence}.
For instance, for Fisher non-degenerate policies, Liu et al. \cite{liu2020improved} proposed a variance-reduced SPG method that converges to a global optimal point with sample complexity of $\tilde{\mathcal{O}}(\epsilon^{-3})$. By considering a momentum term, Ding et al. \cite{ding2021global}   showed that the sample complexities of SPG for the soft-max policy and  Fisher-non-degenerate policies are in the order of $\tilde{\mathcal{O}}(\epsilon^{-4.5})$ and $\tilde{\mathcal{O}}(\epsilon^{-3})$, respectively. 
In the case of weak gradient dominant functions with $\alpha=1$ (See Assumption \ref{relaxed weak PL}), Yuan et al. \cite{yuan2021general} showed that the sample complexity of SPG is bounded by $\tilde{\mathcal{O}}(\epsilon^{-3})$.

\textbf{Theorem 4.}
\textit{For a policy $\pi_{\theta}$ satisfying Assumptions \ref{LS}, \ref{Lip Hes}, and the corresponding objective function $J(\theta)$ satisfying Assumption \ref{relaxed weak PL}, SCRN outputs the solution $\theta_{T}$ such that $J^{*}-\mathbb{E}[J(\theta_{T})]\le \epsilon+\epsilon'$ and the sample complexity (the number of observed state action pairs) is: $T\times m\times \mathsf{H}=\tilde{\mathcal{O}}(\epsilon^{-2.5})$ for $\epsilon'=0$ and $T\times m\times \mathsf{H}=\tilde{\mathcal{O}}(\epsilon^{-0.5}\epsilon'^{-2})$ for $\epsilon'>0$.}
\subsubsection{Proof of Theorem \ref{th_RL_SCRN_PL_alpha=1}}\label{Proof_of_th_RL_SCRN_PL_alpha=1}

In order to obtain Lipschitz Hessian for $J(\theta)$, we use the following lemma in \citep{zhang2020global}.
\begin{lemma}\label{lemma12}
Under Assumptions \ref{LS}, and \ref{Lip Hes}, we have
\begin{align}
    \|\nabla^{2}J(\theta)-\nabla^{2}J(\theta')\|\le \tilde{L}\|\theta-\theta'\|_{2},
\end{align}
where 
\[
\tilde{L}=\frac{R_{max}G_{1}G_{2}}{(1-\gamma)^{2}}+\frac{R_{max}G_{1}^{3}(1+\gamma)}{(1-\gamma)^{3}}+\frac{R_{max}G_{1}}{1-\gamma}\max\left\{G_{2},\frac{\gamma G_{1}^{2}}{1-\gamma},\frac{\bar{L}_{2}}{G_{1}},\frac{G_{2}\gamma}{1-\gamma},\frac{G_{1}(1+\gamma)+G_{2}\gamma(1-\gamma)}{1-\gamma^{2}}\right\}.
\]
\end{lemma}
It is good to mention that $\tilde{L}=\mathcal{O}(1/(1-\gamma)^{3})$. Using 
\eqref{eq:3} for $\alpha=1$ in Lemma \ref{lemma:two_ineq} for $F(\theta):= -J(\theta)$, we get
\begin{align}
&J^{*}-J(\theta_{t+1})-\epsilon'\leq \tau_{J}\|\nabla J(\theta_{t+1})\|\nonumber\\
&\le C(J(\theta_{t+1})-J(\theta_{t}))^{2/3}+C_g\|\nabla J(\theta_{t})-\hat{\nabla}_{m} J(\theta_{t})\|+C_{H}\|\nabla^{2}J(\theta_{t})-\hat{\nabla}_{m}^{2}J(\theta_{t})\|^{2},
\end{align}
where 
\begin{align}
     \begin{split}
      C&=3^{1/3}\tau_F \Big(\frac{M+\tilde{L}+1}{2}\Big)\Big(\frac{12}{3M-2\tilde{L}-8}\Big)^{2/3}
       \\
     C_g&=2^{2/3}\times3^{\frac{-2}{3}}\tau_J \Big(\frac{M+\tilde{L}+1}{2}\Big)\Big(\frac{12}{3M-2\tilde{L}-8}\Big)^{2/3}+\tau_{J}\\
     C_{H}&=2^{-2/3}\times3^{\frac{-2}{3}}\tau_J \Big(\frac{M+\tilde{L}+1}{2}\Big)\Big(\frac{12}{3M-2\tilde{L}-8}\Big)^{2/3}+2^{-1}\tau_{J}.
    \end{split}
\end{align}

We know that $\hat{\nabla}_{m} J(\theta_{t})$ and $\hat{\nabla}^{2}_{m}J(\theta_{t})$ are in fact unbiased estimators of $\nabla J_{\mathsf{H}}(\theta_{t})$ and $\nabla^{2}J_{\mathsf{H}}(\theta_{t})$. Thus,
\begin{align*}
    &J^{*}-J(\theta_{t+1})\le C(J(\theta_{t+1})-J(\theta_{t}))^{2/3}+C_{g}\|\nabla J_{\mathsf{H}}(\theta_{t})-\hat{\nabla}_{m} J(\theta_{t})\|+2C_{H}\|\nabla^{2}J_{\mathsf{H}}(\theta_{t})-\hat{\nabla}^{2}_{m}J(\theta_{t})\|^{2}+\epsilon'\\
    &+C_{g}\|\nabla J(\theta_{t})-\nabla J_{\mathsf{H}}(\theta_{t})\|+2C_{H}\|\nabla^{2} J(\theta_{t})-\nabla^{2} J_{\mathsf{H}}(\theta_{t})\|^{2}\\
    &\le C(J(\theta_{t+1})-J(\theta_{t}))^{2/3}+C_{g}\|\nabla J_{\mathsf{H}}(\theta_{t})-\hat{\nabla}_{m} J(\theta_{t})\|+2C_{H}\|\nabla^{2}J_{\mathsf{H}}(\theta_{t})-\hat{\nabla}^{2}_{m}J(\theta_{t})\|^{2}+\epsilon'\\
    &+C_{g}D_{g}\gamma^{H}+2C_{H}D_{H}\gamma^{2\mathsf{H}}.
\end{align*}
where $C_g$ and $C_{H}$ is defined in Lemma \ref{LS_yields_truncation}.
We take a expectation from both sides given $\theta_{t}$ and use Jensen's inequality as follows:
\begin{align}\label{recursion_ineq_RL}
    &J^{*}-\mathbb{E}J(\theta_{t+1})\le C(\mathbb{E}J(\theta_{t+1})-\mathbb{E}J(\theta_{t}))^{2/3}+C_{g}\mathbb{E}[\|\nabla J_{\mathsf{H}}(\theta_{t})-\hat{\nabla}_{m} J(\theta_{t})\|]\nonumber\\&+2C_{H}\mathbb{E}\left[\|\nabla^{2}J_{\mathsf{H}}(\theta_{t})-\hat{\nabla}^{2}_{m}J(\theta_{t})\|^{2}\right]
    +C_{g}D_{g}\gamma^{\mathsf{H}}+2C_{H}D_{H}\gamma^{2\mathsf{H}}+\epsilon'.
\end{align}
We use the following lemma to bound the error terms of the gradient estimator and the Hessian estimator.
\begin{lemma}\label{Lemma_bound_var_grad_Hss_RL} Under Assumption \ref{LS}, we have
\begin{equation}
    \begin{split}
        &\mathbb{E}\|\hat{\nabla}_{m}J(\theta)-\nabla J_{\mathsf{H}}(\theta)\|^{2}\leq \frac{\mathsf{H}G_{1}^2R_{\max}^2}{m(1-\gamma)^2},\\
        &\mathbb{E}\|\nabla^{2}J_{\mathsf{H}}(\theta)-\hat{\nabla}^{2}_{m}J(\theta)\|^{2}\leq \frac{4e\cdot\max\{2,\log d\} R_{\max}^2(\mathsf{H}^2G_{1}^4+G_{2}^2)}{m(1-\gamma)^4}.
    \end{split}
\end{equation}
\end{lemma}
The proof of Lemma \ref{Lemma_bound_var_grad_Hss_RL} is given in Appendix \ref{proof_Lemma_bound_var_grad_Hss_RL}.

Now we have from Jensen's inequality 
\[
\mathbb{E}\|\nabla J_{\mathsf{H}}(\theta_{t})-\hat{\nabla}_{m} J(\theta_{t})\|\le \left(\mathbb{E}\|\nabla J_{\mathsf{H}}(\theta_{t})-\hat{\nabla}_{m} J(\theta_{t})\|^{2}\right)^{1/2}\le \frac{ \mathsf{H} GR_{\max}}{(1-\gamma)\sqrt{m}}.
\]
Then we get 
\begin{align}\label{eq_recursion_RL}
      &J^{*}-\mathbb{E}J(\theta_{t+1})\le C(\mathbb{E}J(\theta_{t+1})-\mathbb{E}J(\theta_{t}))^{2/3}+C_{g}\frac{ \mathsf{H}GR_{\max}}{(1-\gamma)\sqrt{m}}+2C_{H}\frac{4e\cdot\max\{2,\log d\} R_{\max}^{2}(F^{2}+G^{4})}{(1-\gamma)^{4}m}\nonumber\\
    &+\epsilon'+C_{g}D_{g}\gamma^{\mathsf{H}}+2C_{H}D_{H}\gamma^{2\mathsf{H}}.
\end{align}
Let define a stationary value for $J^{*}-\mathbb{E}J(\theta_{t})$ as $P(m,\epsilon')$ which is as follows:
\begin{align}
    P(m,\epsilon'):=C_{g}\frac{ GR_{\max}}{(1-\gamma)^{3/2}\sqrt{m}}+2C_{H}\frac{4e\cdot\max\{2,\log d\} R_{\max}^{2}(F^{2}+G^{4})}{(1-\gamma)^{4}m}+\epsilon'+C_{g}D_{g}\gamma^{\mathsf{H}}+2C_{H}D_{H}\gamma^{2\mathsf{H}}.
\end{align}
We define: 
\[
\delta_{t}:=\frac{J^{*}-\mathbb{E}J(\theta_{t})-P(m,\epsilon')}{C^{3}}.
\]
Then we get following recursion from \eqref{eq_recursion_RL}:
\[
\delta_{t+1}\le (\delta_{t}-\delta_{t+1})^{2/3}.
\]
We know from  \eqref{eq_recursion_exp_alpha<3/2} in the proof of Theorem \ref{th1_expectation} that
\begin{align}
    \delta_{T}\le\frac{4}{(DT)^{2}}
\end{align}
where $D=\min\{1/2, 2(2^{1/3}-1)\delta^{-1/2}_{0}\}$. Hence,
\[
J^{*}-\mathbb{E}J(\theta_{T})-P(m,\epsilon)\le \frac{4C^{3}}{(DT)^{2}}.
\]
Therefore, $J^{*}-\mathbb{E}J(\theta_{t})$ converges to $P(m,\epsilon)$ with the convergence rate of $\mathcal{O}(1/T^{2})$. Let 
\begin{align*}
    &m\ge \max\left\{\frac{C^{2}_{g}\mathsf{H}^{2}G_{1}^{2}R_{\max}^{2}}{(1-\gamma)^{2}\epsilon'^{2}},\frac{C_{H}4e\cdot\max\{2,\log d\} R^{2}_{\max}(G_{2}^{2}+G_{1}^{4})}{(1-\gamma)^{4}\epsilon'}\right\},\\
    &\mathsf{H}\ge \max\left\{\frac{\log\left(\frac{C_{g}D_{g}}{\epsilon'}\right)}{\log(1/\gamma)},\frac{\log\left(\frac{2C_{H}D_{H}}{\epsilon'}\right)}{2\log(1/\gamma)}\right\}.
\end{align*}
Then, $\theta_{T}$ satisfies:
\begin{align}\label{2001}
    J^{*}-\mathbb{E}J(\theta_{T})\le 5\epsilon'+\epsilon
\end{align}
where $T\ge \frac{2C^{3/2}D}{\sqrt{\epsilon}}$.
Finally, the number of observed state-action pairs are required to get \eqref{2001} is as follows:
\begin{align}
    T\times \mathsf{H}\times m \ge \mathcal{O}\left(\frac{C^{3/2}DC^{2}_{g}G_{1}^{2}R_{\max}^{2}\log^{2}\left(\frac{C_{g}D_{g}}{\epsilon'}\right)}{(1-\gamma)^{2}\log^{2}(1/\gamma)\epsilon'^{2}\epsilon^{0.5}}\right)\overset{(a)}{=}\mathcal{O}\left(\frac{G^{2+\frac{7}{6}}_{1}G_{2}^{\frac{7}{6}}R_{\max}^{2+\frac{7}{6}}}{(1-\gamma)^{4+\frac{7}{3}}\epsilon'^{2}\epsilon^{0.5}}\right)
\end{align}
where (a) comes from $C=\mathcal{O}(M^{1/3})$ and $C_{g}=\mathcal{O}(M^{1/3})$ and the fact that $M=\mathcal{O}(\tilde{L})$ and from Lemma \ref{lemma12}, $\tilde{L}=\frac{R_{\max}G_{1}G_{2}}{(1-\gamma)^{2}}$.


\subsubsection{Proof of Lemma \ref{Lemma_bound_var_grad_Hss_RL}}\label{proof_Lemma_bound_var_grad_Hss_RL}
In \citep[Lemma 4.2]{yuan2021general}, it has been shown that 
\begin{equation}
    \begin{split}
        \mathbb{E}\|\hat{\nabla}_{m}J(\theta)-\nabla J_{\mathsf{H}}(\theta)\|^{2}\leq \frac{\mathsf{H}G_{1}^2R_{\max}^2}{m(1-\gamma)^2}.
    \end{split}
\end{equation}
We know from \citep[Lemma 4.1]{shen2019hessian}\label{upper_bound_variance_RL} that
\begin{align}
    \|\hat{\nabla}^{2}(\theta,\tau)\|^{2}\leq \frac{R_{\max}^2(\mathsf{H}^2G_{1}^4+G_{2}^2)}{(1-\gamma)^4},
\end{align}
where $\hat{\nabla}^{2}(\theta,\tau):=\nabla \Phi(\theta;\tau)\nabla \log p(\tau|\pi_{\theta})^T+ \nabla^2\Phi(\theta;\tau)$. We have an upper bound on the variance of Hessian of value function for each trajectory as follows:
\[
\mathbb{E}[\|\nabla^{2}J_{\mathsf{H}}(\theta)-\hat{\nabla}^{2}(\theta,\tau)\|^{2}]\le \mathbb{E}[\|\hat{\nabla}^{2}(\theta,\tau)\|^{2}]\le \frac{R_{\max}^2(\mathsf{H}^2G_{1}^4+G_{2}^2)}{(1-\gamma)^4}.
\]
Denote $\sigma^{2}_{2,1}:=\frac{R_{\max}^2(\mathsf{H}^2G_{1}^4+G_{2}^2)}{(1-\gamma)^4}$. Then from Equation \eqref{10030}, we can get
\begin{align}
    \mathbb{E}[\|\nabla^{2}J_{\mathsf{H}}(\theta)-\hat{\nabla}^{2}_{m}J(\theta)\|^{2}]\le\frac{4e\cdot\max\{2,\log d\} R_{\max}^2(\mathsf{H}^2G_{1}^4+G_{2}^2)}{m(1-\gamma)^4}
\end{align}

\subsubsection{Soft-max policies satisfy Lipschitz Hessian}\label{soft-max_lip_hessian}
Recall that soft-max tabular policy as follows:
\[
\pi_{\theta}(a|s)=\frac{\exp(\theta_{s,a})}{\sum_{a\in\mathcal{A}}\exp(\theta_{s,a})}.
\]
For any $(s,a,a')\in\mathcal{S}\times \mathcal{A}\times \mathcal{A}$ with $a'\neq a$, we get the following partial derivatives for the soft-max tabular policy
\begin{align}
    \frac{\partial \pi_{\theta}(a|s)}{\partial\theta_{s,a}}=\pi_{\theta}(a|s)(1-\pi_{\theta}(a|s)),\quad \frac{\partial \pi_{\theta}(a|s)}{\partial\theta_{s,a'}}=-\pi_{\theta}(a|s)\pi_{\theta}(a'|s).
\end{align}
Note that for $s'\in\mathcal{S}$ with $s'\neq s$, we have $\frac{\partial\pi_{\theta}(a|s)}{\partial\theta_{s',a}}=0$. We denote $\theta_{s}=[\theta_{s,a}]_{a\in\mathcal{A}}\in\mathbb{R}^{|\mathcal{A}|}$. We obtain the gradient and Hessian of $\log\pi_{\theta}(a|s)$ w.r.t. $\theta_{s}$ as follows:
\begin{align}\label{grad_hessian_soft_max}
    \frac{\partial\log\pi_{\theta}(a|s)}{\partial\theta_{s}}=\boldsymbol{1}_{a}-\pi_{\theta}(\cdot|s),\quad\frac{\partial^{2}\log\pi_{\theta}(a|s)}{\partial\theta_{s}^{2}}=\pi_{\theta}(\cdot|s)\pi_{\theta}(\cdot|s)^{T}-\text{\rm{Diag}}(\pi_{\theta}(\cdot|s))
\end{align}
where $\boldsymbol{1}_{a}\in\mathbb{R}^{|\mathcal{A}|}$ is a vector with zero entries except one non-zero entry $1$ corresponding to the action $a$ and $\pi_{\theta}(\cdot|s)=[\pi_{\theta}(a|s)]_{a\in\mathcal{A}}\in\mathbb{R}^{|\mathcal{A}|}$ is a vector consists of all policies with the same state $s$. Now we have
\begin{align}
    &\|\nabla^{2}\log\pi_{\theta'}(a|s)-\nabla^{2}\log\pi_{\theta}(a|s)\|=\left\|\frac{\partial^{2}}{\partial\theta^{2}_{s}}\log\pi_{\theta'}(a|s)-\frac{\partial^{2}}{\partial\theta^{2}_{s}}\log\pi_{\theta}(a|s)\right\|\nonumber\\
    &\le \underbrace{\left\|\pi_{\theta'}(\cdot|s)\pi_{\theta'}(\cdot|s)^{T}-\pi_{\theta}(\cdot|s)\pi_{\theta}(\cdot|s)^{T} \right\|}_{(1)}+\underbrace{\left\|\text{\rm{Diag}}(\pi_{\theta'}(\cdot|s))-\text{\rm{Diag}}(\pi_{\theta}(\cdot|s))\right\|}_{(2)}.
\end{align}
The first term (1) can be bounded as follows:
\begin{align}
    &(1)=\left\|\pi_{\theta}(\cdot|s)\pi_{\theta}(\cdot|s)^{T}-\pi_{\theta'}(\cdot|s)\pi_{\theta}(\cdot|s)^{T}+\pi_{\theta'}(\cdot|s)\pi_{\theta}(\cdot|s)^{T}-\pi_{\theta'}(\cdot|s)\pi_{\theta'}(\cdot|s)^{T} \right\|\nonumber\\
    &\le (\|\pi_{\theta}(\cdot|s)\|_{2}+\|\pi_{\theta'}(\cdot|s)\|_{2})\|\pi_{\theta}(\cdot|s)-\pi_{\theta'}(\cdot|s)\|_{2}\nonumber\\
    &\overset{(a)}{\le} 2\|\pi_{\theta}(\cdot|s)-\pi_{\theta'}(\cdot|s)\|_{2}\overset{(b)}{\le}4\|\theta-\theta'\|_{2}.
\end{align}
where (a) comes from the fact that as $\sum_a \pi(a|s)=1$ then the $\|\pi_{\theta}(\cdot|s)\|_{2}\le 1$ is less than one. Inequality (b) is derived by
$\max_{\theta}\|\nabla_{\theta}\pi_{\theta}(a|s)\|\le 2$ which yields $\pi_{\theta}(a,s)$ is 2-Lipschitz.

The second term (2) is bounded by $2\|\theta-\theta'\|_{2}$. Hence, $\log\pi_{\theta}(a|s)$ for soft-max policy is 6-Lipschitz Hessian.

\textbf{Lemma 2.}
\textit{Under Assumption \ref{LS}, we have $\|\nabla J(\theta)-\nabla J_{\mathsf{H}}(\theta)\|\le D_{g}\gamma^{\mathsf{H}}$ and $\|\nabla^{2} J(\theta)-\nabla^{2} J_{\mathsf{H}}(\theta)\|\le D_{H}\gamma^{\mathsf{H}},$ where $D_{g}=\frac{G_1R_{\max}}{1-\gamma}\sqrt{\frac{1}{1-\gamma}+\mathsf{H}}$ and $D_{H}=\frac{R_{\max}(G_2+G_{1}^{2})}{1-\gamma}\left(\mathsf{H}+\frac{1}{1-\gamma}\right)$.}
\subsubsection{Proof of Lemma \ref{LS_yields_truncation}}\label{truncation_grad_hessian_appndix}
Under Assumption \ref{LS}, it has been shown that $D_{g}=\frac{G_{1}R_{\max}}{1-\gamma}\sqrt{\frac{1}{1-\gamma}+\mathsf{H}}$ \citep[Lemma 4.5]{yuan2021general}. We will show that under Assumption \ref{LS}, $D_{H}=\frac{R_{\max}(G_{2}+G_{1}^{2})}{1-\gamma}\left(\mathsf{H}+\frac{1}{1-\gamma}\right)$.

It can be shown that \citep[Proof of Lemma 4.4]{yuan2021general} 
\begin{align*}
    &\nabla^{2}J(\theta)=\mathbb{E}_{\tau}\left[\sum_{t'=0}^{\infty}\nabla_{\theta}\log\pi_{\theta}(a_{t'}|s_{t'})\left(\sum_{t=0}^{\infty}\gamma^{t}R(s_{t},a_{t})\left(\sum_{k=0}^{t}\nabla_{\theta}\log\pi_{\theta}(a_{k}|s_{k})\right)\right)^{T}\right]\\
    &+\mathbb{E}_{\tau}\left[\sum_{t=0}^{\infty}\gamma^{t}R(s_{t},a_{t})\left(\sum_{k=0}^{t}\nabla^{2}_{\theta}\log\pi_{\theta}(a_{k}|s_{k})\right)\right]
\end{align*}
and then
\begin{align*}
    &\nabla^{2}J_{\mathsf{H}}(\theta)=\mathbb{E}_{\tau}\left[\sum_{t'=0}^{\mathsf{H}-1}\nabla_{\theta}\log\pi_{\theta}(a_{t'}|s_{t'})\left(\sum_{t=0}^{\mathsf{H}-1}\gamma^{t}R(s_{t},a_{t})\left(\sum_{k=0}^{t}\nabla_{\theta}\log\pi_{\theta}(a_{k}|s_{k})\right)\right)^{T}\right]\\
    &+\mathbb{E}_{\tau}\left[\sum_{t=0}^{\mathsf{H}-1}\gamma^{t}R(s_{t},a_{t})\left(\sum_{k=0}^{t}\nabla^{2}_{\theta}\log\pi_{\theta}(a_{k}|s_{k})\right)\right].
\end{align*}
Therefore,
\begin{align*}
    &\nabla^{2}J(\theta)-\nabla^{2}J_{\mathsf{H}}(\theta)=\underbrace{\mathbb{E}_{\tau}\left[\sum_{t=\mathsf{H}}^{\infty}\gamma^{t}R(s_{t},a_{t})\left(\sum_{k=0}^{t}\nabla^{2}_{\theta}\log\pi_{\theta}(a_{k}|s_{k})\right)\right]}_{(1)}\nonumber\\
    &+\underbrace{\mathbb{E}_{\tau}\left[\sum_{t'=\mathsf{H}}^{\infty}\nabla_{\theta}\log\pi_{\theta}(a_{t'}|s_{t'})\left(\sum_{t=0}^{\infty}\gamma^{t}R(s_{t},a_{t})\left(\sum_{k=0}^{t}\nabla_{\theta}\log\pi_{\theta}(a_{k}|s_{k})\right)\right)^{T}\right]}_{(2)}\nonumber\\
    &+\underbrace{\mathbb{E}_{\tau}\left[\sum_{t'=0}^{\mathsf{H}-1}\nabla_{\theta}\log\pi_{\theta}(a_{t'}|s_{t'})\left(\sum_{t=\mathsf{H}}^{\infty}\gamma^{t}R(s_{t},a_{t})\left(\sum_{k=0}^{t}\nabla_{\theta}\log\pi_{\theta}(a_{k}|s_{k})\right)\right)^{T}\right]}_{(3)}\nonumber\\
\end{align*}
First we bound $\|(1)\|_{2}$ as follows:
\begin{align*}
    &\|(1)\|_{2}\le \mathbb{E}_{\tau}\left[\sum_{t=\mathsf{H}}^{\infty}\gamma^{t}|R(s_{t},a_{t})|\left(\sum_{k=0}^{t}\|\nabla_{\theta}\log\pi_{\theta}(a_{k|s_{k}})\|\right)\right]\\
    &\le R_{\max}\left[\sum_{t=\mathsf{H}}^{\infty}\gamma^{t}\left(\sum_{k=0}^{t}\mathbb{E}_{\tau}\|\nabla_{\theta}\log\pi_{\theta}(a_{k|s_{k}})\|\right)\right]\\
    &\le G_{2}R_{\max}\sum_{t=\mathsf{H}}^{\infty}\gamma^{t}(1+t)=\gamma^{\mathsf{H}}G_{2}R_{\max}\sum_{t=0}^{\infty}\gamma^{t}(1+t+\mathsf{H})=\gamma^{H}G_{2}R_{\max}\left(\frac{1}{(1-\gamma)^{2}}+\frac{\mathsf{H}}{1-\gamma}\right).
\end{align*}
To bound the second term (2), it is good to note that for $0\le t<t'$
\[
\mathbb{E}_{\tau}\left[\nabla_{\theta}\log\pi_{\theta}(a_{t'}|s_{t'})\left(\gamma^{t}R(s_{t},a_{t})\left(\sum_{k=0}^{t}\nabla_{\theta}\log\pi_{\theta}(a_{k}|s_{k})\right)\right)^{T}\right]=0.
\]
Because
\begin{align}\label{condition_t<t'}
&\mathbb{E}_{\tau}\left[\nabla_{\theta}\log\pi_{\theta}(a_{t'}|s_{t'})\left(\gamma^{t}R(s_{t},a_{t})\left(\sum_{k=0}^{t}\nabla_{\theta}\log\pi_{\theta}(a_{k}|s_{k})\right)\right)^{T}\right]\nonumber\\
&=\mathbb{E}_{s_{0:t'},a_{0:(t'-1)}}\left[\underbrace{\mathbb{E}_{a_{t'}}\left[\nabla_{\theta'}\log\pi_{\theta}(a_{t'}|s_{t'})|s_{t'}\right]}_{=0}\cdot\left(\gamma^{t}R(s_{t},a_{t})\left(\sum_{k=0}^{t}\nabla_{\theta}\log\pi_{\theta}(a_{k}|s_{k})\right)\right)^{T}\right]=0
\end{align}
where
\[
\mathbb{E}_{a_{t'}}\left[\nabla_{\theta'}\log\pi_{\theta}(a_{t'}|s_{t'})|s_{t'}\right]=\int \nabla_{\theta}\pi_{\theta}(a_{t'}|s_{t'})da_{t'}=\nabla_{\theta}\underbrace{\int \pi_{\theta}(a_{t'}|s_{t'})da_{t'}}_{=1}=0.
\]
Hence, we get
\begin{align*}
    \|(2)\|&=\left\|\mathbb{E}_{\tau}\left[\sum_{t'=\mathsf{H}}^{t}\nabla_{\theta}\log\pi_{\theta}(a_{t'}|s_{t'})\left(\sum_{t=\mathsf{H}}^{\infty}\gamma^{t}R(s_{t},a_{t})\left(\sum_{k=0}^{t}\nabla_{\theta}\log\pi_{\theta}(a_{k}|s_{k})\right)\right)^{T}\right]\right\|\\
    &\overset{(a)}{=}\left\|\mathbb{E}_{\tau}\left[\sum_{t=\mathsf{H}}^{\infty}\gamma^{t}R(s_{t},a_{t})\left(\sum_{t'=\mathsf{H}}^{t}\nabla_{\theta}\log\pi_{\theta}(a_{t'}|s_{t'})\right)\left(\sum_{k=\mathsf{H}}^{t}\nabla_{\theta}\log\pi_{\theta}(a_{k}|s_{k})\right)^{T}\right]\right\|\\
    &\le R_{\max}\sum_{t=\mathsf{H}}^{\infty}\gamma^{t}\mathbb{E}_{\tau}\left[\left\|\sum_{k=\mathsf{H}}^{t}\nabla_{\theta}\log\pi_{\theta}(a_{k}|s_{k})\right\|^{2}\right]\\
    &\overset{(b)}{=}R_{\max}\sum_{t=\mathsf{H}}^{\infty}\gamma^{t}\sum_{k=\mathsf{H}}^{t}\mathbb{E}_{\tau}\left[\left\|\nabla_{\theta}\log\pi_{\theta}(a_{k}|s_{k})\right\|^{2}\right]\\
    &\le R_{\max}G_{1}^{2}\sum_{t=\mathsf{H}}^{\infty}\gamma^{t}(t-\mathsf{H}+1)=R_{\max}G_{1}^{2}\gamma^{\mathsf{H}}\sum_{t=0}^{\infty}\gamma^{t}(t+1)\le \frac{G_{1}^{2}R_{\max}\gamma^{\mathsf{H}}}{(1-\gamma)^{2}}
\end{align*}
where (a) comes from the similar argument to \eqref{condition_t<t'} and (b) comes from the fact that for any $k\neq k'$
\begin{align}\label{grad_log_pi_k_neq_k'}
\mathbb{E}_{\tau}\left[\nabla_{\theta}\log\pi_{\theta}(a_{k}|s_{k})^{T}\nabla_{\theta}\log\pi_{\theta}(a_{k'}|s_{k'})\right]=0.
\end{align}
To bound the third term (3), first, with the similar argument to \eqref{condition_t<t'}, we can rewrite (3) as follows:
\begin{align*}
    \|(3)\|&=\left\|\mathbb{E}_{\tau}\left[\sum_{t'=0}^{\mathsf{H}-1}\nabla_{\theta}\log\pi_{\theta}(a_{t'}|s_{t'})\left(\sum_{t=\mathsf{H}}^{\infty}\gamma^{t}R(s_{t},a_{t})\left(\sum_{k=0}^{t}\nabla_{\theta}\log\pi_{\theta}(a_{k}|s_{k})\right)\right)^{T}\right]\right\|\\
    &=\left\|\mathbb{E}_{\tau}\left[\sum_{t=\mathsf{H}}^{\infty}\gamma^{t}R(s_{t},a_{t})\cdot\left(\sum_{t'=0}^{\mathsf{H}-1}\nabla_{\theta}\log\pi_{\theta}(a_{t'}|s_{t'})\right)\cdot\left(\sum_{k=0}^{t}\nabla_{\theta}\log\pi_{\theta}(a_{k}|s_{k})\right)^{T}\right]\right\|\\
    &\overset{(a)}{=}\left\|\mathbb{E}_{\tau}\left[\sum_{t=\mathsf{H}}^{\infty}\gamma^{t}R(s_{t},a_{t})\cdot\left(\sum_{t'=0}^{\mathsf{H}-1}\nabla_{\theta}\log\pi_{\theta}(a_{t'}|s_{t'})\right)\cdot\left(\sum_{k=0}^{\mathsf{H}-1}\nabla_{\theta}\log\pi_{\theta}(a_{k}|s_{k})\right)^{T}\right]\right\|\\
    &\le R_{\max}\sum_{t=\mathsf{H}}^{\infty}\gamma^{t}\cdot\mathbb{E}_{\tau}\left[\left\|\sum_{k=0}^{\mathsf{H}-1}\nabla_{\theta}\log\pi_{\theta}(a_{k}|s_{k})\right\|^{2}\right]\\
    &\le R_{\max}\sum_{t=\mathsf{H}}^{\infty}\gamma^{t}\cdot\sum_{k=0}^{\mathsf{H}-1}\mathbb{E}_{\tau}\left[\left\|\nabla_{\theta}\log\pi_{\theta}(a_{k}|s_{k})\right\|^{2}\right]\le R_{\max}G_{1}^{2}\sum_{t=\mathsf{H}}^{\infty}\gamma^{t}\mathsf{H}=\frac{\gamma^{\mathsf{H}}R_{\max}G_{1}^{2}\mathsf{H}}{1-\gamma}
\end{align*}
where (a) comes from \eqref{grad_log_pi_k_neq_k'}.

Using triangle inequality and combining all bounds for (1), (2), and (3), we have
\begin{align}
    &\|\nabla^{2}J(\theta)-\nabla^{2}J_{\mathsf{H}}(\theta)\|\le \|(1)\|+\|(2)\|+\|(3)\|\nonumber\\
    &\le \gamma^{\mathsf{H}}\left(G_{2}R_{\max}\left(\frac{1}{(1-\gamma)^{2}}+\frac{\mathsf{H}}{1-\gamma}\right)+\frac{G_{1}^{2}R_{\max}}{(1-\gamma)^{2}}+\frac{R_{\max}G_{1}^{2}\mathsf{H}}{1-\gamma}\right).
\end{align}

\subsubsection{Analysis of the variance-reduced SCRN under weak gradient dominance property with $\alpha=1$}\label{vr-scrn_RL_append}
\begin{definition}[Importance Sampling]
For any given trajectory $\tau$ and two parameters $\theta,\theta'\in \mathbb{R}^d$, we define the importance
sampling weight as 
\begin{align}
  w(\tau|\theta',\theta)=\frac{p(\tau|\pi_{\theta'})}{p(\tau|\pi_{\theta})}=\prod_{i=1}^{\mathsf{H}}\frac{\pi_{\theta'}(a_{h}|s_{h})}{\pi_{\theta}(a_{h}|s_{h})},
\end{align}
where $\mathsf{H}$ is the horizon length.
\end{definition}
\begin{assumption}\label{bounded_var_IS}
There exists a constant $W\in(0,\infty)$, such that, for any pair of policies generated by Algorithm \ref{algorithm3}, the following holds
\begin{align}
    \var(w(\tau|\theta',\theta))\le W,\quad \theta,\theta'\in\mathbb{R}^{d},\, \tau\sim p(\cdot|\pi_{\theta'}).
\end{align}
\end{assumption}
In following, we present a version of Algorithm \ref{alg:cap1} adapted for RL setting by utilizing importance sampling. In what follows, $\hat{\nabla} J(\theta';\tau)$ and $\hat{\nabla}^{2} J(\theta';\tau)$ are gradient and Hessian estimators of the value function, respectively which are sampled from distribution $p(\tau|\pi_{\theta})$ for some $\theta$.
\begin{algorithm}
\caption{IS-VR stochastic cubic regularized Newton method}\label{alg:cap3}
\textbf{Input:} Maximum number of iterations $T$, batch sizes $\{n^{t}_{g}\}_{t=1}^{T}$, $\{n^{t}_{H}\}_{t=1}^{T}$, the period length $S$, initial point $\theta_{0}$, and cubic penalty parameter $M$.
\begin{algorithmic}[1]
\State $t\gets 0$
\While{$t\le T$}
\State Sample trajectory set $\mathcal{J}_{t}=\{\tau^{t}_{j}\}^{n^{t}_{g}}_{j=1}$ and trajectory set $\mathcal{I}_{t}=\{\tau^{t}_{i}\}^{n^{t}_{H}}_{i=1}$ according to policy $p(\tau|\pi_{\theta_{t}})$
\State \[
\vv_{t} \gets \begin{cases}
\frac{1}{n^{t}_{g}}\sum_{j=1}^{n^{t}_{g}}\hat{\nabla} J(\theta_{t};\tau^{t}_{j}), &\text{mod}(t,S)=0\\
\frac{1}{n^{t}_{g}}\sum_{j=1}^{n^{t}_{g}}[\hat{\nabla} J(\theta_{t};\tau^{t}_{j})-w(\tau^{t}_{j}|\theta_{t-1},\theta_{t})\hat{\nabla} J(\theta_{t-1};\tau^{t}_{j})]+\vv_{t-1},&\text{else}
\end{cases}
\]
\State \[
\Uv_{t} \gets \begin{cases}
\frac{1}{n^{t}_{H}}\sum_{i=1}^{n^{t}_{H}}\hat{\nabla}^{2} J(\theta_{t};\tau^{t}_{i}),\quad &\text{mod}(t,S)=0\\
\frac{1}{n^{t}_{H}}\sum_{i=1}^{n^{t}_{H}}[\hat{\nabla}^{2} J(\theta_{t};\tau^{t}_{i})-w(\tau^{t}_{i}|\theta_{t-1},\theta_{t})\hat{\nabla}^{2} J(\theta_{t-1};\tau^{t}_{i})]+\Uv_{t-1},\quad&\text{else}
\end{cases}
\]
\State ${\bf\Delta}_{t} \gets \argmin_{{\bf\Delta}\in\mathbb{R}^{d}}\langle \vv_{t}, {\bf\Delta}\rangle+\frac{1}{2}\langle {\bf\Delta}, \Uv_{t}{\bf\Delta}\rangle+\frac{M}{6}\|{\bf\Delta}\|^{3}$
\State $\theta_{t+1}\gets \theta_{t}+{\bf \Delta}_{t}$
\State $t \gets t+1$
\EndWhile
\State\Return{$\theta_{t}$}
\end{algorithmic}
\label{algorithm3}
\end{algorithm}

Assumptions \ref{assump2} with $\alpha=1$ and Assumption \ref{assump:individual_smoothness} are satisfied in the considered RL setting. To see this, $\|\hat{\nabla}^{2}J(\theta,\tau)\|^{2}\le \frac{\mathsf{H}^{2}G_{1}^{4}R^{2}_{\max}+G_{2}^{2}R^{2}_{\max}}{(1-\gamma)^4}$ from \citep[Lemma 4.1]{shen2019hessian}. Then, $\hat{\nabla}J(\theta,\tau)$ is Lipschitz.
Moreover, similar to the arguments in Lemma \ref{lemma12}, one can show that $\|\hat{\nabla}^{2}J(\theta,\tau)-\hat{\nabla}^{2}J(\theta',\tau)\|\le \tilde{L}\|\theta-\theta'\|_{2}$. Hence, we can imply that Lipschitzness of $\hat{\nabla}J(\theta,\tau)$ and $\hat{\nabla}^{2}J(\theta,\tau)$ yield Assumption \ref{assump:individual_smoothness}. Furthermore, from proof of Lemma 3 (Section \ref{proof_Lemma_bound_var_grad_Hss_RL}), the variances of $\hat{\nabla}J(\theta,\tau)$ and $\hat{\nabla}^{2}J(\theta,\tau)$ are bounded.

Now, with 
\[
\mathsf{H}\ge \max\left\{\frac{\log\left(\frac{C_{g}D_{g}}{\epsilon}\right)}{\log(1/\gamma)},\frac{\log\left(\frac{2C_{H}D_{H}}{\epsilon}\right)}{2\log(1/\gamma)}\right\},
\]
Equation \eqref{recursion_ineq_RL} can be rewritten as follows:
\begin{align}\label{eq:1163}
     &J^{*}-\mathbb{E}J(\theta_{t+1})\le C(\mathbb{E}J(\theta_{t+1})-\mathbb{E}J(\theta_{t}))^{2/3}+C_{g}\mathbb{E}[\|\nabla J_{\mathsf{H}}(\theta_{t})-\vv_{t}\|]\nonumber\\&+2C_{H}\mathbb{E}\left[\|\nabla^{2}J_{\mathsf{H}}(\theta_{t})-\Uv_{t}\|^{2}\right]+2\epsilon.
\end{align}
\begin{lemma}[Lemma 6.1 in \cite{xu2020improved}]
For a policy $\pi_{\theta}$ satisfying Assumptions \ref{LS} and \ref{bounded_var_IS}, and for any $\theta,\theta'\in \mathbb{R}^d$, we have
\begin{align}
    \var(w(\tau|\theta,\theta'))\le C_{w}\|\theta-\theta'\|^{2},
\end{align}
where $C_{w}=\mathsf{H}(2\mathsf{H}G_{1}^{2}+G_{2})(W+1)$.
\label{lemma:14}
\end{lemma}
In order to prove Theorem \ref{th_vr_SCRN_Pl_alpha=1}, we must show that with the following number of samples at checkpoints:
\begin{equation}
    n_g^t=\Theta\left( \frac{1}{\epsilon^2}\right), n_H^t=\Theta\left( \frac{(\log d)^{2}}{\epsilon}\right),
\end{equation}
and the following number of samples at the other iterations:
\begin{equation}
    n_g^t(\theta_{t},\theta_{t-1})=\Theta \left(\frac{S \|\mathbf{\Delta}_{t-1}\|^2}{\epsilon^2}\right), n_H^t(\theta_{t},\theta_{t-1})=\Theta\left(\frac{ (\log d)^{2}\cdot S \|\mathbf{\Delta}_{t-1}\|^2}{\epsilon}\right),
\end{equation}
we have:
\begin{equation}
    \mathbb{E}[\|\nabla J_{\mathsf{H}}(\theta_{t})-\vv_t\|]\leq \epsilon, \qquad \mathbb{E}[\|\nabla^2 J_{\mathsf{H}}(\theta_{t})-\Uv_t\|^2]\leq \epsilon.
    \label{eq:VR_SCRN_RL_eps}
\end{equation}
We use the same argument in the proof of Lemma \ref{lemma:vr} and we borrow the notations in Lemma \ref{lemma:vr}. Let us define $\uv_{k}$ and $\Vv_{k}$. For the gradient estimator $\vv_{t}$, we have $\vv_{t}-\nabla J_{\mathsf{H}}(\theta_{t})=\sum_{k=\lfloor t/S\rfloor S}^{t}\uv_{k}$ such that:
\[
\uv_{k}=\begin{cases}
\frac{1}{n^{k}_{g}}\sum_{j=1}^{n^{k}_{g}}\hat{\nabla} J(\theta_{k};\tau^{k}_{j})-\nabla J_{\mathsf{H}}(\theta_{k}),\quad &k=\lfloor t/S\rfloor S,\\
\frac{1}{n^{k}_{g}}\sum_{j=1}^{n^{k}_{g}}[\hat{\nabla} J(\theta_{k};\tau^{k}_{j})-w(\tau^{k}_{j}|\theta_{k-1},\theta_{k})\hat{\nabla} J(\theta_{k-1};\tau^{k}_{j})]-\nabla J_{\mathsf{H}}(\theta_{k})+\nabla J_{\mathsf{H}}(\theta_{k-1}),\quad& k>\lfloor t/S\rfloor S,
\end{cases}
\]
and for the Hessian estimator $\Uv_{t}$, we have $\Uv_{t}-\nabla^{2} J_{\mathsf{H}}(\theta_{t})=\sum_{k=\lfloor t/S\rfloor S}^{t}\Vv_{k}$ such that
\[
\Vv_{k}=\begin{cases}
\frac{1}{n^{k}_{H}}\sum_{i=1}^{n^{k}_{H}}\hat{\nabla}^{2} J(\theta_{k};\tau^{k}_{i})-\nabla^{2} J_{\mathsf{H}}(\theta_{k}),\quad &k=\lfloor t/S\rfloor S,\\
\frac{1}{n^{k}_{H}}\sum_{i=1}^{n^{k}_{H}}[\hat{\nabla}^{2} J(\theta_{k};\tau^{k}_{i})-w(\tau^{k}_{i}|\theta_{k-1},\theta_{k})\hat{\nabla}^{2} J(\theta_{k-1};\tau^{k}_{i})]-\nabla^{2} J_{\mathsf{H}}(\theta_{k})+\nabla^{2} J_{\mathsf{H}}(\theta_{k-1}),\quad& k>\lfloor t/S\rfloor S.
\end{cases}
\]
To prove that \eqref{eq:VR_SCRN_RL_eps}, as the importance sampling terms appears in $\uv_{k}$ and $\Vv_{k}$ for $k>\lfloor t/S\rfloor S$, we must only modify the proof of $\mathbb{E}[\|\uv_{k}\|^{2}|\mathcal{G}_{k}]\le \frac{C_{1}}{n^{k}_{g}}\|\theta_{k}-\theta_{k-1}\|^{2}$ and $\mathbb{E}[\|\Vv_{k}\|^{2}|\mathcal{H}_{k}]\le \frac{C_{2}}{n^{k}_{g}}\|\theta_{k}-\theta_{k-1}\|^{2}$ for $k>\lfloor t/S\rfloor S$. We prove them as follows:
\begin{align}
        &\mathbb{E}[\|\uv_{k}\|^{2}|\mathcal{G}_{k}]=\mathbb{E}\left\|\frac{1}{n^{k}_{g}}\sum_{j=1}^{n^{k}_{g}}[\hat{\nabla} J(\theta_{k};\tau^{k}_{j})-w(\tau^{k}_{j}|\theta_{k-1},\theta_{k})\hat{\nabla} J(\theta_{k-1};\tau^{k}_{j})]-\nabla J_{\mathsf{H}}(\theta_{k})+\nabla J_{\mathsf{H}}(\theta_{k-1})\right\|^{2}\nonumber\\
    &\overset{(a)}{=}\frac{1}{n^{k}_{g}}\mathbb{E}\left\|\hat{\nabla} J(\theta_{k};\tau^{k}_{1})-w(\tau^{k}_{1}|\theta_{k-1},\theta_{k})\hat{\nabla} J(\theta_{k-1};\tau^{k}_{1})-\nabla J_{\mathsf{H}}(\theta_{k})+\nabla J_{\mathsf{H}}(\theta_{k-1})\right\|^{2}\nonumber\\
    &\overset{(b)}{\le} \frac{2}{n^{k}_{g}}\mathbb{E}\|\hat{\nabla} J(\theta_{k};\tau^{k}_{1})-w(\tau^{k}_{1}|\theta_{k-1},\theta_{k})\hat{\nabla} J(\theta_{k-1};\tau^{k}_{1})\|^{2}+\frac{2}{n^{k}_{g}}\mathbb{E}\|\nabla J_{\mathsf{H}}(\theta_{k})-\nabla J_{\mathsf{H}}(\theta_{k-1})\|^{2}\nonumber\\
    &\overset{(c)}{\le} \frac{4}{n^{k}_{g}}\mathbb{E}\|\hat{\nabla} J(\theta_{k};\tau^{k}_{1})-\hat{\nabla} J(\theta_{k-1};\tau^{k}_{1})\|^{2}+\frac{4}{n^{k}_{g}}\mathbb{E}[(w(\tau^{k}_{1}|\theta_{k-1},\theta_{k})-1)^{2}\|\hat{\nabla}J(\theta_{k-1};\tau^{k}_{1})\|^{2}]\nonumber\\
    &+\frac{2}{n^{k}_{g}}\mathbb{E}\|\nabla J_{\mathsf{H}}(\theta_{k})-\nabla J_{\mathsf{H}}(\theta_{k-1})\|^{2}\nonumber\\
    &\overset{(d)}{\le} \frac{4}{n^{k}_{g}}\mathbb{E}\|\hat{\nabla} J(\theta_{k};\tau^{k}_{1})-\hat{\nabla} J(\theta_{k-1};\tau^{k}_{1})\|^{2}+\frac{4G_{1}^{2}}{n^{k}_{g}}\mathbb{E}[(w(\tau^{k}_{1}|\theta_{k-1},\theta_{k})-1)^{2}]+\frac{2}{n^{k}_{g}}\mathbb{E}\|\nabla J_{\mathsf{H}}(\theta_{k})-\nabla J_{\mathsf{H}}(\theta_{k-1})\|^{2}\nonumber\\
    &\overset{(e)}{\le} \frac{6L'^{2}_{1}+4G^{2}_{1}C_{w}}{n^{k}_{g}}\|\theta_{k}-\theta_{k-1}\|^{2}
\end{align}
where (a) comes from $\hat{\nabla} J(\theta_{k};\tau^{k}_{j})-w(\tau^{k}_{j}|\theta_{k-1},\theta_{k})\hat{\nabla} J(\theta_{k-1};\tau^{k}_{j})-\nabla J_{\mathsf{H}}(\theta_{k})+\nabla J_{\mathsf{H}}(\theta_{k-1})$'s are i.i.d conditioned on $\mathcal{G}_k$ for $1\le i\le n^{k}_{g}$. Inequalities (b) and (c) are from $(a+b)^{2}\le 2a^{2}+2b^{2}$. (d) is due to Assumption \ref{LS}. (e) is derived by Assumption \ref{assump:individual_smoothness} and Lemma \ref{lemma:14}.

Similarly, for Hessian estimator we have

\begin{align}
    &\mathbb{E}[\|\Vv_{k}\|^{2}|\mathcal{H}_{k}]=\mathbb{E}\left\|\frac{1}{n^{k}_{H}}\sum_{i=1}^{n^{k}_{H}} [\hat{\nabla}^{2}J(\theta_{k},\tau_{i})-w(\tau^{k}_{i}|\theta_{k-1},\theta_{k})\hat{\nabla}^{2}J(\theta_{k-1},\tau_{i})]-\nabla^{2} J_{\mathsf{H}}(\theta_{k})+\nabla^{2} J_{\mathsf{H}}(\theta_{k-1})\right\|^{2}\nonumber\\
    &\overset{(a)}{\le}8e\log d\cdot\mathbb{E}\left\|\frac{1}{(n^{k}_{H})^{2}}\sum_{i=1}^{n^{k}_{H}} \left[[\hat{\nabla}^{2}J(\theta_{k},\tau_{i})-w(\tau^{k}_{i}|\theta_{k-1},\theta_{k})\hat{\nabla}^{2}J(\theta_{k-1},\tau_{i})]-\nabla^{2} J_{\mathsf{H}}(\theta_{k})+\nabla^{2} J_{\mathsf{H}}(\theta_{k-1})\right]^{2}\right\|\nonumber\\
    &\overset{(b)}{\le}\frac{8e\log d}{n^{k}_{H}}\cdot \frac{1}{n^{k}_{H}}\sum_{i=1}^{n^{k}_{H}}\mathbb{E}\left\| [\hat{\nabla}^{2}J(\theta_{k},\tau_{i})-w(\tau^{k}_{i}|\theta_{k-1},\theta_{k})\hat{\nabla}^{2}J(\theta_{k-1},\tau_{i})-\nabla^{2} J_{\mathsf{H}}(\theta_{k})+\nabla^{2} J_{\mathsf{H}}(\theta_{k-1})]^{2}\right\|\nonumber\\
    &\overset{(c)}{\le}\frac{8e\log d}{n^{k}_{H}}\mathbb{E}\left\| \hat{\nabla}^{2}J(\theta_{k},\tau_{1})-w(\tau^{k}_{1}|\theta_{k-1},\theta_{k})\hat{\nabla}^{2}J(\theta_{k-1},\tau_{1})-\nabla^{2} J_{\mathsf{H}}(\theta_{k})+\nabla^{2} J_{\mathsf{H}}(\theta_{k-1})\right\|^{2}\nonumber\\
    &\overset{(d)}{\le} \frac{16e\log d}{n^{k}_{H}}\mathbb{E}\|\hat{\nabla}^{2}J(\theta_{k},\tau_{1})-w(\tau^{k}_{1}|\theta_{k-1},\theta_{k})\hat{\nabla}^{2}J(\theta_{k-1},\tau_{1})\|^{2}+\frac{16e\log d}{n^{k}_{H}}\|\nabla^{2} J_{\mathsf{H}}(\theta_{k})-\nabla^{2} J_{\mathsf{H}}(\theta_{k-1})\|^{2}\nonumber\\
    &\overset{(e)}{\le} \frac{32e\log d}{n^{k}_{H}}\mathbb{E}\|\hat{\nabla}^{2}J(\theta_{k},\tau_{1})-\hat{\nabla}^{2}J(\theta_{k-1},\tau_{1})\|^{2}+\frac{32e\log d}{n^{k}_{H}}\mathbb{E}\left[(1-w(\tau^{k}_{1}|\theta_{k-1},\theta_{k}))^{2}\|\hat{\nabla}^{2}J(\theta_{k-1},\tau_{1})\|^{2}\right]\nonumber\\
    &+\frac{16e\log d}{n^{k}_{H}}\|\nabla^{2} J_{\mathsf{H}}(\theta_{k})-\nabla^{2} J_{\mathsf{H}}(\theta_{k-1})\|^{2}\nonumber\\
    &\overset{(f)}{\le} \frac{32e\log d}{n^{k}_{H}}\mathbb{E}\|\hat{\nabla}^{2}J(\theta_{k},\tau_{1})-\hat{\nabla}^{2}J(\theta_{k-1},\tau_{1})\|^{2}+\frac{32e\cdot G^{2}_{2}\log d}{n^{k}_{H}}\mathbb{E}\left[(1-w(\tau^{k}_{1}|\theta_{k-1},\theta_{k}))^{2}\right]\nonumber\\
    &+\frac{16e\log d}{n^{k}_{H}}\|\nabla^{2} J_{\mathsf{H}}(\theta_{k})-\nabla^{2} J_{\mathsf{H}}(\theta_{k-1})\|^{2}\nonumber\\
    &\overset{(g)}{\le} \frac{32e\log d(2L'^{2}_{2}+C_{w}G^{2}_{2})}{n^{k}_{H}}\|\theta_{k}-\theta_{k-1}\|^{2},
\end{align}
where (a) comes from Lemma \ref{lemma_norm_op_summation_of_iid_matrix} and the fact that $[\hat{\nabla}^{2}J(\theta_{k},\tau_{i})-w(\tau^{k}_{i}|\theta_{k-1},\theta_{k})\hat{\nabla}^{2}J(\theta_{k-1},\tau_{i})-\nabla^{2} J_{\mathsf{H}}(\theta_{k})+\nabla^{2} J_{\mathsf{H}}(\theta_{k-1})]$'s are i.i.d conditioned on $\mathcal{H}_{k}$ for $1\le i\le n^{k}_{H}$. (b) comes from Jensen's inequality for operator norm $\|\cdot\|$. (c) is derived by $\|AB\|\le \|A\|\|B\|$ and the fact that $[\hat{\nabla}^{2}J(\theta_{k},\tau_{i})-w(\tau^{k}_{i}|\theta_{k-1},\theta_{k})\hat{\nabla}^{2}J(\theta_{k-1},\tau_{i})-\nabla^{2} J_{\mathsf{H}}(\theta_{k})+\nabla^{2} J_{\mathsf{H}}(\theta_{k-1})]$'s are i.i.d conditioned on $\mathcal{H}_{k}$ for $1\le i\le n^{k}_{H}$. Inequalities (d) and (e) are from inequality $(a+b)^{2}\le 2a^{2}+2b^{2}$ and (f) is derived by Assumption \ref{LS}. (g) comes from Assumption \ref{assump:individual_smoothness} and Lemma \ref{lemma:14}. Similar to the proof of Lemma \ref{lemma:vr}, we can argue that for Algorithm \ref{algorithm3}, \eqref{eq:VR_SCRN_RL_eps} holds.

Now, based on what we proved above, with the same arguments in Theorem \ref{th_vr_SCRN_Pl_alpha=1}, Algorithm \ref{algorithm3} achieves $\epsilon$-global stationary point by querying ${\mathcal{O}}(\epsilon^{-2})$ stochastic gradients and ${\mathcal{O}}(\epsilon^{-1})$ stochastic Hessian on average. More precisely, the following average number of queries of stochastic gradient and Hessian are required in order to achieve $\epsilon$-global stationary point:
\[
 2\mathsf{H}\times \sum_{i=1}^{T}n_{g}^{t}= \mathcal{O}\left(\frac{\sigma_{1}^{2}\mathsf{H}\left[4^{2}(C_{g}+C_{H})C^{3}+C_{g}+C_{H}\right]^{2}}{\epsilon^{2}}\right)=\mathcal{O}\left(\frac{(G_{1}^{2+\frac{8}{3}}G_{2}^{\frac{8}{3}}R_{\max}^{2+\frac{8}{3}}}{(1-\gamma)^{4+\frac{16}{3}}\epsilon^{2}}\right).
\]
where (a) comes from $\sigma_{1}^{2}=\frac{\mathsf{H}G_{1}^2R_{\max}^2}{(1-\gamma)^2}$ Lemma \ref{Lemma_bound_var_grad_Hss_RL} and $C=\mathcal{O}(M^{1/3})$, $C_{g}=\mathcal{O}(M^{1/3})$, and $C_{H}=\mathcal{O}(M^{1/3})$ and the fact that $M=\mathcal{O}(\tilde{L})$ and from Lemma \ref{lemma12}, $\tilde{L}=\mathcal{O}\left(\frac{R_{\max}G_{1}G_{2}}{(1-\gamma)^{2}}\right)$.


\subsection{Further experimental results}
\label{app:further resutls}
\begin{figure}
    \centering
    \includegraphics[width=0.8\textwidth]{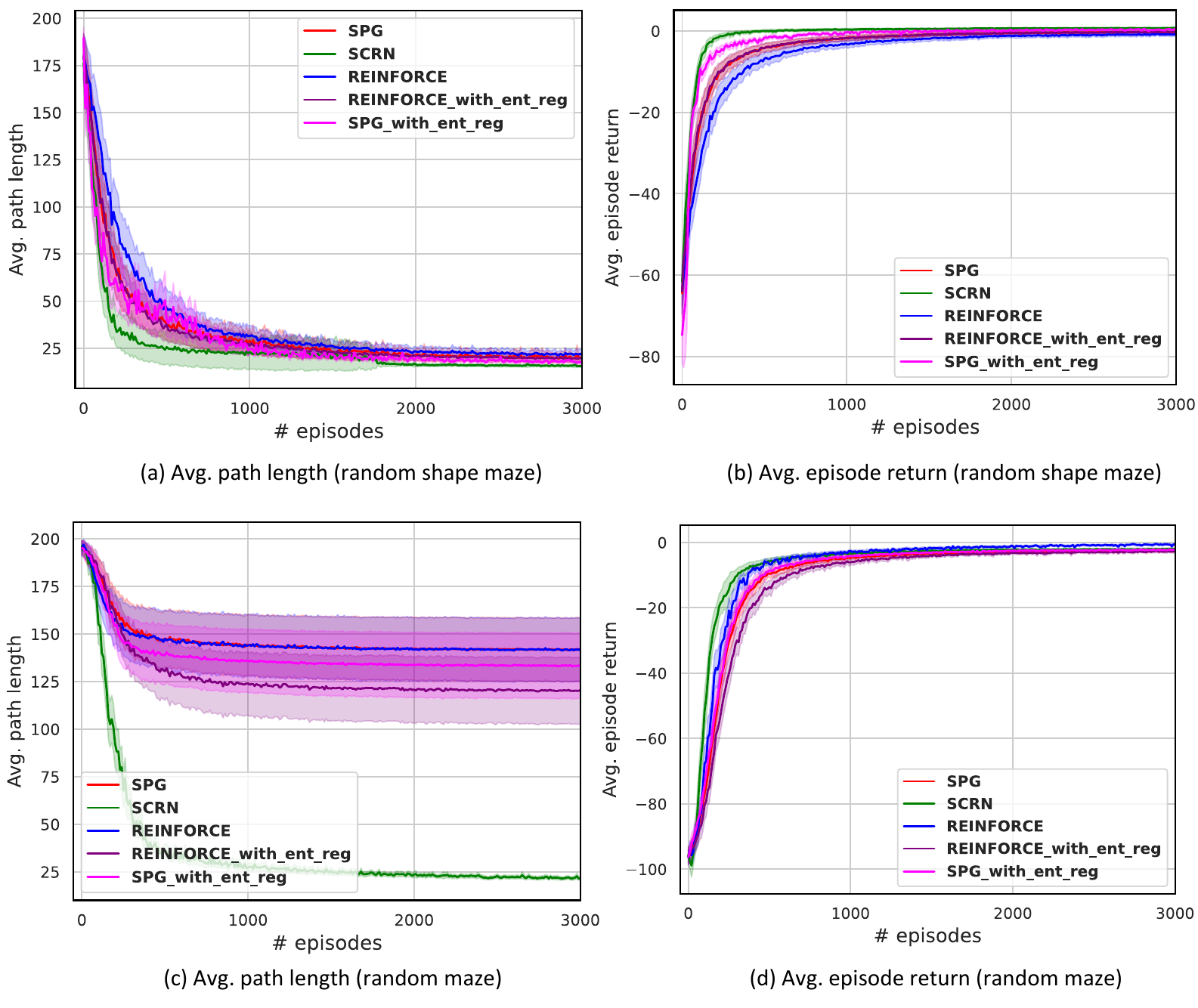}
    \caption{Comparison of SCRN with first-order methods in maze environments. In random shape maze, the percentages of successful instances for SPG, SCRN, REINFORCE, SPG with entropy regularization, and REINFORCE with entropy regularization are $86\%$, $100\%$, $95.3\%$, $92.2\%$, and $95.3\%$, respectively. In random maze, the percentages of successful instances for SPG, SCRN, REINFORCE, SPG with entropy regularization, and REINFORCE with entropy regularization are $45.3\%$, $97\%$, $40.6\%$, $54.7\%$, and $54.7\%$, respectively.}
    \label{fig:maze_app}
\end{figure}
\textbf{Discussion on results:} In Fig. \ref{fig:maze_app} (a), in the random shape maze environment, SCRN finds paths with average length of $25$ after about $600$ episodes while SPG and REINFORCE can only achieve the average path lengths of greater than $25$ after $3000$ episodes. In the random maze environment, SCRN again outperforms the other two algorithms and achieves the average path length of $25$ after about $2000$ episodes. However, for SPG and REINFROCE, the average path lengths are more than $100$ which also indicates the low percentages of successful instances for these two algorithms in this environment. Furthermore, as shown in Fig. \ref{fig:maze_app} (c,d),  in both random shape maze and random maze environments, SCRN has better performance than SPG and REINFORCE in terms of average episode return. We also provide the results for the random shape maze and random maze environments where we add entropy regularization term to the objective functions of SPG and REINFORCE. As can be seen, the results with entropy regularization term are slightly better both in the average path length and average episode return.

\begin{figure*}[t]
\vspace{-1.5mm}
    \centering
    \includegraphics[width=1\textwidth]{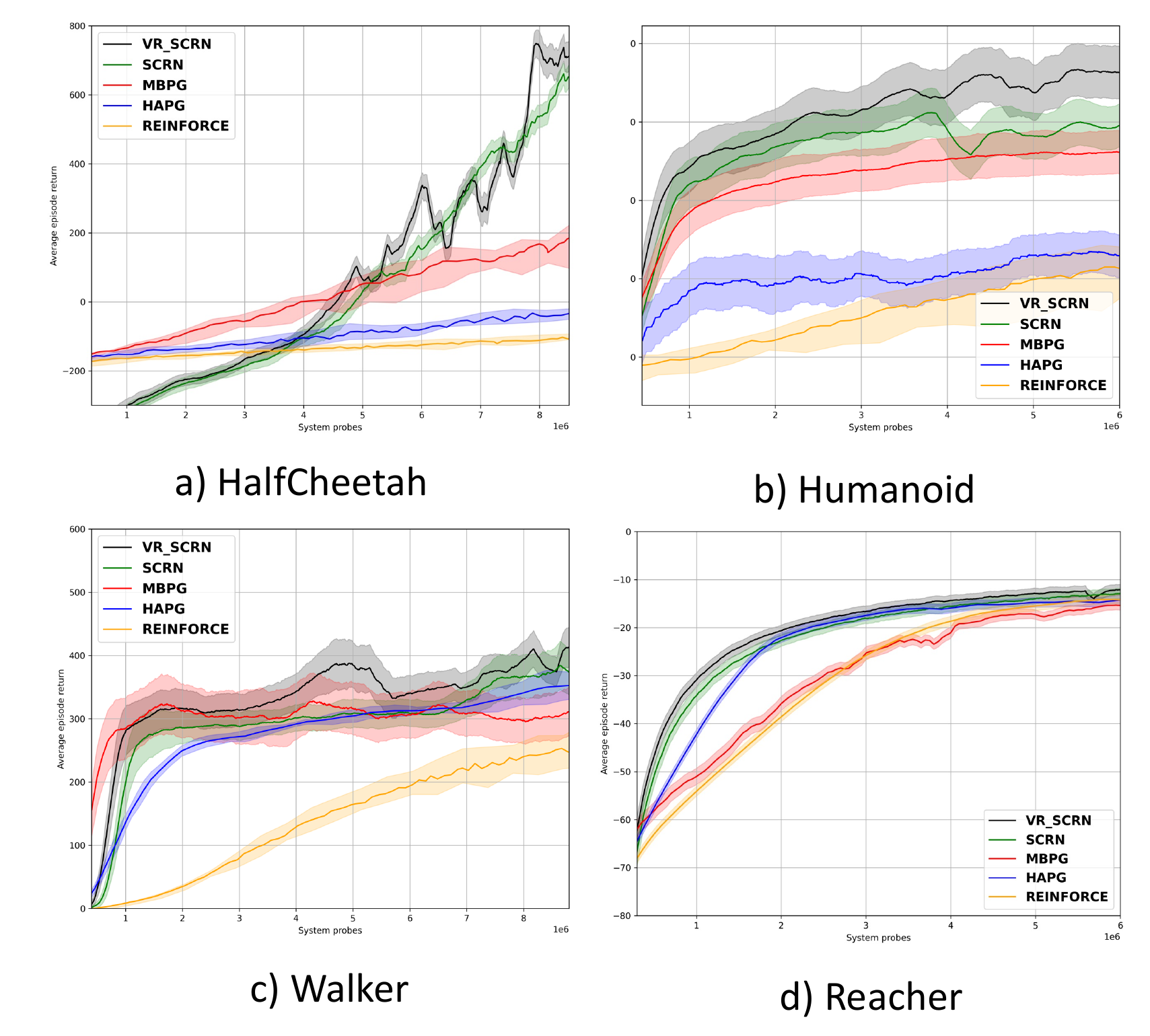}
    \caption{Comparison of SCRN and its variance-reduced version with REINFORCE and variance-reduced SPG methods in MuJoCo environments.}
    \label{fig7}
    \vspace{-2mm}
\end{figure*}

In Fig. \ref{fig7}, for control tasks in MuJoCo simulator, we compared SCRN and its variance-reduced version with first-order methods such as REINFORCE, and two state-of-the-art representatives of variance-reduced PG methods, HAPG \cite{shen2019hessian} and MBPG \cite{huang2020momentum}, both with guaranteed convergence to $\epsilon$-FOSP in general non-convex settings. In HalfCheetah environment, SCRN and variance-reduced SCRN outperform the other methods by achieving a score of around $600$. In the Humanoid environment, variance-reduced SCRN achieves the best performance among other algorithms. In Walker environment, the performances of SCRN, its variance-reduced version, and HAPG perform better than MBPG and REINFORCE. In Reacher environment, the variance-reduced SCRN has the best performance among the considered algorithms.
\begin{figure*}[t]
\vspace{-1.5mm}
    \centering
    \includegraphics[width=1\textwidth]{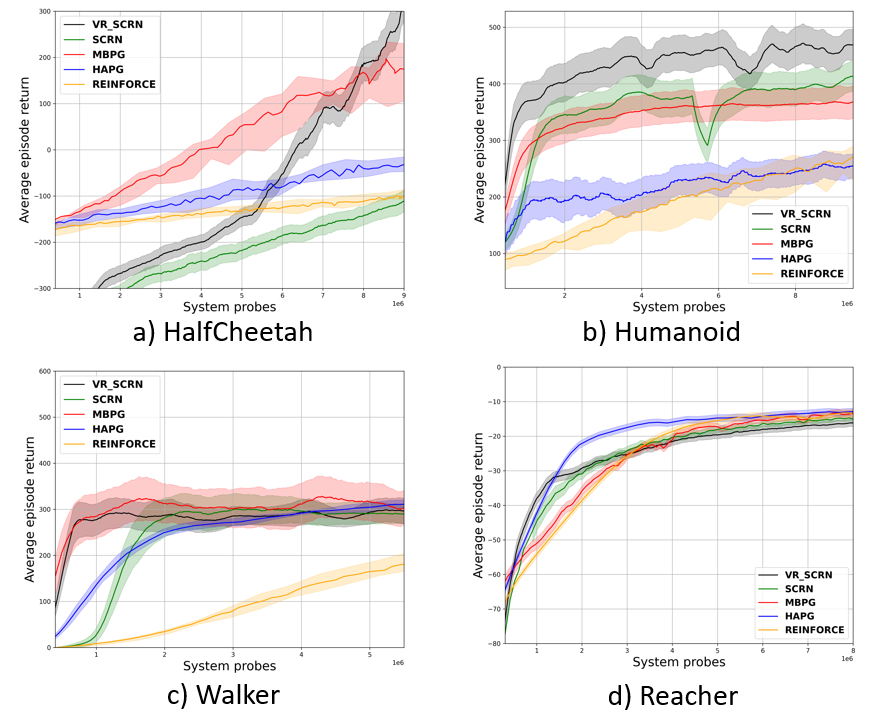}
    \caption{Comparison of SCRN and its variance-reduced version (without using the trick in updates) with REINFORCE and variance-reduced SPG methods in MuJoCo environments.}
    \label{fig8}
    \vspace{-2mm}
\end{figure*}
\subsection{Implementation details}
\textbf{Environments:}
\begin{itemize}
    \item Environments with finite state and action spaces:
We considered two grid world environments in our experiments: \citep[Example 6.6]{sutton2018reinforcement}, and random mazes \cite{mazelab}. In cliff walking, the agent's aim is to reach a goal state from a start state, avoiding a region of cells called ``cliff''. The episode is terminated if the agent enters the cliff region or the number of steps exceeds $100$ without reaching the goal. The reward is $-0.1$ in all transitions except those into the cliff where the reward is $-100$. 
The reward of reaching the goal is $100$. Moreover, we considered a softmax tabular policy for all the experiments of this part. 

In random mazes, the size of each maze is $10\times 10$.
In the random shape maze, random shape blocks are placed on a grid and the agent tries to reach the goal state finding the shortest path, avoiding blocks. 
  The reward is $-0.1$ in all transitions except if the agent tries to go to a cell which belongs to a block where the reward is $-1$. Moreover, the reward of reaching the goal is $1$. An episode is terminated if the agent could not reach the goal after $200$ steps.
 \item Environments with continuous state and action spaces:
We considered the following control tasks in MuJoCo simulator \cite{todorov2012mujoco}: Walker, Humanoid, Reacher, and HalfCheeta. We compared SCRN with first-order methods such as REINFORCE, and two state-of-the-art representatives of variance-reduced PG methods, HAPG \cite{shen2019hessian} and MBPG \cite{huang2020momentum}. For each task, we utilized a Gaussian multi-layer perceptron (MLP) policy whose mean and variance are parameterized by an MLP with two hidden layers of 64 neurons. For a fair comparison, we considered the same network architecture for all the methods. 
\end{itemize}

\textbf{Algorithms:}

For the environments with finite state and action spaces, we provided an implementation of SCRN, SPG, and REINFORCE with NumPy where the gradient and Hessian (just for SCRN) of a given trajectory are computed based on their closed forms. We also implemented SCRN with PyTorch in Garage library in order to execute it on environments with continuous state and action spaces. 
In our experiments, we used a Linux server with Intel Xeon Gold 6240 CPU (36 cores) operating at 2.60GHz with 377 GB DDR4 of memory, and Nvidia Titan X GPU.

Regarding SPG and REINFORCE in grid world environments, we adapted a time-varying learning rate by checking various forms and the form of $a/(\lfloor t/P\rfloor +b)$ provided the best performance for the first-order methods where $a,b,$ and $P$ are some constants that are needed to be tuned. The parameter $P$ shows the number of episodes that the learning rate remains unchanged.

\textbf{A trick in updates of SCRN and its variance-reduced version:}
Experiments in Figures \ref{fig:MuJoCo} and \ref{fig7} are conducted using the following trick to prevent updating based on a large $\|{\bf\Delta}_{t}\|$. If $\|{\bf\Delta}_{t}\|$ precedes a threshold, we neglect that whole iteration and do not update $\xv_{t}$. Please note that we take the number of observed state-actions of that neglected iteration into account in computing system probes.

 For sake of completeness, in Figure \ref{fig8}, we also provided the performances of SCRN and its variance-reduced version without using the trick described above, and compared with other variance-reduced methods.
 
\textbf{Demonstrations:}

We studied how the parameters of the softmax tabular policy are evolving over time in cliff walking environment. For each cell, we considered an arrow with four directions: up, down, left, and right. At any time, the direction of each arrow shows which one of the four directions has the highest probability of being taken by the agent in the corresponding cell. The color of the arrow becomes darker as this probability increases. For each algorithm, we run $100$ episodes and each episode, we demonstrated how the parameters are being updated. For SPG and REINFORCE, for a long period of time, at the start state, the agent takes the ``up'' action in order to avoid falling off the cliff. In fact, it takes many episodes until the parameters of the start state are updated such that the agent tries to find a path to the goal. In contrast, SCRN finds a path to the goal after few episodes and then tries to improve the path length.

\end{document}